\documentclass[wcp]{jmlr}

\usepackage{longtable}

\usepackage{booktabs}

\usepackage{graphicx}
\usepackage{booktabs}

\usepackage{url}            
\usepackage{booktabs}       
\usepackage{nicefrac}       
\usepackage{microtype}      
\usepackage{xcolor}         

\usepackage{algorithm}
\usepackage{diagbox}
\SetKwInOut{Parameter}{parameter}
\usepackage{wrapfig}
\usepackage{caption}
\usepackage{subcaption}
\usepackage{multirow}
\usepackage{paralist}
\usepackage{makecell}
\usepackage{color}
\usepackage{xcolor}
\usepackage{threeparttable}
\newcounter{mylistcntr}

\newcommand\numberthis{\addtocounter{equation}{1}\tag{\theequation}}

\usepackage{dsfont}

\newcommand{\x}{\mathbf{x}}
\def\sign{\texttt{sign}}
\def\ORAT{\texttt{ORAT}}

\def\1{\cellcolor{green!30}}
\def\2{\cellcolor{green!30}}
\def\3{\cellcolor{yellow!30}}
\def\4{\cellcolor{yellow!30}}
\def\5{\cellcolor{yellow!30}}
\def\6{\cellcolor{orange!40}}
\def\7{\cellcolor{orange!40}}
\def\8{\cellcolor{orange!40}}
\def\9{\cellcolor{orange!40}}
\def\g{\cellcolor{orange!40}}
\def\gg{\cellcolor{blue!30}}
\def\ggg{\cellcolor{blue!30}}
\definecolor{color_best}{RGB}{22, 138, 173}

\usepackage{tikz}
\newcommand\mybox[2][]{\tikz[overlay]\node[fill=blue!20,inner sep=2pt, anchor=text, rectangle, rounded corners=1mm,#1] {#2};\phantom{#2}}

\usepackage{colortbl}

\makeatletter
\let\Ginclude@graphics\@org@Ginclude@graphics 
\makeatother

\jmlrvolume{}
\jmlryear{2023}
\jmlrworkshop{ACML 2023}

\title[ORAT]{Outlier Robust Adversarial Training}

 \author{\Name{Shu Hu} \Email{hu968@purdue.edu}\\
 \addr Department of Computer and Information Technology, Purdue School of Engineering and Technology\\
 Indiana University–Purdue University Indianapolis\\
 Indianapolis, IN, 46202. USA
 \AND
 \Name{Zhenhuan Yang} \Email{zhenhuan.yang@hotmail.com}\\
\addr Etsy, Inc, Brooklyn, New York, USA
\AND
\Name{Xin Wang} \Email{xwang56@albany.edu}\\
\addr Department of Epidemiology and Biostatistics,
School of Public Health\\
University at Albany, State University of New York\\
Albany, NY 12222, USA
\AND
\Name{Yiming Ying} \Email{yying@albany.edu}\\
\addr Department of Mathematics and Statistics\\
University at Albany, State University of New York\\
Albany, NY 12222, USA
\AND
\Name{Siwei Lyu} \Email{siweilyu@buffalo.edu}\\
\addr Department of Computer Science and Engineering\\
University at Buffalo, State University of New York\\
Buffalo, NY 14260-2500, USA
}

\editors{Berrin Yan{\i}ko\u{g}lu and Wray Buntine}

\begin{document}

\maketitle

\begin{abstract}
Supervised learning models are challenged by the intrinsic complexities of training data such as outliers and minority subpopulations and intentional attacks at inference time with adversarial samples. While traditional robust learning methods and the recent adversarial training approaches are designed to handle each of the two challenges, to date, no work has been done to develop models that are robust with regard to the low-quality training data and the potential adversarial attack at inference time simultaneously. It is for this reason that we introduce \underline{O}utlier \underline{R}obust \underline{A}dversarial \underline{T}raining (\ORAT) in this work. \ORAT~is based on a bi-level optimization formulation of adversarial training with a robust rank-based loss function. Theoretically, we show that the learning objective of \ORAT~satisfies the $\mathcal{H}$-consistency \citep{awasthi2021calibration} in binary classification, which establishes it as a proper surrogate to adversarial 0/1 loss. Furthermore, we analyze its generalization ability and provide uniform convergence rates in high probability. \ORAT~can be optimized with a simple algorithm. Experimental evaluations on three benchmark datasets demonstrate the effectiveness and robustness of \ORAT~in handling outliers and adversarial attacks. 
Our code is available at \url{https://github.com/discovershu/ORAT}.
\end{abstract}

\vspace{-5mm}
\section{Introduction}
\label{sec:intro}


In supervised learning, we obtain a parametric model $f_\theta(\x)$ to predict the label (discrete or continuous) $y$ from an input $\x$. The optimal value of the parameter $\theta$ is obtained with a set of labeled training data  $\{(\x_i,y_i)\}_{i=1}^n$, by minimizing a loss function. However, supervised learning algorithms are challenged by two types of data degradation. First, the quality of training data is affected by erroneous samples due to mistakes in data collection or labeling (often known as the {\em outliers}) or isolated sub-populations of actual samples \citep{sukhbaatar2015training,pu2022learning, guo2022robust}. In addition, at inference time, an input data point $\x$ could be intentionally modified to create an adversarial sample that misleads $f_\theta(\x)$ to make a wrong prediction \citep{ goodfellow2014explaining,hu2021tkml}. 

To date, the resilience of supervised learning models with regard to unintentional non-ideal training data or intentional adversarial attacks has been studied separately in machine learning. The former is the topic of robust learning methods \citep{ wang2018iterative, hu2020learning, zhai2021doro, hu2022sum}, and the latter is addressed with adversarial training (AT) \citep{madry2018towards, wang2019improving, zhang2021geometry}. However, in practice, the two issues can occur in tandem. This has been noticed in \cite{sanyal2021benign} and \cite{zhu2021understanding}. They have demonstrated that the outlier problem (i.e., label noise) exists in AT and found the model performance degrades with the increase in noise level because outliers hurt the quality of training data. Only \cite{zhu2021understanding} provide a heuristic algorithm based on a label correction strategy to correct the noise label. However, their approach may introduce more extra noisy labels due to the imperfect classifier and cannot handle outliers with true labels that do not belong to the existing label list.

To reduce the influence of the outliers in AT, one may think of adapting the self-learning \citep{han2019deep} approach (i.e., a robust learning algorithm) to remove examples that are most likely outliers (e.g., examples with larger loss) from data before training and then conducting AT on the cleaner set. However, there are two drawbacks to this simple scheme. First, it is not an end-to-end approach, which will increase the training cost. Second, it is hard to remove outliers precisely, and an imperfect strategy may drop examples of clean data points. 
This may hurt the final model performance in the AT phase.  An alternative solution is that combine robust learning and AT by using robust losses (e.g., Huber loss \citep{hastie2009elements}, symmetric cross entropy loss \citep{wang2019symmetric}, etc.) in AT. However, most existing robust losses cannot eliminate the influence of outliers. In addition, constructing the theoretical guarantee is a challenge for using a robust loss in AT, especially for $\mathcal{H}$-calibration and $\mathcal{H}$-consistency properties \citep{awasthi2021calibration,steinwart2007compare} of the designed adversarial loss with respect to the adversarial 0/1 loss in classification.  

\begin{figure*}[t]
\centering
\includegraphics[trim=1 1 1 1, clip,keepaspectratio, width=1\textwidth]{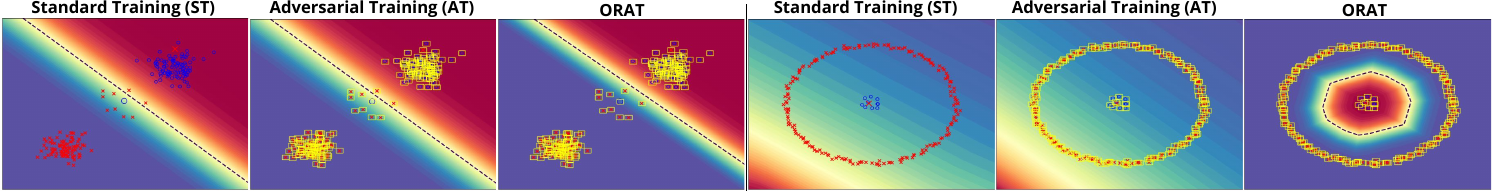}
\vspace{-1.7em}
\caption{\small \it Illustrative examples of standard training (ST), adversarial training (AT), and \ORAT~for binary classification on a balanced but multi-modal synthetic dataset (left panel) with two outliers and an imbalanced synthetic dataset (right panel) with one outlier. outliers in the blue and red classes are shown in {\color{red!100} $\times$} and {\color{blue!100} $\circ$}, respectively. The yellow squares around data samples represent the samples are perturbed within a $l_\infty$ ball. The dash lines are the decision boundaries. In the right figures, using ST and AT without data transfer techniques, there are no decision boundaries since both strategies cannot achieve a classifier that separates the two classes. 
}
\vspace{-1.7em}
\label{fig:interpretation}
\end{figure*}

  
In this work, we introduce \underline{O}utlier \underline{R}obust \underline{A}dversarial \underline{T}raining (\ORAT) to combine robust learning and adversarial training. Specifically, we develop an effective adversarial training algorithm for a rank-based learning objective that can exclude the influence of outliers from the training procedure. 
Figure \ref{fig:interpretation} shows several illustrative examples. The learning objective in \ORAT~lends itself to an efficient numerical algorithm based on gradient descent methods after reformulation that removes the explicit ranking operation.
{To verify whether the optimal minimizers of the \ORAT~loss are close to or exactly the optimal minimizers of the adversarial 0/1 loss, we show that the \ORAT~loss satisfies the $\mathcal{H}$-calibration and consistency properties for classification under some moderate conditions,  
which establishes it as a proper surrogate to adversarial 0/1 loss. The notion of consistency has been studied in \cite{awasthi2021calibration}. However, the \ORAT~loss is the first adversarial surrogate loss proven to satisfy $\mathcal{H}$-consistency and evaluated on real-world datasets. It encourages the applicability of the $\mathcal{H}$-consistent adversarial surrogate loss in real tasks.} 
We further provide a quantitative error bound of the generalization gap between the training and testing performance of \ORAT.
Experimental evaluations on three benchmark datasets demonstrate the effectiveness and robustness of \ORAT~in handling outliers and adversarial attacks. 
Our contributions  can be summarized as follows:
\begin{compactitem}
\item We present outlier robust adversarial training (\ORAT), which can handle outliers in adversarial training jointly.
\item We show \ORAT~loss satisfies $\mathcal{H}$-calibration and $\mathcal{H}$-consistency. To the best of our knowledge, the \ORAT~loss is the first  $\mathcal{H}$-consistent adversarial surrogate loss that is evaluated on real-world datasets. 
\item We provide a detailed theoretical analysis on the generalization error of training with \ORAT~loss. 
\end{compactitem}

\vspace{-2mm}
\section{Background}

\textbf{Robust Learning}.
Training accurate machine learning models in the presence of outliers is of great practical importance. To combat outliers, the traditional methods are designed based on label correction \citep{wang2018iterative}, loss correction \citep{han2020training}, and refined training strategies \citep{yu2019does}. However, they require an extra clean dataset or potentially expensive detection process to estimate the outliers. Recent works \cite{hu2020learning, hu2023rank} proposed the average of ranked range (AoRR) loss, which can eliminate the influence of the outliers if their proportion in training data is known.  Let $\mathcal{X}$ denote the input feature space, $\x\in \mathcal{X}$ is a training sample, $y\in \mathcal{Y}=\{1,\cdots,C\}$ is its associated label, and $C\geq 2$. $f_{\theta}(\cdot): \mathcal{X}\rightarrow \mathbb{R}^C$ is the logits output of the predictor, and $\ell: \mathbb{R}^C \times \mathcal{Y} \rightarrow \mathbb{R}$ is a loss function. $\ell_{[i]}(f_\theta(\x_j),y_j)_{j=1}^n$ represents the $i$-th largest loss among the training sample set.
For two integers $k$ and $m$, $0\leq m < k \leq n$, the AoRR loss is defined as follows,
\vspace{-2mm}
\begin{equation}
    \begin{aligned}
    \min_\theta \ \ \frac{1}{k-m}\sum_{i=m+1}^k \ell_{[i]}(f_\theta(\x_j),y_j)_{j=1}^n.
    \end{aligned}
\label{eq:aorr}
\vspace{-2mm}
\end{equation}
The AoRR loss excludes training samples with the top $m$-largest loss value, as well as samples with small losses. This is to reduce the influence of the outliers (larger losses) and to enhance the effect of the minority subgroup of the data because the samples with small loss values are most likely from the majority subgroup in the training set. 


\textbf{Adversarial Training}.
Recent studies have shown some surprising vulnerabilities of advanced supervised learning models, especially those based on deep neural networks, to specially designed adversarial samples \citep{goodfellow2014explaining,carlini2017towards}. To mitigate this issue, adversarial training (AT) \citep{madry2018towards} is  proposed as a training approach against adversarial attacks. 
Let $\mathcal{B}_\epsilon(\x)=\{\x'\in\mathcal{X}|\|\x-\x'\|_p\leq \epsilon\}$ 
be the closed ball of radius $\epsilon>0$ centered at $\x$, where $p$ is usually chosen as 1, 2, or $\infty$. Given a training dataset $\mathcal{S}=\{(\x_i,y_i)\}_{i=1}^n$ independently drawn from a distribution $\mathcal{D}$, where $\x_i\in \mathcal{X}$ and $y_i\in \mathcal{Y}$. The objective function of AT is then
\vspace{-2mm}
\begin{equation}
\begin{aligned}
    \min_{\theta} \frac{1}{n}\sum_{i=1}^n \Big[\max_{\tilde{\x}_i\in\mathcal{B}_\epsilon(\x)}\ell(f_{\theta}(\tilde{\x}_i), y_i)\Big],
\end{aligned}
\label{eq:std_at}
\vspace{-1mm}
\end{equation}
where $\tilde{\x}$ is the most extreme adversarial sample within the $\epsilon$-ball centered at $\x$. To generate adversarial data, The original AT applies projected gradient descent (PGD) using a fixed number of iterations $P$ as a stopping criterion, namely the PGD$^{P}$ algorithm, to approximately solve the inner maximization problem of the Eq.(\ref{eq:std_at}). Other methods can also generate adversarial data, e.g., the fast gradient signed method (FGSM) \citep{goodfellow2014explaining} and the CW attack \citep{carlini2017towards}.

\textbf{Robust Learning under Adversarial Attacks}.
Currently, several works have realized the impact of label noise on adversarial training. For example, in \cite{sanyal2021benign}, the authors identified label noise as one of the causes for adversarial vulnerability but no defense methods are proposed to solve this problem.
The work \cite{zhu2021understanding} empirically studies the efficacy of AT for mitigating the effect of label noise in training data. However, their proposed annotator algorithm (RAA) is based on the label correction strategy, which may introduce more extra noisy labels due to the bottleneck of the selected classifier. Furthermore, both of these methods are only considered label noise problems instead of outliers, especially the error that comes from the sample itself. Therefore, an outlier robust adversarial training method is urgently needed to fill this gap.
{Several works \citep{augustin2020adversarial, bitterwolf2020certifiably} connect adversarial robustness to out-of-distribution (OOD) problems. 
Note that the notion of outliers is different from OOD points.
One assumption for the OOD problem is that the training and test data distributions are mismatched. However, we do not have this assumption in this work.} 
More related works are discussed in Appendix \ref{sec:related_works}.

\vspace{-2mm}
\section{Method}

In \ORAT, we combine the rank-based AoRR loss (Eq.(\ref{eq:aorr})) and adversarial training  (Eq.(\ref{eq:std_at})) into the same framework. Let $L\left(\{(\x_j,y_j)\}_{j=1}^n\right):= \{\ell(f_\theta(\x_1),y_{1}), \cdots, \ell(f_\theta(\x_n),y_{n})\}$ be the set of all individual losses on the training samples. For notational brevity, we drop the explicit dependence on the individual loss $\ell$, the parametric model $f_\theta$, but it should be clear that the set changes with the training dataset $\{(\x_j,y_j)\}_{j=1}^n$.
In addition, we denote $\ell_{[i]}\left(\{(\x_j,y_j)\}_{j=1}^n\right)$ as the $i$-th largest individual loss after sorting the elements in set $L\left(\{(\x_j,y_j)\}_{j=1}^n\right)$ (ties can be broken in any consistent way). With these definitions, we define the learning objective of \ORAT~as a bi-level optimization problem: 
\begin{equation}
\begin{aligned}
    &\min_\theta \  \frac{1}{k-m}\sum_{i=m+1}^k \ell(f_\theta(\tilde{\x}_{[i]}),y_{[i]}), \ \mbox{s.t.} \  (\tilde{\x}_{[i]},y_{[i]}) = {\arg\!\max}_{\tilde{\x}_j\in\mathcal{B}_\epsilon(\x_j)} \ell_{[i]}(\{\tilde{\x}_j,y_j\}_{j=1}^n),
    \end{aligned}
\label{eq:atrr}
\end{equation}
where $k$ and $m$ are two integers such that $0\leq m < k \leq n$. The notations here require some further explanation. First, $\tilde{\x}_{j}$ is the result of adversarial perturbation of an original data point $\x_j$ with a perturbation strength $\epsilon$, and  $y_{[i]}$ is the corresponding label of the original $\x_j$. So the maximization sub-problem of Eq.(\ref{eq:atrr}) reads as follows. For each original training sample ($\x_j,y_j$), we find the extreme adversarial input $\tilde{\x}_j$ within the $\epsilon$-ball centered at $\x_j$. Then we calculate the individual loss of the perturbed samples $\ell(f_\theta(\tilde{\x}_j),y_j)$ and sort the losses to find the adversarial sample and label $(\tilde{\x}_{[i]},y_{[i]})$ corresponding to the top-$i$ individual loss. The outer minimization problem of Eq.(\ref{eq:atrr}) is the average of ranked range of $(m,k]$ individual losses found with such a procedure. Note that \ORAT~contains the original AT as a special case with $k=n$ and $m=0$.

{The robustness to outliers and adversarial attacks of the \ORAT~(Eq.(\ref{eq:atrr})) can be more clearly understood as follows.
The overall method can be viewed as a sample selection method and can filter incorrect data samples according to the top-$k$ and top-$m$ individual losses. Specifically, the perturbed samples with small losses are most likely come from the clean data. Therefore, it ignores $m$ samples with a large loss (i.e., top-1 to top-$m$ losses) during the training. On the other hand, the perturbed samples with an extremely low loss value are most likely very easy to be learned in the training procedure. They usually come from the majority group in a dataset. On the contrary, the perturbed samples from the minority subgroup can be also viewed as hard samples, which are very hard to be learned. In this case, ignoring the bottom $n-k$ would reduce the influence of the majority subgroup of data, which further prevents the effect of the imbalance data and enhance the impact of the minority subgroup of the data. Figure \ref{fig:interpretation} also demonstrates this influence.}

{Although \ORAT~is designed by combining AoRR and AT, this combination has not been explored in the existing literature. More importantly, we will demonstrate and validate both theoretically (Section \ref{sec:analysis}) and experimentally (Section \ref{sec:exp}) that this combination makes sense. To the best of our knowledge, no robust losses are combined with AT that can handle both low-quality training data (with outliers) and adversarial attacks in inference. In addition, the $\rho$-margin loss \citep{awasthi2021calibration}, a generalization of the ramp loss, is proved to satisfy the calibration and consistency. Inspired by this loss, we naturally thought that the truncated robust loss could be used to design a new adversarial surrogate loss that satisfies $\mathcal{H}$-calibration and consistency. Therefore, we choose AoRR loss to create \ORAT~because it is a well-defined truncated loss.} 
The ranking operation in \ORAT~is the main obstacle to using Eq.(\ref{eq:atrr}) as a learning objective in an efficient way.  
However, we can substitute the ranking operation by introducing two auxiliary variables $\lambda$  and $\hat{\lambda}$ and use the equivalent form of \ORAT~in the following result.



\vspace{-0.1cm}
\begin{theorem}
\label{theorem1}
Denote $[a]_+ = \max\{0,a\}$ as the hinge function.
Eq.(\ref{eq:atrr}) is equivalent to
\begin{equation}\label{eq:theorem1}
\begin{aligned}
& \frac{1}{k-m}\min_{\theta,\lambda} \max_{\hat{\lambda}} \ \ \sum_{i=1}^n\hat{\mathcal{L}}(f_\theta,\lambda, \hat{\lambda}):= 
 \Big[ \frac{k-m}{n}\lambda+\frac{n-m}{n}\hat{\lambda} -[\hat{\lambda}-[\ell(f_{\theta}(\tilde{\x}_i), y_i)-\lambda]_+]_+\Big],\\
& \emph{s.t.} \ \  \tilde{\x}_{i} = \arg\max_{\tilde{\x}\in\mathcal{B}_\epsilon(\x_i)} \ell(f_{\theta}(\tilde{\x}), y_i).
\end{aligned}
\end{equation}

Furthermore, $\hat{\lambda}>\lambda$, when the optimal solution is achieved.
\end{theorem}

\begin{algorithm}[t]
    \caption{Outlier Robust Adversarial Training}\label{alg:atrr}
    \SetAlgoLined
    \KwIn{A training dataset $\mathcal{S}$ of size $n$}
    \KwOut{A robust model with parameters $\theta^*$} 
    
    \textbf{Initialization:} $\theta^{(0)}$, $\lambda^{(0)}$, $\hat{\lambda}^{(0)}$, $t$=0, $\eta$, $k$, $m$,  $\epsilon$, $\alpha$, and $P$  

    
    \For{$e=1$ to num\_epoch}{
    \For{$b=1$ to num\_batch}{
    { \mbox{Sample a mini-batch $\mathcal{S}_b=\{(\x_i, y_i)\}_{i=1}^{|\mathcal{S}_b|}$ from $\mathcal{S}$}}
    
    \For{$i=1$ to batch\_size}{
    $\tilde{\x}_i\leftarrow \x_i$
    
    \While{$P > 0$}{
    
    { $\tilde{\x}_i$=$\Pi_{\mathcal{B}_\epsilon(\x_i)}(\tilde{\x}_i$+$\alpha \sign(\nabla_{\tilde{\x}_i}\ell(f_{\theta^{(t)}}(\tilde{\x}_i),y_i)))$}
    
    $P\leftarrow P-1$
    
    }

    }
    
    $\theta^{(t\!+\!1)}\!\leftarrow\!\theta^{(t)}\!-\!\frac{\eta}{|\mathcal{S}_b|}\sum_{i\in \mathcal{S}_b}\!\partial_\theta \hat{\mathcal{L}}(f_{\theta^{(t)}},\lambda^{(t)}, \hat{\lambda}^{(t)})$
    
    $\lambda^{(t\!+\!1)}\!\leftarrow\!\lambda^{(t)}\!-\!\frac{\eta}{|\mathcal{S}_b|}\sum_{i\in \mathcal{S}_b}\!\partial_\lambda \hat{\mathcal{L}}(f_{\theta^{(t)}},\lambda^{(t)}, \hat{\lambda}^{(t)})$
    
    $\hat{\lambda}^{(t\!+\!1)}\!\leftarrow\!\hat{\lambda}^{(t)}\!+\!\frac{\eta}{|\mathcal{S}_b|}\sum_{i\in \mathcal{S}_b}\!\partial_{\hat{\lambda}} \hat{\mathcal{L}}(f_{\theta^{(t)}},\lambda^{(t)}, \hat{\lambda}^{(t)})$
    
    $t\leftarrow t+1$
    
    }
    }
\end{algorithm}

The proof of Theorem \ref{theorem1} can be found in Appendix \ref{sec:proof_theorem1}.
Using Theorem \ref{theorem1}, we can develop a learning algorithm based on stochastic (mini-batch) gradient descent to optimize Eq.(\ref{eq:atrr}). 
Specifically, with initial choice for the values of $\theta^{(0)}$, $\lambda^{(0)}$, and $\hat{\lambda}^{(0)}$, at the $t$-th iteration, a mini-batch set $\mathcal{S}_b$ of training samples is chosen uniformly at random from the training set and used to estimate the (sub)gradient of the objective. Since the optimal solution of parameters of $\theta$, $\lambda$, and $\hat{\lambda}$ do not depend on the factor of $\frac{1}{k-m}$, we can replace it to $\frac{1}{|\mathcal{S}_b|}$ for mini-batch optimization, where $|\mathcal{S}_b|$ is the size of $\mathcal{S}_b$. Following the original AT \citep{madry2018towards}, we use the projected gradient descent approach to approximately solve the constraint. $\Pi_{\mathcal{B}_\epsilon(\x_i)}(\cdot)$ is the projection function 
that projects the adversarial data back into the $\epsilon$-ball centered at $\x_i$ if necessary. $\partial_\theta \hat{\mathcal{L}}$, $\partial_\lambda \hat{\mathcal{L}}$, and $\partial_{\hat{\lambda}} \hat{\mathcal{L}}$ are the (sub)gradients of $\hat{\mathcal{L}}$ with respect to $\theta$, $\lambda$, and $\hat{\lambda}$, respectively.  Their explicit forms can be found in Appendix \ref{appendix:explicit_forms}. The pseudo-code of optimizing \ORAT~is described in Algorithm \ref{alg:atrr}.
Note that our optimization process is similar to the traditional adversarial training algorithms \citep{madry2018towards, zhang2021geometry, liu2021probabilistic} with one additional minimization with respect to $\lambda$ and one additional maximization for $\hat{\lambda}$. Therefore, our algorithm has the same time complexity to original AT.
We will show that our algorithm can converge in experiments.

\vspace{-2mm}
\section{Analysis}\label{sec:analysis}
In this section, we prove that \ORAT~is a $\mathcal{H}$-consistent adversarial surrogate loss and study its generalization property. 
Subsequently, we focus on the case of binary classification ($\mathcal{Y}$=$ \{\pm 1\}$) with margin-based loss function, {\em i.e.},  $\ell(yf(\x))$=$\ell(f(\x),y)$=$\ell(f,\x,y)$, where we omit model parameter $\theta$ for simplicity.
We first introduce several notations. Let $\mathcal{H}$ be a \underline{H}ypothesis function set  from $\mathbb{R}^d$ to $\mathbb{R}$. We say $\mathcal{H}$ is symmetric, if for any $f\in\mathcal{H}$, $-f$ is also in $\mathcal{H}$. The 0/1 risk of a classifier $f\in \mathcal{H}$ is $\mathcal{R}_{\ell_0}$=$\mathbb{E}_{(\x,y)\sim\mathcal{D}}[\ell_0(f(\x),y)]$, where $\ell_0(f(\x),y)$=$\mathds{1}_{yf(\x)\leq 0}$ is the  0/1 loss. Denote the adversarial 0/1 risk as $\mathcal{R}_{\widetilde{\ell}_0}(f)$=$ \mathbb{E}_{(\x,y)\sim \mathcal{D}}[\widetilde{\ell}_0(f(\x),y)]$, where $\widetilde{\ell}_0(f(\x),y)$=$\sup_{\tilde{\x}\in\mathcal{B}_\epsilon(\x)}\mathds{1}_{yf(\tilde{\x})\leq 0}$ is the adversarial 0/1 loss. We also define the $\ell_s$-risk of $f$ for a surrogate loss $\ell_s(f(\x),y)$ as $\mathcal{R}_{\ell_s}(f)$= $\mathbb{E}_{(\x,y)\sim \mathcal{D}}[\ell_s(f(\x),y)]$ and the minimal ($\ell_s$, $\mathcal{H}$)-risk, which 
is defined by $\mathcal{R}^*_{\ell_s,\mathcal{H}}$=$\inf_{f\in \mathcal{H}}\mathcal{R}_{\ell_s}(f)$.


\subsection{Classification Calibration and Consistency}
Designing a robust algorithm entails using an appropriate surrogate loss to the standard 0/1 loss since 0/1 loss is very hard to optimize for most hypothesis sets \citep{awasthi2021calibration, bao2020calibrated}. We first provide a definition of $\mathcal{H}$-consistency.
\begin{definition}\label{def:consistency}
($\mathcal{H}$-Consistency). \citep{awasthi2021calibration} Given a hypothesis set $\mathcal{H}$, we say that a loss function $\ell_1$ is $\mathcal{H}$-consistent w.r.t. a loss function $\ell_2$, if, for all probability distributions and sequences of $\{f_n\}_{n\in\mathbb{N}}\subset\mathcal{H}$, there holds {\small $\mathcal{R}_{\ell_1}(f_n)-\mathcal{R}_{\ell_1,\mathcal{H}}^*\xrightarrow[]{n\rightarrow +\infty}0 \Rightarrow \mathcal{R}_{\ell_2}(f_n)-\mathcal{R}_{\ell_2,\mathcal{H}}^*\xrightarrow[]{n\rightarrow +\infty}0$}.
\end{definition}
$\mathcal{H}$-consistency  guarantees that the optimal minimizers of the surrogate adversarial loss are equal to or near the optimal minimizers of the 0/1 adversarial loss on a restricted hypothesis set $\mathcal{H}$.

For a distribution $\mathcal{D}$ over $\mathcal{X}\times \mathcal{Y}$ with random variables $X$ and $Y$. Let $\eta:= \Pr(Y=1|X=\x)\in [0,1]$, for any $\x\in \mathcal{X}$. Using conditional expectation, we can rewrite $\mathcal{R}_{\ell_s}(f)$ as $\mathcal{R}_{\ell_s}(f)=\mathbb{E}_X[\mathcal{C}_{\ell_s}(f,\x,\eta)]$, where $ \mathcal{C}_{\ell_s} (f,\x,\eta):=\eta \ell_s(f,\x,1)+(1-\eta)\ell_s(f,\x,-1), \forall \x\in \mathcal{X}.$ Furthermore, the minimal inner $\ell_s$-risk on $\mathcal{H}$ is denoted by $\mathcal{C}^*_{\ell_s, \mathcal{H}}(\x,\eta):=\inf_{f\in \mathcal{H}}\mathcal{C}_{\ell_s} (f,\x,\eta)$. 
We now define $\mathcal{H}$-calibration.
\begin{definition}\label{eq:calibration}
($\mathcal{H}$-Calibration). \citep{steinwart2007compare, awasthi2021calibration} Given a hypothesis set $\mathcal{H}$, we say that a loss function $\ell_1$ is $\mathcal{H}$-calibrated with respect to a loss function $\ell_2$ if, for any $\tau>0$, $\eta\in [0,1]$, and $\x\in\mathcal{X}$, there exists $\delta>0$ such that, for all $f\in \mathcal{H}$, {\small $\mathcal{C}_{\ell_1}(f,\x,\eta)\!<\!\mathcal{C}_{\ell_1,\mathcal{H}}^*(\x,\eta)+\delta \Rightarrow  \mathcal{C}_{\ell_2}(f,\x,\eta)\!<\!\mathcal{C}_{\ell_2,\mathcal{H}}^*(\x,\eta)+\tau$}. 
\end{definition}
As shown in \citep{steinwart2007compare}, if $\ell_1$ is $\mathcal{H}$-calibrated with respect to $\ell_2$, then $\mathcal{H}$-consistency holds for any probability distribution verifying the additional conditions, {which will be further reflected in the following Theorem \ref{theorem2}}. 

Without considering the adversarial perturbations and omitting the parameter $\theta$, the population version of the objective function of Eq.(\ref{eq:theorem1}) can be formulated as follows:
\begin{equation*}
\begin{aligned}
    &\min_{\lambda} \max_{\hat{\lambda}} \frac{1}{k-m}\!\sum_{i=1}^n\!\hat{\mathcal{L}}(f,\lambda, \hat{\lambda})
    \!\xrightarrow[n\rightarrow \infty]{\frac{k-m}{n}\rightarrow \nu, \frac{m}{n} \rightarrow \mu}\!\frac{1}{\nu}\min_{\lambda} \max_{\hat{\lambda}} \mathbb{E}\Big[\hat{\lambda}\!-\![\hat{\lambda}\!-\![\ell(Yf(X) )\!-\!\lambda]_+]_+ \Big] \!+\!\nu \lambda\!-\!\mu \hat{\lambda}.
\end{aligned}
\end{equation*}
Throughout the paper, we assume that $\mu>0$ since if $\mu=0$ then it will lead to $\hat{\lambda}=\infty$. As such, the population version of our \ORAT~loss is given by $(f_0^*, \lambda^*, \hat{\lambda}^*)=\arg \inf_{f,\lambda} \sup_{\hat{\lambda}}\Big\{\mathbb{E}\Big[\hat{\lambda}-[\hat{\lambda}-[\ell(Yf(X))-\lambda]_+]_+ \Big] +\nu \lambda-\mu \hat{\lambda}\Big\}.$
It is difficult to directly work with the optima $f_0^*$ since the above problem is a non-convex minmax problem and the standard minmax theorem does not apply here. Instead, we assume the existence of $\lambda^*$ and $\hat{\lambda}^*$ and work with the minimizer $f^*=\arg \inf_{f}\mathcal{L}(f,\lambda^*, \hat{\lambda}^*)$ where $\mathcal{L}(f,\lambda^*, \hat{\lambda}^*):=\mathbb{E}\Big[\hat{\lambda}^*-[\hat{\lambda}^*-[\ell(Yf(X))-\lambda^*]_+]_+ \Big] +\nu \lambda^*-\mu \hat{\lambda}^*$. Since the term $\nu \lambda^*-\mu \hat{\lambda}^*$ does not depend on $f$, we have $f^*=\arg\inf_{f}\mathbb{E}\Big[\hat{\lambda}^*-[\hat{\lambda}^*-[\ell(Yf(X))-\lambda^*]_+]_+ \Big]$.
From the above observations, we denote by $\phi_{\ORAT}$ and $\tilde{\phi}_{\ORAT}$  for \ORAT~loss without and with the adversarial perturbation, respectively,  as follows: 
\begin{equation}
\begin{aligned}
    &\phi_{\ORAT}(t) =  \hat{\lambda}^*-[\hat{\lambda}^*-[\ell(t)-\lambda^*]_+]_+, \ \ \ \tilde{\phi}_{\ORAT}(f,\x,y) =\sup_{\tilde{\x}\in \mathcal{B}_\epsilon(\x)}\phi_{\ORAT}( yf(\tilde{\x})).
\end{aligned}
\label{eq:two_phi}
\end{equation}
We can then obtain the following theorem. Its proof can be found in Appendix \ref{proof_consistency}.
\begin{theorem}\label{theorem2}
Let $\mathcal{H}$ be a symmetric hypothesis set consisting of the family of all measurable functions $\mathcal{H}_{all}$, suppose  $\nu>\min\{\hat{\lambda}^*, \mathcal{R}^*_{\ell,\mathcal{H}}\}$, $0\leq\lambda^*<\hat{\lambda}^*$, $\lambda^*$ and $\hat{\lambda}^*$ are bounded, and $\ell$ is a non-negative, continuous, and non-increasing margin-based loss.

(i) Then $\tilde{\phi}_{\ORAT}$ is $\mathcal{H}$-calibrated with respect to $\widetilde{\ell}_0$. 

(ii) Furthermore, $\tilde{\phi}_{\ORAT}$ is $\mathcal{H}$-consistent with respect to $\widetilde{\ell}_0$ for all distributions $\mathcal{D}$ over $\mathcal{X}\times \mathcal{Y}$ that satisfy: $\mathcal{R}_{\widetilde{\ell}_0, \mathcal{H}}^* =0$ and there exists $f^*\in\mathcal{H}$ such that $\mathcal{R}_{\phi_{\ORAT}}(f^*)=\mathcal{R}_{\phi_{\ORAT},\mathcal{H}_{all}}^*<+\infty$. 
\end{theorem}
\vspace{-2mm}

$\mathcal{H}$ can be linear models or deep neural networks. The commonly used hinge loss, logistic loss, and cross-entropy loss all satisfy the conditions of $\ell$ {(See Appendix \ref{sec:margin_ce} for details)}.
Furthermore, according to Theorem \ref{theorem1} and assuming $\ell \ge 0$, we have $0\leq\lambda^*<\hat{\lambda}^*$. For $\nu$, it should be larger than the smaller value among $\hat{\lambda}^*$ and minimal ($\ell$, $\mathcal{H}$)-risk, which means $k$ and $m$ should be as far away from each other as possible. We call $\mathcal{R}_{\widetilde{\ell}_0, \mathcal{H}}^* =0$ as the realizability condition. Therefore, we can conclude that \ORAT~loss satisfies $\mathcal{H}$-consistency. Our proof is mainly inspired by \cite{awasthi2021calibration}. 
However, our analysis for the $\tilde{\phi}_{\ORAT}$ adversarial surrogate loss is more involved since it is a composition function based on $\ell$ for which we need to consider the different conditions of $\lambda^*$ and $\hat{\lambda}^*$.

\subsection{Generalization Error}
In this subsection, we present the generalization error bound for the proposed \ORAT~loss. Define the adversarial surrogate loss $\widetilde{\ell}(yf(\x)) = \max_{\tilde{\x} \in \mathcal{B}_\epsilon(\x)} \ell(yf(\tilde{\x}))$ and the composite function class $\widetilde{\ell}_{\mathcal{H}} := \{\widetilde{\ell}(yf(\x)): f \in \mathcal{H}\}$. The generalization error studies the discrepancy   between the empirical adversarial risk $\mathcal{R}_{\widetilde{\ell}}(f; \mathcal{S})$ defined on the training data and its population risk $\mathcal{R}_{\widetilde{\ell}}(f)$ measuring the performance on the test data, where $\mathcal{R}_{\widetilde{\ell}}(f; \mathcal{S}) = \inf_{\lambda}\sup_{\hat{\lambda}} \frac{1}{k-m}\sum_{i=1}^n \hat{\mathcal{L}}(f, \lambda, \hat{\lambda})$. 
First, we need the Rademacher complexity \citep{bartlett2002rademacher} which is defined as follows.
\begin{definition}[Rademacher complexity]
For any function class $\mathcal{H}$, given a dataset $\{\x_i\}_{i=1}^n$, the empirical Rademacher complexity is defined as $\mathfrak{R}_n(\mathcal{H}) = \frac{1}{n}\mathbb{E}_\sigma \Big[\sup_{f \in \mathcal{H}} \sum_{i=1}^n \sigma_i f(\x_i)\Big]$,
where $\sigma_1, \cdots, \sigma_n$ are i.i.d. Rademacher random variables with $\mathbb{P}[\sigma_i = 1] = \mathbb{P}[\sigma_i = -1] = 1/2$.
\end{definition}
\vspace{-2mm}

With these preparations, we can get the following theorem. 

\begin{theorem}\label{thm:generalization}
Suppose that the range of $\ell(f(\x), y)$ is $[0, M]$ . Then, for any $\delta \in (0, 1)$, with probability at least $1 -\delta$ {over the draw of an i.i.d. training dataset of size $n$}, the following holds for all $f \in \mathcal{H}$, 
\begin{equation*}
    \begin{aligned}
    \mathcal{R}_{\widetilde{\ell}}(f) - \mathcal{R}_{\widetilde{\ell}}(f; \mathcal{S}) 
\leq \frac{2}{\nu}\Big( 2\mathfrak{R}_n(\widetilde{\ell}_\mathcal{H}) + \frac{M(2\sqrt{2} + 3\sqrt{\log(2/\delta)})}{\sqrt{2n}}\Big).
    \end{aligned}
\end{equation*}
\vspace{-2mm}
\end{theorem}
\vspace{-2mm}

The proof of Theorem \ref{thm:generalization} is by rewriting the empirical risk and population risk with optimal $\lambda^*$ and $\hat{\lambda}^*$, 
and then analytically bound the Rademacher complexity of $\lambda^*$ and $\hat{\lambda}^*$ by utilizing the boundness condition. See Appendix \ref{sec:proof-gen} for details. Theorem \ref{thm:generalization} characterizes uniform convergence between the training and testing on \ORAT~given hypothesis set $\mathcal{H}$. The generalization error depends on the limit of ranked range $\frac{k-m}{n} \rightarrow \nu$ as well as the Rademacher complexity of the adversarial loss function class. {It is worth noting the i.i.d. assumption over the sample is restrictive but required for applying the symmetrization trick in the proof, and how to relax this assumption to a more realistic assumption, for example, $n-m$ i.i.d. inliers with $m$ outliers \citep{laforgue2021generalization}, is future work for us. Nevertheless, Theorem \ref{thm:generalization} still highlights the generalization ability of our \ORAT~objective with respect to the ranked range and adversarial loss in this ideal setting. We provide hypothesis set examples of linear classifiers  and neural networks in Appendix \ref{sec:rademacher_example} for further characterizing the Rademacher complexity $\mathfrak{R}_n(\widetilde{\ell}_\mathcal{H})$. }


\begin{remark}
Theorem \ref{thm:generalization} together with Theorem \ref{theorem2} indicates that learning with \ORAT~asymptotically converges to zero on the adversarial 0/1 risk. Let $f^{**} = \arg\inf_{f \in \mathcal{H}} \mathcal{R}_{\widetilde{\ell}}(f)$ and $f_S^{**} = \arg\inf_{f \in \mathcal{H}} \mathcal{R}_{\widetilde{\ell}}(f; S)$. By standard error decomposition, $\mathcal{R}_{\widetilde{\ell}}(f_S^{**}) - \mathcal{R}_{\widetilde{\ell}}(f^{**}) = \big(\mathcal{R}_{\widetilde{\ell}}(f_S^{**}) - \mathcal{R}_{\widetilde{\ell}}(f_S^{**};S)\big) + \big(\mathcal{R}_{\widetilde{\ell}}(f_S^{**};S) - \mathcal{R}_{\widetilde{\ell}}(f^{**};S)\big) + \big(\mathcal{R}_{\widetilde{\ell}}(f^{**};S) - \mathcal{R}_{\widetilde{\ell}}(f^{**})\big)$. The first term converges to 0 as $n$ goes to $\infty$ and $\mathfrak{R}_n(\widetilde{\ell}_\mathcal{H})$ is bounded. The second term is always non-positive as $f_S^{**}$ minimizes $\mathcal{R}_{\widetilde{\ell}}(f;S)$. The last term is bounded by $\mathcal{O}(1/\sqrt{n})$ with high probability by Hoeffding's inequality \citep{hoeffding1994probability}. Therefore, $\mathcal{R}_{\widetilde{\ell}}(f_S^{**}) - \mathcal{R}_{\widetilde{\ell}}(f^{**}) \rightarrow 0 $  as $n \rightarrow \infty$. Combined with Theorem \ref{theorem2} (ii) and Definition \ref{def:consistency}, it shows that $\mathcal{R}_{\widetilde{\ell_0}}(f_S^{**}) - \mathcal{R}_{\widetilde{\ell_0}}(f^{**}) \rightarrow 0$ as $n \rightarrow \infty$. 
\end{remark}


\section{Experiments}\label{sec:exp}
We evaluate the performance of the proposed \ORAT~with numerical experiments. Due to the
limit of the space, we present the most significant information
and results of our experiments.
More detailed information, additional results, and code are in Appendix \ref{additional_exp_details}, \ref{additional_exp_results}, and supplementary files, respectively. 

\subsection{Experimental Settings} \label{sec:ex_setting}
\textbf{Datasets, Networks, and Baselines}. Our experiments are based on three popular datasets, namely MNIST, 
CIFAR-10, and CIFAR-100.
We follow the same training/testing splitting from the original datasets.  The pixel values of all images are normalized into the range of [0,1]. We adopt LeNet \citep{lecun1998gradient}, Small-CNN \citep{wang2019improving}, and ResNet-18 \citep{he2016deep} for MNIST, CIFAR-10, and CIFAR-100, respectively. 
Settings of networks are in Appendix \ref{sec:computing_description}.
We compare \ORAT~with standard training (ST), adversarial training (AT). In addition, as there are no
dedicated adversarial training methods for handling outliers, we compare three instance-reweighted AT methods GAIRAT \citep{zhang2021geometry}, MAIL \citep{liu2021probabilistic}, and WMMR \citep{zeng2021adversarial}, because they are the most recent works with the best performance against adversarial attacks. We also compare RAA \citep{zhu2021understanding} since it considers label correction in training. 
{To support our claims about the drawbacks of using a self-learning strategy to remove outliers before adversarial training in Section \ref{sec:intro}, we use the AoRR approach to remove examples with larger losses (potential outliers) from the original training set and obtain a cleaner set for adversarial training. We call this baseline AT w/o and compare it with our method. Appendix \ref{sec:AT_w/o} shows more details of this baseline.}
All methods are designed based on the  cross-entropy loss.

\textbf{Defense Settings and Robustness Evaluation}. In training, following \cite{zeng2021adversarial, liu2021probabilistic}, we consider robustness by setting $p=\infty$. For each dataset, we test two different perturbation bounds $\epsilon$, i.e., $\epsilon\in\{0.1, 0.2\}$ for MNIST and $\epsilon\in\{2/255, 8/255\}$ for CIFAR-10 and CIFAR-100. The training attack is PGD$^{10}$ with random start and step size $\epsilon/4$. We evaluate the robustness of our method and baselines using the standard accuracy on natural test data (Natural), FGSM, PGD$^{20}$, and the CW attack  because they are frequently used in the literature of our compared methods. We also evaluate all methods on a very strong attack, AutoAttack (AA) \citep{croce2020reliable}, under specific settings. All of them are constrained by the same perturbation limit $\epsilon$.  

\textbf{Hyperparameters ($k$ and $m$) and Outliers Generation}. Following \cite{hu2020learning}, we apply a grid search to select the values of $k$ and $m$. 
To simulate the outliers in the training dataset, 
as in the work of \cite{wang2019symmetric}, we add the symmetric (uniform)  and asymmetric (class-dependent) noises to the labels of the training data. 
For symmetric noise creation, we randomly choose training samples with a given probability $\gamma$ and change each of their labels to another random label. For asymmetric noise creation, given $\gamma$, flipping labels only within a specific set of classes. More details of the simulation can be found in Appendix \ref{sec:outlier_generation}.
Note that we use 4 different $\gamma\in\{10\%, 20\%, 30\%, 40\%\}$.

\subsection{Results}

\textbf{Overall Performance}.
\begin{table*}[t!]
\centering
\setlength\tabcolsep{0.8pt}
\scalebox{0.53}{
\begin{tabular}{|cc|c|cccc|cccc|cccc|cccc|cccc|cccc|}
\hline

\multicolumn{2}{|c|}{\multirow{3}{*}{Noise}} & \multirow{3}{*}{Defense} & \multicolumn{8}{c|}{MNIST (LeNet)}  & \multicolumn{8}{c|}{CIFAR-10 (Small-CNN)} & \multicolumn{8}{c|}{CIFAR-100 (ResNet-18)} \\ \cline{4-27} 
\multicolumn{2}{|c|}{}                  &                   & \multicolumn{4}{c|}{$\epsilon$=0.1}  & \multicolumn{4}{c|}{$\epsilon$=0.2} & \multicolumn{4}{c|}{$\epsilon$=2/255}  & \multicolumn{4}{c|}{$\epsilon$=8/255}  & \multicolumn{4}{c|}{$\epsilon$=2/255}  & \multicolumn{4}{c|}{$\epsilon$=8/255}  \\ \cline{4-27} 
\multicolumn{2}{|l|}{}                  &                   & \multicolumn{1}{c}{Na} & \multicolumn{1}{c}{FG} & \multicolumn{1}{c}{PGD} & CW            & \multicolumn{1}{c}{Na} & \multicolumn{1}{c}{FG} & \multicolumn{1}{c}{PGD} & CW & \multicolumn{1}{c}{Na} & \multicolumn{1}{c}{FG} & \multicolumn{1}{c}{PGD} & CW  & \multicolumn{1}{c}{Na} & \multicolumn{1}{c}{FG} & \multicolumn{1}{c}{PGD} & CW & \multicolumn{1}{c}{Na} & \multicolumn{1}{c}{FG} & \multicolumn{1}{c}{PGD} & CW  & \multicolumn{1}{c}{Na} & \multicolumn{1}{c}{FG} & \multicolumn{1}{c}{PGD} & CW \\ \hline


\multicolumn{2}{|c|}{\multirow{3}{*}{0}} &                   ST          & \multicolumn{1}{c}{99.51} & \multicolumn{1}{c}{92.43} & \multicolumn{1}{c}{85.68} &  86.05   & \multicolumn{1}{c}{\textbf{99.41}} & \multicolumn{1}{c}{50.19} & \multicolumn{1}{c}{16.76} & 18.65 & \multicolumn{1}{c}{\textbf{92.48}} & \multicolumn{1}{c}{44.27} & \multicolumn{1}{c}{24.87} & 25.06 & \multicolumn{1}{c}{\textbf{92.63}} & \multicolumn{1}{c}{15.03} & \multicolumn{1}{c}{6.82} & 7.15 & \multicolumn{1}{c}{\textbf{73.67}} & \multicolumn{1}{c}{19.64} & \multicolumn{1}{c}{10.66} & 11.13 & \multicolumn{1}{c}{\textbf{73.86}} & \multicolumn{1}{c}{3.58} & \multicolumn{1}{c}{1.61} & 1.66 \\  
\multicolumn{2}{|c|}{}                  &       AT          & \multicolumn{1}{c}{99.51} & \multicolumn{1}{c}{\1 98.60} & \multicolumn{1}{c}{\1 97.93} &  \1 97.94   & \multicolumn{1}{c}{99.16} & \multicolumn{1}{c}{\1 97.33} & \multicolumn{1}{c}{\2 95.14} & \2 95.19 & \multicolumn{1}{c}{89.78} & \multicolumn{1}{c}{\1 79.40} & \multicolumn{1}{c}{\1 73.52} & \1 73.64 & \multicolumn{1}{c}{85.52} & \multicolumn{1}{c}{\1 54.61} & \multicolumn{1}{c}{\1 33.34} & \1 34.28 & \multicolumn{1}{c}{67.05} & \multicolumn{1}{c}{\2 53.70} & \multicolumn{1}{c}{\2 48.18} & \3 47.46 & \multicolumn{1}{c}{55.13} & \multicolumn{1}{c}{\2 33.66} & \multicolumn{1}{c}{\2 25.46} & \2 23.28 \\  \cline{3-27}
\rowcolor{lightgray}
\multicolumn{2}{|c|}{\cellcolor{white}}                  &       \textbf{Ours}         & \multicolumn{1}{c}{\textbf{99.53}} & \multicolumn{1}{c}{\textbf{98.70}} & \multicolumn{1}{c}{\textbf{98.03}} &  \textbf{98.04}   & \multicolumn{1}{c}{99.30} & \multicolumn{1}{c}{\textbf{97.79}} & \multicolumn{1}{c}{\textbf{96.43}} & \textbf{96.38} & \multicolumn{1}{c}{89.54} & \multicolumn{1}{c}{\textbf{79.53}} & \multicolumn{1}{c}{\textbf{73.96}} & \textbf{73.86} & \multicolumn{1}{c}{85.91} & \multicolumn{1}{c}{\textbf{54.85}} & \multicolumn{1}{c}{\textbf{33.46}} & \textbf{34.56} & \multicolumn{1}{c}{68.00} & \multicolumn{1}{c}{\textbf{55.60}} & \multicolumn{1}{c}{\textbf{50.06}} & \textbf{49.50} & \multicolumn{1}{c}{57.82} & \multicolumn{1}{c}{\textbf{35.51}} & \multicolumn{1}{c}{\textbf{27.22}} & \textbf{25.26} \\ \hline \hline

\multicolumn{1}{|c|}{\multirow{28}{*}{\rotatebox{90}{Symmetric \ \ \ \ \ \ \ \ Noise}}} & \multirow{7}{*}{\rotatebox{90}{10\%}} &       ST          & \multicolumn{1}{c}{99.32} & \multicolumn{1}{c}{93.72} & \multicolumn{1}{c}{89.09} &  87.17   & \multicolumn{1}{c}{\textbf{99.32}} & \multicolumn{1}{c}{67.55} & \multicolumn{1}{c}{36.76} & 26.02 & \multicolumn{1}{c}{62.22} & \multicolumn{1}{c}{29.24} & \multicolumn{1}{c}{25.72} & 24.80 & \multicolumn{1}{c}{\textbf{61.80}} & \multicolumn{1}{c}{14.70} & \multicolumn{1}{c}{9.72} & 11.38 & \multicolumn{1}{c}{29.60} & \multicolumn{1}{c}{14.17} & \multicolumn{1}{c}{11.41} & 9.96 & \multicolumn{1}{c}{29.41} & \multicolumn{1}{c}{5.22} & \multicolumn{1}{c}{2.54} & 2.19 \\ 
\multicolumn{1}{|l|}{}                   &                   &          AT         & \multicolumn{1}{c}{99.12} & \multicolumn{1}{c}{\1 97.98} & \multicolumn{1}{c}{\1 97.36} & \1 97.28  & \multicolumn{1}{c}{98.39} & \multicolumn{1}{c}{\1 96.62} & \multicolumn{1}{c}{\1 95.25} & \1 95.11 & \multicolumn{1}{c}{63.03} & \multicolumn{1}{c}{\1 51.36} & \multicolumn{1}{c}{\3 46.02} & \2 44.95 & \multicolumn{1}{c}{59.53} & \multicolumn{1}{c}{\2 34.17} & \multicolumn{1}{c}{\4 23.33} & \2 22.28& \multicolumn{1}{c}{31.02} & \multicolumn{1}{c}{\4 22.04} & \multicolumn{1}{c}{\3 20.25} & \3 18.97 & \multicolumn{1}{c}{28.96} & \multicolumn{1}{c}{\6 15.08} & \multicolumn{1}{c}{\4 12.42} &\4 10.49 \\  
\multicolumn{1}{|l|}{}                   &                   &         GAIRAT    & \multicolumn{1}{c}{99.11} & \multicolumn{1}{c}{\1 97.87} & \multicolumn{1}{c}{\1 97.33} &  \1 97.20  & \multicolumn{1}{c}{98.35} & \multicolumn{1}{c}{\2 96.57} & \multicolumn{1}{c}{\1 95.22} & \2 95.08 & \multicolumn{1}{c}{60.41} & \multicolumn{1}{c}{\3 49.23} & \multicolumn{1}{c}{\5 44.72} & \5 42.02 & \multicolumn{1}{c}{57.12} & \multicolumn{1}{c}{\1 35.48} & \multicolumn{1}{c}{\3 24.22} & \2 21.52& \multicolumn{1}{c}{31.05} & \multicolumn{1}{c}{\5 21.31} & \multicolumn{1}{c}{\3 19.76} & \4 18.08 & \multicolumn{1}{c}{27.73} & \multicolumn{1}{c}{\5 15.13} & \multicolumn{1}{c}{\4 12.28} &\4 10.26\\  
\multicolumn{1}{|l|}{}                   &                   &         MAIL     & \multicolumn{1}{c}{99.09} & \multicolumn{1}{c}{\1 98.06} & \multicolumn{1}{c}{\1 97.50} &  \1 97.40  & \multicolumn{1}{c}{98.48} & \multicolumn{1}{c}{\1 96.69} & \multicolumn{1}{c}{\1 95.47} & \1 95.20 & \multicolumn{1}{c}{55.14} & \multicolumn{1}{c}{\5 47.44} & \multicolumn{1}{c}{\5 44.75} & \8 39.64 & \multicolumn{1}{c}{53.54} & \multicolumn{1}{c}{\3 33.61} & \multicolumn{1}{c}{\1 26.18} & \4 20.23 & \multicolumn{1}{c}{29.06} & \multicolumn{1}{c}{\7 18.90} & \multicolumn{1}{c}{\6 16.62} & \7 14.79 & \multicolumn{1}{c}{27.65} & \multicolumn{1}{c}{\8 13.21} & \multicolumn{1}{c}{\6 10.03} & \7 7.26\\  
\multicolumn{1}{|l|}{}                   &                   &        WMMR     & \multicolumn{1}{c}{99.14} & \multicolumn{1}{c}{\1 97.95} & \multicolumn{1}{c}{\1 97.32} & \1 97.23  & \multicolumn{1}{c}{98.36} & \multicolumn{1}{c}{\2 96.58} & \multicolumn{1}{c}{\1 95.25} & \1 95.11 & \multicolumn{1}{c}{61.83} & \multicolumn{1}{c}{\1 51.43} & \multicolumn{1}{c}{\3 46.56} & \2 44.97 & \multicolumn{1}{c}{57.97} & \multicolumn{1}{c}{\2 34.74} & \multicolumn{1}{c}{\2 25.15} & \2 21.96 & \multicolumn{1}{c}{29.74} & \multicolumn{1}{c}{\6 20.11} & \multicolumn{1}{c}{\4 18.65} & \5 16.86 & \multicolumn{1}{c}{29.23} & \multicolumn{1}{c}{\6 14.78} & \multicolumn{1}{c}{\4 11.84} &\5 9.96 \\  
\multicolumn{1}{|l|}{}                   &                   &       RAA       & \multicolumn{1}{c}{99.22} & \multicolumn{1}{c}{\1 98.01} & \multicolumn{1}{c}{\1 97.37} &  \1 97.31 & \multicolumn{1}{c}{98.65} & \multicolumn{1}{c}{\1 96.83} & \multicolumn{1}{c}{\1 95.32} & \1 95.14 & \multicolumn{1}{c}{62.94} & \multicolumn{1}{c}{\1 51.32} & \multicolumn{1}{c}{\3 46.55} & \2 45.36 & \multicolumn{1}{c}{58.28} & \multicolumn{1}{c}{\3 33.58} & \multicolumn{1}{c}{\4 23.71} &\1 22.89 & \multicolumn{1}{c}{30.72} & \multicolumn{1}{c}{\4 21.95} & \multicolumn{1}{c}{\2 20.34} & \3 18.83 & \multicolumn{1}{c}{28.99} & \multicolumn{1}{c}{\5 15.60} & \multicolumn{1}{c}{\3 13.03} &\3 11.03 \\ \cline{3-27} 
\rowcolor{lightgray}
\multicolumn{1}{|l|}{\cellcolor{white}}                   &   \cellcolor{white}                &       \textbf{Ours}      & \multicolumn{1}{c}{\textbf{99.52}} & \multicolumn{1}{c}{\textbf{98.45}} & \multicolumn{1}{c}{\textbf{97.78}} &  \textbf{97.79}  & \multicolumn{1}{c}{99.20} & \multicolumn{1}{c}{\textbf{97.61}} & \multicolumn{1}{c}{\textbf{96.07}} & \textbf{96.11} & \multicolumn{1}{c}{\textbf{63.83}} & \multicolumn{1}{c}{\textbf{52.04}} & \multicolumn{1}{c}{\textbf{48.77}} & \textbf{46.65} & \multicolumn{1}{c}{60.42} & \multicolumn{1}{c}{\textbf{36.13}} & \multicolumn{1}{c}{\textbf{26.86}} &\textbf{23.48} & \multicolumn{1}{c}{\textbf{35.76}} & \multicolumn{1}{c}{\textbf{25.72}} & \multicolumn{1}{c}{\textbf{22.27}} & \textbf{21.28} & \multicolumn{1}{c}{\textbf{35.35}} & \multicolumn{1}{c}{\textbf{20.37}} & \multicolumn{1}{c}{\textbf{15.50}} &\textbf{14.02} \\ \cline{2-27} 

\multicolumn{1}{|l|}{}                   & \multirow{7}{*}{\rotatebox{90}{20\%}} &          ST           & \multicolumn{1}{c}{99.12} & \multicolumn{1}{c}{93.35} & \multicolumn{1}{c}{89.05} &  86.87  & \multicolumn{1}{c}{\textbf{99.14}} & \multicolumn{1}{c}{70.18} & \multicolumn{1}{c}{41.67} & 29.30 & \multicolumn{1}{c}{58.89} & \multicolumn{1}{c}{31.23} & \multicolumn{1}{c}{25.97} & 24.13 & \multicolumn{1}{c}{\textbf{57.28}} & \multicolumn{1}{c}{15.88} & \multicolumn{1}{c}{12.8} & 13.26& \multicolumn{1}{c}{26.16} & \multicolumn{1}{c}{14.74} & \multicolumn{1}{c}{11.98} & 10.29 & \multicolumn{1}{c}{26.22} & \multicolumn{1}{c}{5.09} & \multicolumn{1}{c}{2.45} & 1.92\\ 
\multicolumn{1}{|l|}{}                   &                   &          AT         & \multicolumn{1}{c}{98.88} & \multicolumn{1}{c}{\1 97.67} & \multicolumn{1}{c}{\1 97.03} & \1 96.89  & \multicolumn{1}{c}{97.97} & \multicolumn{1}{c}{\2 95.98} & \multicolumn{1}{c}{\2 94.67} & \2 94.50 & \multicolumn{1}{c}{\textbf{61.10}} & \multicolumn{1}{c}{\2 50.67} & \multicolumn{1}{c}{\2 46.76} & \1 45.28 & \multicolumn{1}{c}{56.96} & \multicolumn{1}{c}{\2 34.55} & \multicolumn{1}{c}{\3 25.54} & \1 24.34& \multicolumn{1}{c}{27.52} & \multicolumn{1}{c}{\5 20.12} & \multicolumn{1}{c}{\4 18.49} & \4 17.15 & \multicolumn{1}{c}{26.77} & \multicolumn{1}{c}{\6 14.39} & \multicolumn{1}{c}{\4 11.87} & \4 9.96\\  
\multicolumn{1}{|l|}{}                   &                   &         GAIRAT     & \multicolumn{1}{c}{98.82} & \multicolumn{1}{c}{\1 97.59} & \multicolumn{1}{c}{\1 96.96} & \1 96.79  & \multicolumn{1}{c}{97.86} & \multicolumn{1}{c}{\2 95.81} & \multicolumn{1}{c}{\2 94.57} & \2 94.30 & \multicolumn{1}{c}{58.58} & \multicolumn{1}{c}{\5 49.26} & \multicolumn{1}{c}{\4 45.44} & \6 40.62 & \multicolumn{1}{c}{53.55} & \multicolumn{1}{c}{\2 34.40} & \multicolumn{1}{c}{\3 25.52} & \3 22.15& \multicolumn{1}{c}{26.77} & \multicolumn{1}{c}{\6 19.93} & \multicolumn{1}{c}{\4 18.30} & \5 16.78 & \multicolumn{1}{c}{24.80} & \multicolumn{1}{c}{\6 14.28} & \multicolumn{1}{c}{\4 12.03} &\4 10.32 \\  
\multicolumn{1}{|l|}{}                   &                   &         MAIL      & \multicolumn{1}{c}{98.81} & \multicolumn{1}{c}{\1 97.82} & \multicolumn{1}{c}{\1 97.29} & \1 97.12  & \multicolumn{1}{c}{97.98} & \multicolumn{1}{c}{\2 96.14} & \multicolumn{1}{c}{\1 94.81} & \2 94.58 & \multicolumn{1}{c}{54.04} & \multicolumn{1}{c}{\5 47.05} & \multicolumn{1}{c}{\4 44.78} & \6 41.05 & \multicolumn{1}{c}{52.93} & \multicolumn{1}{c}{\2 34.28} & \multicolumn{1}{c}{\1 27.74} & \3 21.90& \multicolumn{1}{c}{25.13} & \multicolumn{1}{c}{\9 17.01} & \multicolumn{1}{c}{\7 16.06} & \7 14.32 & \multicolumn{1}{c}{24.54} & \multicolumn{1}{c}{\7 12.85} & \multicolumn{1}{c}{\6 10.02} & \7 7.09\\  
\multicolumn{1}{|l|}{}                   &                   &        WMMR    & \multicolumn{1}{c}{98.88} & \multicolumn{1}{c}{\1 97.59} & \multicolumn{1}{c}{\1 97.00} &  \1 96.89  & \multicolumn{1}{c}{97.85} & \multicolumn{1}{c}{\2 95.89} & \multicolumn{1}{c}{\2 94.50} & \2 94.36 & \multicolumn{1}{c}{59.01} & \multicolumn{1}{c}{\3 49.64} & \multicolumn{1}{c}{\3 45.96} & \3 43.71 & \multicolumn{1}{c}{56.04} & \multicolumn{1}{c}{\2 34.65} & \multicolumn{1}{c}{\2 26.22} & \1 23.70& \multicolumn{1}{c}{26.72} & \multicolumn{1}{c}{\6 19.43} & \multicolumn{1}{c}{\5 18.04} & \5 16.53 & \multicolumn{1}{c}{26.61} & \multicolumn{1}{c}{\6 13.81} & \multicolumn{1}{c}{\4 11.94} & \4 9.82\\  
\multicolumn{1}{|l|}{}                   &                   &       RAA     & \multicolumn{1}{c}{99.07} & \multicolumn{1}{c}{\1 97.70} & \multicolumn{1}{c}{\1 97.07} &  \1 97.05   & \multicolumn{1}{c}{98.25} & \multicolumn{1}{c}{\2 96.28} & \multicolumn{1}{c}{\2 94.41} & \2 94.20 & \multicolumn{1}{c}{59.53} & \multicolumn{1}{c}{\3 49.56} & \multicolumn{1}{c}{\3 45.73} & \1 45.60 & \multicolumn{1}{c}{56.71} & \multicolumn{1}{c}{\3 33.24} & \multicolumn{1}{c}{\4 24.51} &\1 24.12 & \multicolumn{1}{c}{28.16} & \multicolumn{1}{c}{\5 20.73} & \multicolumn{1}{c}{\4 19.15} & \4 17.74 & \multicolumn{1}{c}{25.34} & \multicolumn{1}{c}{\5 14.79} & \multicolumn{1}{c}{\4 12.46} &\4 10.25 \\ \cline{3-27} 
\rowcolor{lightgray}
\multicolumn{1}{|l|}{\cellcolor{white}}                   &         \cellcolor{white}          &       \textbf{Ours}     & \multicolumn{1}{c}{\textbf{99.56}} & \multicolumn{1}{c}{\textbf{98.37}} & \multicolumn{1}{c}{\textbf{97.65}} &  \textbf{97.64}   & \multicolumn{1}{c}{99.06} & \multicolumn{1}{c}{\textbf{97.33}} & \multicolumn{1}{c}{\textbf{95.71}} & \textbf{95.68} & \multicolumn{1}{c}{61.03} & \multicolumn{1}{c}{\textbf{51.85}} & \multicolumn{1}{c}{\textbf{48.48}} & \textbf{46.16} & \multicolumn{1}{c}{54.37} & \multicolumn{1}{c}{\textbf{35.71}} & \multicolumn{1}{c}{\textbf{27.76}} & \textbf{24.70}& \multicolumn{1}{c}{\textbf{34.45}} & \multicolumn{1}{c}{\textbf{25.07}} & \multicolumn{1}{c}{\textbf{22.21}} & \textbf{20.92} & \multicolumn{1}{c}{\textbf{32.83}} & \multicolumn{1}{c}{\textbf{19.45}} & \multicolumn{1}{c}{\textbf{15.49}} & \textbf{13.34}\\ \cline{2-27} 

\multicolumn{1}{|l|}{}                   & \multirow{7}{*}{\rotatebox{90}{30\%}} &         ST        & \multicolumn{1}{c}{98.91} & \multicolumn{1}{c}{93.01} & \multicolumn{1}{c}{88.62} &   85.73  & \multicolumn{1}{c}{\textbf{98.97}} & \multicolumn{1}{c}{68.59} & \multicolumn{1}{c}{41.24} & 27.73 & \multicolumn{1}{c}{58.17} & \multicolumn{1}{c}{30.55} & \multicolumn{1}{c}{25.71} & 24.54 & \multicolumn{1}{c}{\textbf{57.55}} & \multicolumn{1}{c}{14.18} & \multicolumn{1}{c}{11.50} &11.41 & \multicolumn{1}{c}{23.38} & \multicolumn{1}{c}{14.26} & \multicolumn{1}{c}{11.84} & 10.10 & \multicolumn{1}{c}{24.29} & \multicolumn{1}{c}{5.30} & \multicolumn{1}{c}{2.90} & 2.35\\  
\multicolumn{1}{|l|}{}                   &                   &         AT          & \multicolumn{1}{c}{98.58} & \multicolumn{1}{c}{\2 97.21} & \multicolumn{1}{c}{\1 96.67} &  \1 96.53  & \multicolumn{1}{c}{97.26} & \multicolumn{1}{c}{\3 94.94} & \multicolumn{1}{c}{\2 93.48} & \3 93.17 & \multicolumn{1}{c}{57.78} & \multicolumn{1}{c}{\1 49.56} & \multicolumn{1}{c}{\1 46.22} & \1 44.80 & \multicolumn{1}{c}{54.35} & \multicolumn{1}{c}{\1 34.21} & \multicolumn{1}{c}{\2 26.07} & \1 24.33& \multicolumn{1}{c}{23.79} & \multicolumn{1}{c}{\5 18.96} & \multicolumn{1}{c}{\4 17.49} & \4 16.26 & \multicolumn{1}{c}{25.66} & \multicolumn{1}{c}{\6 13.84} & \multicolumn{1}{c}{\4 11.78} & \4 9.87\\  
\multicolumn{1}{|l|}{}                   &                   &          GAIRAT       & \multicolumn{1}{c}{98.51} & \multicolumn{1}{c}{\2 97.14} & \multicolumn{1}{c}{\1 96.58} &  \2 96.47   & \multicolumn{1}{c}{97.15} & \multicolumn{1}{c}{\3 94.70} & \multicolumn{1}{c}{\3 93.28} & \3 92.87 & \multicolumn{1}{c}{55.41} & \multicolumn{1}{c}{\4 47.20} & \multicolumn{1}{c}{\4 44.04} & \5 40.09 & \multicolumn{1}{c}{49.48} & \multicolumn{1}{c}{\2 33.44} & \multicolumn{1}{c}{\2 26.26} &\4 21.78 & \multicolumn{1}{c}{24.04} & \multicolumn{1}{c}{\5 18.93} & \multicolumn{1}{c}{\4 17.89} & \4 16.50 & \multicolumn{1}{c}{24.41} & \multicolumn{1}{c}{\5 14.11} & \multicolumn{1}{c}{\4 11.84} &\4 9.73 \\  
\multicolumn{1}{|l|}{}                   &                   &           MAIL        & \multicolumn{1}{c}{98.50} & \multicolumn{1}{c}{\1 97.39} & \multicolumn{1}{c}{\1 96.88} &  \1 96.71  & \multicolumn{1}{c}{97.31} & \multicolumn{1}{c}{\2 95.35} & \multicolumn{1}{c}{\2 94.11} & \2 93.66 & \multicolumn{1}{c}{52.94} & \multicolumn{1}{c}{\6 45.13} & \multicolumn{1}{c}{\5 42.47} & \6 39.01 & \multicolumn{1}{c}{47.31} & \multicolumn{1}{c}{\3 32.18} & \multicolumn{1}{c}{\2 25.78} &\4 21.99 & \multicolumn{1}{c}{21.93} & \multicolumn{1}{c}{\g 14.55} & \multicolumn{1}{c}{\8 13.67} & \8 12.14 & \multicolumn{1}{c}{22.61} & \multicolumn{1}{c}{\7 12.47} & \multicolumn{1}{c}{\6 9.62} & \7 7.09\\ 
\multicolumn{1}{|l|}{}                   &                   &           WMMR       & \multicolumn{1}{c}{98.58} & \multicolumn{1}{c}{\2 97.23} & \multicolumn{1}{c}{\1 96.69} & \1  96.59   & \multicolumn{1}{c}{97.23} & \multicolumn{1}{c}{\2 95.10} & \multicolumn{1}{c}{\2 93.59} & \3 93.31 & \multicolumn{1}{c}{58.21} & \multicolumn{1}{c}{\1 49.31} & \multicolumn{1}{c}{\1 46.40} &\1  44.17 & \multicolumn{1}{c}{53.99} & \multicolumn{1}{c}{\1 34.16} & \multicolumn{1}{c}{\1 26.83} &\2 23.70 & \multicolumn{1}{c}{23.68} & \multicolumn{1}{c}{\6 18.27} & \multicolumn{1}{c}{\5 17.08} & \5 15.46 & \multicolumn{1}{c}{25.45} & \multicolumn{1}{c}{\6 13.60} & \multicolumn{1}{c}{\5 11.46} &\4 9.32 \\  
\multicolumn{1}{|l|}{}                   &                   &          RAA       & \multicolumn{1}{c}{98.84} & \multicolumn{1}{c}{\1 97.35} & \multicolumn{1}{c}{\1 96.63} &  \1 96.59    & \multicolumn{1}{c}{97.66} & \multicolumn{1}{c}{\2 95.45} & \multicolumn{1}{c}{\2 93.74} & \2 93.40 & \multicolumn{1}{c}{58.72} & \multicolumn{1}{c}{\2 48.78} & \multicolumn{1}{c}{\2 45.20} & \1 44.85 & \multicolumn{1}{c}{50.81} & \multicolumn{1}{c}{\5 30.50} & \multicolumn{1}{c}{\4 24.49} &\3  22.60& \multicolumn{1}{c}{24.53} & \multicolumn{1}{c}{\4 20.06} & \multicolumn{1}{c}{\3 18.54} & \3 17.03 & \multicolumn{1}{c}{22.08} & \multicolumn{1}{c}{\6 13.95} & \multicolumn{1}{c}{\4 11.62} & \4 9.67\\ \cline{3-27} 
\rowcolor{lightgray}
\multicolumn{1}{|l|}{\cellcolor{white}}                   &          \cellcolor{white}         &         \textbf{Ours}         & \multicolumn{1}{c}{\textbf{99.55}} & \multicolumn{1}{c}{\textbf{98.30}} & \multicolumn{1}{c}{\textbf{97.51}} &  \textbf{97.53}   & \multicolumn{1}{c}{98.85} & \multicolumn{1}{c}{\textbf{96.99}} & \multicolumn{1}{c}{\textbf{95.31}} & \textbf{95.35} & \multicolumn{1}{c}{\textbf{58.99}} & \multicolumn{1}{c}{\textbf{50.29}} & \multicolumn{1}{c}{\textbf{47.11}} & \textbf{45.01} & \multicolumn{1}{c}{55.34} & \multicolumn{1}{c}{\textbf{34.89}} & \multicolumn{1}{c}{\textbf{27.66}} & \textbf{25.00}& \multicolumn{1}{c}{\textbf{31.27}} & \multicolumn{1}{c}{\textbf{23.81}} & \multicolumn{1}{c}{\textbf{21.35}} & \textbf{19.59} & \multicolumn{1}{c}{\textbf{31.13}} & \multicolumn{1}{c}{\textbf{18.96}} & \multicolumn{1}{c}{\textbf{15.47}} & \textbf{13.21}\\ \cline{2-27}

\multicolumn{1}{|l|}{}                   & \multirow{7}{*}{\rotatebox{90}{40\%}} &         ST       & \multicolumn{1}{c}{98.87} & \multicolumn{1}{c}{92.37} & \multicolumn{1}{c}{88.36} &  84.85    & \multicolumn{1}{c}{\textbf{98.91}} & \multicolumn{1}{c}{68.43} & \multicolumn{1}{c}{41.60} & 29.62 & \multicolumn{1}{c}{\textbf{57.60}} & \multicolumn{1}{c}{21.03} & \multicolumn{1}{c}{19.46} & 18.03 & \multicolumn{1}{c}{\textbf{56.67}} & \multicolumn{1}{c}{15.22} & \multicolumn{1}{c}{12.02} & 11.53& \multicolumn{1}{c}{21.25} & \multicolumn{1}{c}{13.95} & \multicolumn{1}{c}{12.15} & 10.21 & \multicolumn{1}{c}{21.55} & \multicolumn{1}{c}{6.03} & \multicolumn{1}{c}{3.77} & 2.96\\  
\multicolumn{1}{|l|}{}                   &                   &         AT          & \multicolumn{1}{c}{98.33} & \multicolumn{1}{c}{\2 96.86} & \multicolumn{1}{c}{\1 96.27} & \2 96.11   & \multicolumn{1}{c}{96.56} & \multicolumn{1}{c}{\3 94.37} & \multicolumn{1}{c}{\2 92.55} & \3 92.18 & \multicolumn{1}{c}{47.16} & \multicolumn{1}{c}{\1 43.45} & \multicolumn{1}{c}{\1 41.54} & \1 40.13 & \multicolumn{1}{c}{40.65} & \multicolumn{1}{c}{\1 31.59} & \multicolumn{1}{c}{\1 26.74} &\1  24.98 & \multicolumn{1}{c}{22.20} & \multicolumn{1}{c}{\5 18.28} & \multicolumn{1}{c}{\4 17.33} & \4 15.92 & \multicolumn{1}{c}{23.41} & \multicolumn{1}{c}{\4 13.75} & \multicolumn{1}{c}{\4 11.25} &\3 9.29\\  
\multicolumn{1}{|l|}{}                   &                   &          GAIRAT        & \multicolumn{1}{c}{98.10} & \multicolumn{1}{c}{\2 96.70} & \multicolumn{1}{c}{\2 96.15} &  \2 95.88  & \multicolumn{1}{c}{96.17} & \multicolumn{1}{c}{\3 93.93} & \multicolumn{1}{c}{\3 92.08} & \3 91.59 & \multicolumn{1}{c}{47.24} & \multicolumn{1}{c}{\4 40.24} & \multicolumn{1}{c}{\4 38.08} & \5 36.51 & \multicolumn{1}{c}{40.05} & \multicolumn{1}{c}{\2 30.33} & \multicolumn{1}{c}{\2 25.93} & \2 23.45 & \multicolumn{1}{c}{ 21.87} & \multicolumn{1}{c}{\5 18.06} & \multicolumn{1}{c}{\4 17.17} & \4 15.67 & \multicolumn{1}{c}{21.99} & \multicolumn{1}{c}{\5 12.97} & \multicolumn{1}{c}{\4 11.26} &\3 9.42\\  
\multicolumn{1}{|l|}{}                   &                   &           MAIL      & \multicolumn{1}{c}{98.33} & \multicolumn{1}{c}{\1 97.06} & \multicolumn{1}{c}{\1 96.44} &  \1 96.20    & \multicolumn{1}{c}{96.50} & \multicolumn{1}{c}{\3 94.45} & \multicolumn{1}{c}{\2 93.01} & \2 92.54 & \multicolumn{1}{c}{43.63} & \multicolumn{1}{c}{\5 39.43} & \multicolumn{1}{c}{\4 37.97} & \5 35.60 & \multicolumn{1}{c}{36.28} & \multicolumn{1}{c}{\2 29.91} & \multicolumn{1}{c}{\1 26.32} & \2 23.31& \multicolumn{1}{c}{18.83} & \multicolumn{1}{c}{\9 14.22} & \multicolumn{1}{c}{\8 13.49} & \8 12.08 & \multicolumn{1}{c}{20.47} & \multicolumn{1}{c}{\5 12.50} & \multicolumn{1}{c}{\4 10.33} &\5 7.35 \\ 
\multicolumn{1}{|l|}{}                   &                   &           WMMR      & \multicolumn{1}{c}{98.15} & \multicolumn{1}{c}{\2 96.79} & \multicolumn{1}{c}{\2 96.10} & \2  95.91    & \multicolumn{1}{c}{96.50} & \multicolumn{1}{c}{\3 94.24} & \multicolumn{1}{c}{\3 92.46} & \3 92.04 & \multicolumn{1}{c}{45.99} & \multicolumn{1}{c}{\2  41.95} & \multicolumn{1}{c}{\3 39.64} & \3 37.87 & \multicolumn{1}{c}{30.30} & \multicolumn{1}{c}{\5 27.14} & \multicolumn{1}{c}{\3 24.47} & \3 22.27 & \multicolumn{1}{c}{21.04} & \multicolumn{1}{c}{\6 17.33} & \multicolumn{1}{c}{\5 16.22} & \5 14.76 & \multicolumn{1}{c}{23.39} & \multicolumn{1}{c}{\5 12.79} & \multicolumn{1}{c}{\4 11.13} &\3 9.10\\  
\multicolumn{1}{|l|}{}                   &                   &          RAA        & \multicolumn{1}{c}{98.67} & \multicolumn{1}{c}{\2 96.90} & \multicolumn{1}{c}{\2 96.07} &  \2 96.00   & \multicolumn{1}{c}{97.05} & \multicolumn{1}{c}{\3 94.47} & \multicolumn{1}{c}{\3 92.37} & \3 92.22 & \multicolumn{1}{c}{53.37} & \multicolumn{1}{c}{\7 37.13} & \multicolumn{1}{c}{\7 35.16} & \7 33.82 & \multicolumn{1}{c}{40.01} & \multicolumn{1}{c}{\3 28.94} & \multicolumn{1}{c}{\3 25.16} & \2 23.73 & \multicolumn{1}{c}{22.44} & \multicolumn{1}{c}{\5 18.59} & \multicolumn{1}{c}{\4 17.54} & \4 16.01 & \multicolumn{1}{c}{20.74} & \multicolumn{1}{c}{\4 13.30} & \multicolumn{1}{c}{\3 11.56} &\3 9.56 \\ \cline{3-27} 
\rowcolor{lightgray}
\multicolumn{1}{|l|}{\cellcolor{white}}                   &            \cellcolor{white}       &         \textbf{Ours}         & \multicolumn{1}{c}{\textbf{99.36}} & \multicolumn{1}{c}{\textbf{98.00}} & \multicolumn{1}{c}{\textbf{97.22}} &  \textbf{97.20}   & \multicolumn{1}{c}{98.63} & \multicolumn{1}{c}{\textbf{96.49}} & \multicolumn{1}{c}{\textbf{94.48}} &\textbf{ 94.53 }& \multicolumn{1}{c}{53.20} & \multicolumn{1}{c}{\textbf{43.75}} & \multicolumn{1}{c}{\textbf{41.92}} & \textbf{40.60} & \multicolumn{1}{c}{41.86} & \multicolumn{1}{c}{\textbf{31.70}} & \multicolumn{1}{c}{\textbf{27.18}} & \textbf{25.19} & \multicolumn{1}{c}{\textbf{29.38}} & \multicolumn{1}{c}{\textbf{22.99}} & \multicolumn{1}{c}{\textbf{20.85}} & \textbf{19.20} & \multicolumn{1}{c}{\textbf{26.80}} & \multicolumn{1}{c}{\textbf{17.30}} & \multicolumn{1}{c}{\textbf{14.32}} &\textbf{11.95} \\ \hline \hline

\multicolumn{1}{|l|}{\multirow{28}{*}{\rotatebox{90}{Asymmetric \ \ \ \ \ \ \ \ \ \ \ Noise}}} & \multirow{7}{*}{\rotatebox{90}{10\%}} &       ST        & \multicolumn{1}{c}{99.39} & \multicolumn{1}{c}{93.05} & \multicolumn{1}{c}{86.09} &   89.09    & \multicolumn{1}{c}{\textbf{99.39}} & \multicolumn{1}{c}{65.48} & \multicolumn{1}{c}{22.17} & 35.02 & \multicolumn{1}{c}{\textbf{90.30}} & \multicolumn{1}{c}{40.73} & \multicolumn{1}{c}{8.24} & 7.25 & \multicolumn{1}{c}{\textbf{89.88}} & \multicolumn{1}{c}{18.85} & \multicolumn{1}{c}{0.27} &0.00 & \multicolumn{1}{c}{31.81} & \multicolumn{1}{c}{15.10} & \multicolumn{1}{c}{11.23} & 10.03 & \multicolumn{1}{c}{31.71} & \multicolumn{1}{c}{4.34} & \multicolumn{1}{c}{1.84} &1.71 \\ 
\multicolumn{1}{|l|}{}                   &                   &       AT           & \multicolumn{1}{c}{99.46} & \multicolumn{1}{c}{\1 98.14} & \multicolumn{1}{c}{\1 97.36} &  \1 97.61  & \multicolumn{1}{c}{99.08} & \multicolumn{1}{c}{\1 97.29} & \multicolumn{1}{c}{\1 95.21} &\1 95.64 & \multicolumn{1}{c}{88.49} & \multicolumn{1}{c}{\2 78.30} & \multicolumn{1}{c}{\2 72.22} &\1  71.43 & \multicolumn{1}{c}{75.38} & \multicolumn{1}{c}{\3 52.31} & \multicolumn{1}{c}{\1 30.99} & \1 30.43 & \multicolumn{1}{c}{33.21} & \multicolumn{1}{c}{\5 22.38} & \multicolumn{1}{c}{\4 20.53} & \4 19.09 & \multicolumn{1}{c}{30.83} & \multicolumn{1}{c}{\5 15.80} & \multicolumn{1}{c}{\3 12.49} &\4 10.55\\ 
\multicolumn{1}{|l|}{}                   &                   &        GAIRAT       & \multicolumn{1}{c}{99.44} & \multicolumn{1}{c}{\1 98.13} & \multicolumn{1}{c}{\1 97.34} &  \1 97.50      & \multicolumn{1}{c}{99.03} & \multicolumn{1}{c}{\1 97.20} & \multicolumn{1}{c}{\1 95.14} &\1  95.45 & \multicolumn{1}{c}{85.13} & \multicolumn{1}{c}{\3 75.19} & \multicolumn{1}{c}{\5 68.89} & \g 64.18 & \multicolumn{1}{c}{71.29} & \multicolumn{1}{c}{\2 52.77} & \multicolumn{1}{c}{\1 31.62} & \ggg 19.26 & \multicolumn{1}{c}{33.03} & \multicolumn{1}{c}{\5 22.73} & \multicolumn{1}{c}{\4 20.63} & \4 19.38 & \multicolumn{1}{c}{29.88} & \multicolumn{1}{c}{\5 15.58} & \multicolumn{1}{c}{\3 12.49} &\4 10.57\\  
\multicolumn{1}{|l|}{}                   &                   &         MAIL       & \multicolumn{1}{c}{99.50} & \multicolumn{1}{c}{\1 98.20} & \multicolumn{1}{c}{\1 97.38} & \1   97.73    & \multicolumn{1}{c}{99.08} & \multicolumn{1}{c}{\1 97.52} & \multicolumn{1}{c}{\1 95.79} & \1 96.35 & \multicolumn{1}{c}{77.93} & \multicolumn{1}{c}{\g 69.74} & \multicolumn{1}{c}{\g 65.63} & \ggg 54.36 & \multicolumn{1}{c}{67.39} & \multicolumn{1}{c}{\3 49.66} & \multicolumn{1}{c}{\1 31.62} &\ggg 18.65 & \multicolumn{1}{c}{31.45} & \multicolumn{1}{c}{\6 21.36} & \multicolumn{1}{c}{\6 17.92} & \7 16.43 & \multicolumn{1}{c}{29.66} & \multicolumn{1}{c}{\7 13.99} & \multicolumn{1}{c}{\6 10.18} &\7 7.64\\  
\multicolumn{1}{|l|}{}                   &                   &         WMMR       & \multicolumn{1}{c}{99.45} & \multicolumn{1}{c}{\1 98.11} & \multicolumn{1}{c}{\1 97.38} &   \1 97.59    & \multicolumn{1}{c}{99.08} & \multicolumn{1}{c}{\1 97.28} & \multicolumn{1}{c}{\2 95.30} & \1 95.67 & \multicolumn{1}{c}{86.68} & \multicolumn{1}{c}{\3 76.44} & \multicolumn{1}{c}{\3 70.23} & \3 68.39 & \multicolumn{1}{c}{72.85} & \multicolumn{1}{c}{\2 53.77} & \multicolumn{1}{c}{\1 32.57} & \3 28.74 & \multicolumn{1}{c}{32.52} & \multicolumn{1}{c}{\5 22.08} & \multicolumn{1}{c}{\5 19.63} & \5 17.81 & \multicolumn{1}{c}{30.89} & \multicolumn{1}{c}{\6 14.86} & \multicolumn{1}{c}{\4 11.91} &\4 10.01\\  
\multicolumn{1}{|l|}{}                   &                   &        RAA         & \multicolumn{1}{c}{99.42} & \multicolumn{1}{c}{\1 98.11} & \multicolumn{1}{c}{\1 97.34} &  \1 97.51   & \multicolumn{1}{c}{99.08} & \multicolumn{1}{c}{\1 97.24} & \multicolumn{1}{c}{\1 95.17} & \1 95.45 & \multicolumn{1}{c}{88.75} & \multicolumn{1}{c}{\1 78.20} & \multicolumn{1}{c}{\1 72.33} & \1 69.98 & \multicolumn{1}{c}{72.85} & \multicolumn{1}{c}{\1 53.98} & \multicolumn{1}{c}{\1 32.02} & \1 31.69 & \multicolumn{1}{c}{33.01} & \multicolumn{1}{c}{\4 23.07} & \multicolumn{1}{c}{\3 21.07} & \3 19.69 & \multicolumn{1}{c}{30.35} & \multicolumn{1}{c}{\5 16.09} & \multicolumn{1}{c}{\3 12.79} &\4 10.70 \\ \cline{3-27}  
\rowcolor{lightgray}
\multicolumn{1}{|l|}{\cellcolor{white}}                   &          \cellcolor{white}         &         \textbf{Ours}        & \multicolumn{1}{c}{\textbf{99.50}} & \multicolumn{1}{c}{\textbf{98.58}} & \multicolumn{1}{c}{\textbf{97.81}} &   \textbf{98.04}   & \multicolumn{1}{c}{99.18} & \multicolumn{1}{c}{\textbf{97.68}} & \multicolumn{1}{c}{\textbf{96.11}} & \textbf{96.36} & \multicolumn{1}{c}{{88.77}} & \multicolumn{1}{c}{\textbf{78.75}} & \multicolumn{1}{c}{\textbf{72.72}} & \textbf{71.97} & \multicolumn{1}{c}{75.42} & \multicolumn{1}{c}{\textbf{54.62}} & \multicolumn{1}{c}{\textbf{32.70}} & \textbf{32.17} & \multicolumn{1}{c}{\textbf{37.09}} & \multicolumn{1}{c}{\textbf{27.07}} & \multicolumn{1}{c}{\textbf{23.65}} & \textbf{22.59} & \multicolumn{1}{c}{\textbf{35.72}} & \multicolumn{1}{c}{\textbf{20.51}} & \multicolumn{1}{c}{\textbf{15.46}} & \textbf{13.83}\\ \cline{2-27} 

\multicolumn{1}{|l|}{}                   & \multirow{7}{*}{\rotatebox{90}{20\%}} &        ST          & \multicolumn{1}{c}{99.19} & \multicolumn{1}{c}{92.64} & \multicolumn{1}{c}{86.00} &   88.63   & \multicolumn{1}{c}{\textbf{99.23}} & \multicolumn{1}{c}{64.76} & \multicolumn{1}{c}{20.49} & 32.18 & \multicolumn{1}{c}{\textbf{87.69}} & \multicolumn{1}{c}{37.40} & \multicolumn{1}{c}{8.08} & 7.52 & \multicolumn{1}{c}{\textbf{87.16}} & \multicolumn{1}{c}{15.95} & \multicolumn{1}{c}{0.06} &0.00 & \multicolumn{1}{c}{30.46} & \multicolumn{1}{c}{14.79} & \multicolumn{1}{c}{11.57} & 10.38 & \multicolumn{1}{c}{30.49} & \multicolumn{1}{c}{4.92} & \multicolumn{1}{c}{2.14} & 2.03\\ 
\multicolumn{1}{|l|}{}                   &                   &       AT           & \multicolumn{1}{c}{99.44} & \multicolumn{1}{c}{\1 98.12} & \multicolumn{1}{c}{\1 97.29} &  \1 97.59   & \multicolumn{1}{c}{99.07} & \multicolumn{1}{c}{\1 97.11} & \multicolumn{1}{c}{\1 95.20} & \1 95.84 & \multicolumn{1}{c}{87.51} & \multicolumn{1}{c}{\1 77.08} & \multicolumn{1}{c}{\2 71.56} & \1 71.38 & \multicolumn{1}{c}{73.38} & \multicolumn{1}{c}{\3 52.04} & \multicolumn{1}{c}{\1 30.84} &\1 30.80 & \multicolumn{1}{c}{31.67} & \multicolumn{1}{c}{\5 21.15} & \multicolumn{1}{c}{\4 19.44} &\4  18.19 & \multicolumn{1}{c}{29.96} & \multicolumn{1}{c}{\5 15.47} & \multicolumn{1}{c}{\3 12.46} & \4 10.51\\ 
\multicolumn{1}{|l|}{}                   &                   &        GAIRAT          & \multicolumn{1}{c}{99.41} & \multicolumn{1}{c}{\1 98.10} & \multicolumn{1}{c}{\1 97.20} &  \1 97.54   & \multicolumn{1}{c}{99.01} & \multicolumn{1}{c}{\2 96.99} & \multicolumn{1}{c}{\1 95.16} & \1 95.66 & \multicolumn{1}{c}{74.61} & \multicolumn{1}{c}{\ggg 64.33} & \multicolumn{1}{c}{\ggg 62.22} &\ggg  59.39 & \multicolumn{1}{c}{70.03} & \multicolumn{1}{c}{\1 52.72} & \multicolumn{1}{c}{\1 31.36} & \ggg 17.68& \multicolumn{1}{c}{31.64} & \multicolumn{1}{c}{\5 21.64} & \multicolumn{1}{c}{\4 19.80} & \4 18.42 & \multicolumn{1}{c}{29.15} & \multicolumn{1}{c}{\5 15.45} & \multicolumn{1}{c}{\3 12.58} & \3 10.70\\  
\multicolumn{1}{|l|}{}                   &                   &         MAIL       & \multicolumn{1}{c}{99.46} & \multicolumn{1}{c}{\1 98.19} & \multicolumn{1}{c}{\1 97.38} &   \1 97.78    & \multicolumn{1}{c}{99.17} & \multicolumn{1}{c}{\1 97.38} & \multicolumn{1}{c}{\1 95.67} & \1 96.26 & \multicolumn{1}{c}{74.72} & \multicolumn{1}{c}{\ggg 66.63} & \multicolumn{1}{c}{\g 62.76} & \ggg 50.04 & \multicolumn{1}{c}{66.52} & \multicolumn{1}{c}{\3 51.00} & \multicolumn{1}{c}{\1 30.50} & \3 29.09& \multicolumn{1}{c}{30.94} & \multicolumn{1}{c}{\6 20.38} & \multicolumn{1}{c}{\6 16.98} & \6 15.60 & \multicolumn{1}{c}{28.48} & \multicolumn{1}{c}{\7 13.32} & \multicolumn{1}{c}{\6 9.89} & \7 7.18\\  
\multicolumn{1}{|l|}{}                   &                   &         WMMR       & \multicolumn{1}{c}{99.44} & \multicolumn{1}{c}{\1 98.11} & \multicolumn{1}{c}{\1 97.19} & \1 97.60    & \multicolumn{1}{c}{99.06} & \multicolumn{1}{c}{\1 97.07} & \multicolumn{1}{c}{\1 95.22} & \1 95.83 & \multicolumn{1}{c}{70.41} & \multicolumn{1}{c}{\ggg 60.90} & \multicolumn{1}{c}{\ggg 56.48} & \ggg 51.44 & \multicolumn{1}{c}{54.06} & \multicolumn{1}{c}{\ggg 41.00} & \multicolumn{1}{c}{\g 25.25} &\g 24.82 & \multicolumn{1}{c}{31.90} & \multicolumn{1}{c}{\5 20.82} & \multicolumn{1}{c}{\4 18.96} & \5 17.31 & \multicolumn{1}{c}{30.61} & \multicolumn{1}{c}{\6 14.28} & \multicolumn{1}{c}{\4 11.75} & \4 9.81\\  
\multicolumn{1}{|l|}{}                   &                   &        RAA         & \multicolumn{1}{c}{99.40} & \multicolumn{1}{c}{\1 98.14} & \multicolumn{1}{c}{\1 97.32} &  \1 97.60    & \multicolumn{1}{c}{99.13} & \multicolumn{1}{c}{\1 97.05} & \multicolumn{1}{c}{\2 95.05} & \2 95.44 & \multicolumn{1}{c}{87.58} & \multicolumn{1}{c}{\1 77.84} & \multicolumn{1}{c}{\1 72.10} & \3 70.30 & \multicolumn{1}{c}{71.61} & \multicolumn{1}{c}{\1 53.56} & \multicolumn{1}{c}{\1 31.99} & \1 31.27& \multicolumn{1}{c}{31.91} & \multicolumn{1}{c}{\5 21.75} & \multicolumn{1}{c}{\3 20.08} & \3 18.75 & \multicolumn{1}{c}{30.13} & \multicolumn{1}{c}{\4 15.90} & \multicolumn{1}{c}{\3 12.83} &\3 10.72\\ \cline{3-27}  
\rowcolor{lightgray}
\multicolumn{1}{|l|}{\cellcolor{white}}                   &       \cellcolor{white}            &         \textbf{Ours}        & \multicolumn{1}{c}{\textbf{99.51}} & \multicolumn{1}{c}{\textbf{98.60}} & \multicolumn{1}{c}{\textbf{97.87}} &  \textbf{98.16}    & \multicolumn{1}{c}{99.12} & \multicolumn{1}{c}{\textbf{97.56}} & \multicolumn{1}{c}{\textbf{96.11}} & \textbf{96.55} & \multicolumn{1}{c}{88.00} & \multicolumn{1}{c}{\textbf{77.94}} & \multicolumn{1}{c}{\textbf{72.26}} & \textbf{73.03} & \multicolumn{1}{c}{73.90} & \multicolumn{1}{c}{\textbf{54.43}} & \multicolumn{1}{c}{\textbf{32.64}} & \textbf{32.71}& \multicolumn{1}{c}{\textbf{36.05}} & \multicolumn{1}{c}{\textbf{25.76}} & \multicolumn{1}{c}{\textbf{22.83}} & \textbf{21.47} & \multicolumn{1}{c}{\textbf{34.11}} & \multicolumn{1}{c}{\textbf{19.55}} & \multicolumn{1}{c}{\textbf{15.12}} & \textbf{13.70}\\ \cline{2-27}

\multicolumn{1}{|l|}{}                   & \multirow{7}{*}{\rotatebox{90}{30\%}} &       ST         & \multicolumn{1}{c}{99.07} & \multicolumn{1}{c}{91.12} & \multicolumn{1}{c}{83.58} &   85.97   & \multicolumn{1}{c}{99.12} & \multicolumn{1}{c}{59.61} & \multicolumn{1}{c}{20.58} & 29.17 & \multicolumn{1}{c}{\textbf{84.75}} & \multicolumn{1}{c}{33.23} & \multicolumn{1}{c}{8.60} & 8.05 & \multicolumn{1}{c}{\textbf{85.14}} & \multicolumn{1}{c}{14.29} & \multicolumn{1}{c}{0.20} &0.03 & \multicolumn{1}{c}{28.19} & \multicolumn{1}{c}{14.94} & \multicolumn{1}{c}{11.91} & 10.76 & \multicolumn{1}{c}{28.80} & \multicolumn{1}{c}{3.97} & \multicolumn{1}{c}{1.68} &1.41 \\ 
\multicolumn{1}{|l|}{}                   &                   &       AT           & \multicolumn{1}{c}{99.41} & \multicolumn{1}{c}{\1 98.07} & \multicolumn{1}{c}{\1 97.14} &  \1 97.55   & \multicolumn{1}{c}{98.95} & \multicolumn{1}{c}{\1 96.90} & \multicolumn{1}{c}{\1 95.12} & \1 95.80 & \multicolumn{1}{c}{86.53} & \multicolumn{1}{c}{\2 76.10} & \multicolumn{1}{c}{\2 70.56} & \2 71.59 & \multicolumn{1}{c}{71.53} & \multicolumn{1}{c}{\2 51.01} & \multicolumn{1}{c}{\2 30.78} &\1 30.86 & \multicolumn{1}{c}{30.19} & \multicolumn{1}{c}{\4 20.77} & \multicolumn{1}{c}{\3 18.72} & \3 17.58 & \multicolumn{1}{c}{28.42} & \multicolumn{1}{c}{\5 14.57} & \multicolumn{1}{c}{\3 11.93} &\4 10.15 \\ 
\multicolumn{1}{|l|}{}                   &                   &        GAIRAT        & \multicolumn{1}{c}{99.37} & \multicolumn{1}{c}{\1 98.06} & \multicolumn{1}{c}{\1 97.27} &   \1  97.65   & \multicolumn{1}{c}{99.02} & \multicolumn{1}{c}{\1 97.01} & \multicolumn{1}{c}{\2 94.90} & \1 95.65 & \multicolumn{1}{c}{73.92} & \multicolumn{1}{c}{\ggg 63.94} & \multicolumn{1}{c}{\g 61.20} & \ggg 58.65 & \multicolumn{1}{c}{63.39} & \multicolumn{1}{c}{\3 49.31} & \multicolumn{1}{c}{\2 31.06} &\ggg 15.88 & \multicolumn{1}{c}{29.99} & \multicolumn{1}{c}{\4 20.64} & \multicolumn{1}{c}{\3 18.69} & \3 17.27 & \multicolumn{1}{c}{27.33} & \multicolumn{1}{c}{\5 15.00} & \multicolumn{1}{c}{\4 11.64} & \4 9.94\\  
\multicolumn{1}{|l|}{}                   &                   &         MAIL        & \multicolumn{1}{c}{99.44} & \multicolumn{1}{c}{\1 98.12} & \multicolumn{1}{c}{\1 97.41} & \1  97.56    & \multicolumn{1}{c}{99.13} & \multicolumn{1}{c}{\1 97.33} & \multicolumn{1}{c}{\1 95.48} & \1 96.11 & \multicolumn{1}{c}{68.58} & \multicolumn{1}{c}{\ggg 60.40} & \multicolumn{1}{c}{\ggg 56.75} & \ggg 44.55 & \multicolumn{1}{c}{63.48} & \multicolumn{1}{c}{\g 46.45} & \multicolumn{1}{c}{\3 29.04} & \ggg 17.50 & \multicolumn{1}{c}{29.28} & \multicolumn{1}{c}{\6 19.12} & \multicolumn{1}{c}{\5 16.09} & \6 14.71 & \multicolumn{1}{c}{27.23} & \multicolumn{1}{c}{\7 12.99} & \multicolumn{1}{c}{\6 9.43} &\7 6.89\\  
\multicolumn{1}{|l|}{}                   &                   &         WMMR       & \multicolumn{1}{c}{99.40} & \multicolumn{1}{c}{\1 98.07} & \multicolumn{1}{c}{\1 97.15} &  \1  97.63    & \multicolumn{1}{c}{99.05} & \multicolumn{1}{c}{\1 97.03} & \multicolumn{1}{c}{\1 94.98} & \1 95.75 & \multicolumn{1}{c}{69.49} & \multicolumn{1}{c}{\ggg 60.25} & \multicolumn{1}{c}{\ggg 56.11} & \ggg 50.76 & \multicolumn{1}{c}{53.57} & \multicolumn{1}{c}{\ggg 40.10} & \multicolumn{1}{c}{\g 23.97} &\ggg 21.27  & \multicolumn{1}{c}{30.20} & \multicolumn{1}{c}{\5 19.46} & \multicolumn{1}{c}{\4 17.70} & \4 16.17 & \multicolumn{1}{c}{29.02} & \multicolumn{1}{c}{\6 13.98} & \multicolumn{1}{c}{\4 11.27} &\5 9.20\\  
\multicolumn{1}{|l|}{}                   &                   &        RAA       & \multicolumn{1}{c}{99.40} & \multicolumn{1}{c}{\1 98.07} & \multicolumn{1}{c}{\1 97.20} &  \1  97.53     & \multicolumn{1}{c}{99.04} & \multicolumn{1}{c}{\1 97.00} & \multicolumn{1}{c}{\1 95.11} & \2 95.62 & \multicolumn{1}{c}{80.75} & \multicolumn{1}{c}{\3 72.27} & \multicolumn{1}{c}{\3 66.87} & \9 66.11 & \multicolumn{1}{c}{64.93} & \multicolumn{1}{c}{\3 48.97} & \multicolumn{1}{c}{\3 28.68} & \3 29.17& \multicolumn{1}{c}{30.58} & \multicolumn{1}{c}{\4 21.08} & \multicolumn{1}{c}{\2 19.19} & \3 17.98 & \multicolumn{1}{c}{29.08} & \multicolumn{1}{c}{\5 15.18} & \multicolumn{1}{c}{\3 12.28} &\4 10.38\\ \cline{3-27}  
\rowcolor{lightgray}
\multicolumn{1}{|l|}{\cellcolor{white}}                   &        \cellcolor{white}           &         \textbf{Ours}       & \multicolumn{1}{c}{\textbf{99.47}} & \multicolumn{1}{c}{\textbf{98.51}} & \multicolumn{1}{c}{\textbf{97.82}} &  \textbf{98.12}     & \multicolumn{1}{c}{\textbf{99.24}} & \multicolumn{1}{c}{\textbf{97.45}} & \multicolumn{1}{c}{\textbf{95.98}} & \textbf{96.63} & \multicolumn{1}{c}{86.76} & \multicolumn{1}{c}{\textbf{76.29}} & \multicolumn{1}{c}{\textbf{71.15}} & \textbf{71.60} & \multicolumn{1}{c}{71.52} & \multicolumn{1}{c}{\textbf{52.90}} & \multicolumn{1}{c}{\textbf{32.34}} & \textbf{32.58} & \multicolumn{1}{c}{\textbf{34.58}} & \multicolumn{1}{c}{\textbf{24.18}} & \multicolumn{1}{c}{\textbf{21.05}} & \textbf{20.11} & \multicolumn{1}{c}{\textbf{33.80}} & \multicolumn{1}{c}{\textbf{19.35}} & \multicolumn{1}{c}{\textbf{14.84}} &\textbf{13.70}\\ \cline{2-27} 

\multicolumn{1}{|l|}{}                   & \multirow{7}{*}{\rotatebox{90}{40\%}} &        ST       & \multicolumn{1}{c}{97.43} & \multicolumn{1}{c}{81.71} & \multicolumn{1}{c}{73.74} &    75.11    & \multicolumn{1}{c}{97.58} & \multicolumn{1}{c}{50.69} & \multicolumn{1}{c}{16.99} & 22.82 & \multicolumn{1}{c}{\textbf{79.83}} & \multicolumn{1}{c}{28.98} & \multicolumn{1}{c}{6.84} & 6.34 & \multicolumn{1}{c}{\textbf{78.37}} & \multicolumn{1}{c}{12.44} & \multicolumn{1}{c}{0.02} & 0.00& \multicolumn{1}{c}{26.99} & \multicolumn{1}{c}{12.97} & \multicolumn{1}{c}{10.29} &9.56  & \multicolumn{1}{c}{27.11} & \multicolumn{1}{c}{4.32} & \multicolumn{1}{c}{1.91} & 1.67\\ 
\multicolumn{1}{|l|}{}                   &                   &          AT      & \multicolumn{1}{c}{99.33} & \multicolumn{1}{c}{\1 97.81} & \multicolumn{1}{c}{\1 96.90} &    \1 97.22   & \multicolumn{1}{c}{98.89} & \multicolumn{1}{c}{\1 96.82} & \multicolumn{1}{c}{\1 94.63} &\1  95.46 & \multicolumn{1}{c}{85.14} & \multicolumn{1}{c}{\1 74.66} & \multicolumn{1}{c}{\2 69.57} & \2 69.44 & \multicolumn{1}{c}{66.68} & \multicolumn{1}{c}{\1 50.25} & \multicolumn{1}{c}{\1 30.76} & \3 30.40& \multicolumn{1}{c}{28.18} & \multicolumn{1}{c}{\4 20.20} & \multicolumn{1}{c}{\3 18.52} & \3 17.13 & \multicolumn{1}{c}{27.09} & \multicolumn{1}{c}{\5 13.90} & \multicolumn{1}{c}{\3 11.32} & \3 9.85\\ 
\multicolumn{1}{|l|}{}                   &                   &          GAIRAT     & \multicolumn{1}{c}{99.26} & \multicolumn{1}{c}{\1 97.70} & \multicolumn{1}{c}{\1 96.78} &  \1  97.03     & \multicolumn{1}{c}{98.85} & \multicolumn{1}{c}{\1 96.65} & \multicolumn{1}{c}{\2 94.36} & \1 95.24 & \multicolumn{1}{c}{70.19} & \multicolumn{1}{c}{\ggg 60.32} & \multicolumn{1}{c}{\ggg 55.63} & \ggg 49.99 & \multicolumn{1}{c}{62.43} & \multicolumn{1}{c}{\1 48.98} & \multicolumn{1}{c}{\1 30.85} &\ggg 15.61 & \multicolumn{1}{c}{28.42} & \multicolumn{1}{c}{\4 19.67} & \multicolumn{1}{c}{\3 17.88} & \3 16.75 & \multicolumn{1}{c}{25.70} & \multicolumn{1}{c}{\4 14.49} & \multicolumn{1}{c}{\3 11.55} & \3 9.85\\ 
\multicolumn{1}{|l|}{}                   &                   &         MAIL      & \multicolumn{1}{c}{99.24} & \multicolumn{1}{c}{\1 97.74} & \multicolumn{1}{c}{\1 96.87} & \2 96.84      & \multicolumn{1}{c}{98.79} & \multicolumn{1}{c}{\1 96.94} & \multicolumn{1}{c}{\1 94.93} & \1 95.26 & \multicolumn{1}{c}{64.62} & \multicolumn{1}{c}{\ggg 57.14} & \multicolumn{1}{c}{\ggg 53.99} & \gg 42.62 & \multicolumn{1}{c}{53.75} & \multicolumn{1}{c}{\g 41.72} & \multicolumn{1}{c}{\g 26.41} & \ggg 15.29& \multicolumn{1}{c}{27.68} & \multicolumn{1}{c}{\6 17.73} & \multicolumn{1}{c}{\5 15.77} & \6 13.91 & \multicolumn{1}{c}{25.82} & \multicolumn{1}{c}{\7 11.76} & \multicolumn{1}{c}{\6 9.00} &\6 6.92 \\  
\multicolumn{1}{|l|}{}                   &                   &        WMMR      & \multicolumn{1}{c}{99.35} & \multicolumn{1}{c}{\1 97.84} & \multicolumn{1}{c}{\1 96.89} &  \1 97.22      & \multicolumn{1}{c}{98.91} & \multicolumn{1}{c}{\1 96.70} & \multicolumn{1}{c}{\1 94.58} & \1 95.30 & \multicolumn{1}{c}{67.12} & \multicolumn{1}{c}{\ggg 57.91} & \multicolumn{1}{c}{\ggg 54.26} & \ggg 49.42 & \multicolumn{1}{c}{47.34} & \multicolumn{1}{c}{\ggg 36.84} & \multicolumn{1}{c}{\g 21.89} & \ggg 20.17& \multicolumn{1}{c}{28.65} & \multicolumn{1}{c}{\5 19.15} & \multicolumn{1}{c}{\4 17.54} & \4 15.94 & \multicolumn{1}{c}{27.70} & \multicolumn{1}{c}{\5 13.69} & \multicolumn{1}{c}{\3 11.33} &\4 9.67\\ 
\multicolumn{1}{|l|}{}                   &                   &        RAA     & \multicolumn{1}{c}{99.40} & \multicolumn{1}{c}{\1 98.02} & \multicolumn{1}{c}{\1 97.14} &  \1 97.23      & \multicolumn{1}{c}{98.96} & \multicolumn{1}{c}{\1 96.91} & \multicolumn{1}{c}{\1 95.02} & \1 95.29 & \multicolumn{1}{c}{72.65} & \multicolumn{1}{c}{\g 65.66} & \multicolumn{1}{c}{\g 61.10} & \ggg 59.58 & \multicolumn{1}{c}{53.59} & \multicolumn{1}{c}{\g 43.28} & \multicolumn{1}{c}{\g 26.02} & \g 25.86 & \multicolumn{1}{c}{28.25} & \multicolumn{1}{c}{\4 19.97} & \multicolumn{1}{c}{\3 18.58} & \3 17.36 & \multicolumn{1}{c}{27.68} & \multicolumn{1}{c}{\5 14.42} & \multicolumn{1}{c}{\3 11.70} &\3 10.21\\ \cline{3-27} 
\rowcolor{lightgray}
\multicolumn{1}{|l|}{\cellcolor{white}}                   &       \cellcolor{white}            &        \textbf{Ours}       & \multicolumn{1}{c}{\textbf{99.41}} & \multicolumn{1}{c}{\textbf{98.38}} & \multicolumn{1}{c}{\textbf{97.68}} &   \textbf{97.97 }    & \multicolumn{1}{c}{\textbf{99.09}} & \multicolumn{1}{c}{\textbf{97.25}} & \multicolumn{1}{c}{\textbf{95.62}} & \textbf{96.19} & \multicolumn{1}{c}{85.10} & \multicolumn{1}{c}{\textbf{75.35}} & \multicolumn{1}{c}{\textbf{69.89}} & \textbf{69.72} & \multicolumn{1}{c}{66.75} & \multicolumn{1}{c}{\textbf{50.70}} & \multicolumn{1}{c}{\textbf{31.52}} &\textbf{32.47} & \multicolumn{1}{c}{\textbf{33.65}} & \multicolumn{1}{c}{\textbf{23.25}} & \multicolumn{1}{c}{\textbf{20.76}} & \textbf{19.46} & \multicolumn{1}{c}{\textbf{31.69}} & \multicolumn{1}{c}{\textbf{18.47}} & \multicolumn{1}{c}{\textbf{14.16}} & \textbf{12.81}\\ \hline

\end{tabular}
}
\vspace{-0.6em}
\caption{\small \it Testing accuracy (\%) of seven methods on MNIST, CIFAR-10, and CIFAR-100 with different levels of symmetric and asymmetric noisy. The best results are shown in bold. We color the performance of all adversarial training methods on three different attacks. The performance gap between current method and \ORAT~are shown in \mybox[fill=green!30]{green}: $\leq$2\%; \mybox[fill=yellow!30]{yellow}: (2\%,5\%]; \mybox[fill=orange!40]{orange}: (5\%,10\%]; \mybox[fill=blue!30]{blue}: $>$10\%. According to results, \ORAT~outperforms all adversarial training methods on all attack settings. `Na' represents Natural. `FG' represents FGSM.}
\label{tab:general_performance_1}
\vspace{-1.5em}
\end{table*}
We report the overall results (accuracy on the testing sets) in Table \ref{tab:general_performance_1}. 
First, when there is no noise in the datasets, \ORAT~outperforms the AT method in three different attack settings on all datasets by a significant margin (0, 2.14\%]. This is probably because the original datasets may contain outliers, as verified in \cite{sanyal2021benign}.

Second, under the symmetric noise setting, \ORAT~outperforms all compared methods in all attack scenarios in general. 
We observe the performance gap between noise and noise-free settings is small on MNIST because it is a relatively smaller dataset and LeNet achieves almost perfect performance.
However, on a larger dataset such as CIFAR-10 and using a more complex neural network, the outliers have a stronger impact.
The instance-reweighted AT methods such as GAIRAT, MAIL, and MAIL are inferior to AT and \ORAT~in 20\%, 30\%, 40\% settings under $\epsilon=2/255$, because they allocate the loss of the most outliers with higher weight during the training. 
For CIFAR-100, which is larger and more challenging than CIFAR-10 and the training network ResNet is also larger than Small-CNN. In this case, we find that outliers have a large influence on robust training. Compared to noise-free settings, the performance of AT is reduced by half under the noise settings. However, \ORAT~still outperforms all baselines in all settings. The performance gap between the baselines and \ORAT~is more than 2\% and even can be achieved near 10\% (in 30\% symmetric noise, MAIL, FGSM, and $\epsilon$=2/255 settings). 
Furthermore, we find RAA is more robust than other compared methods but vulnerable to adversarial samples than \ORAT. As we discussed before, their label correction approach may correct the wrong labels but inevitably introduces more extra noisy labels during training.



Third, for asymmetric noise settings, we can get similar observations. Note that \ORAT~can even outperform the baselines over 10\% on the CIFAR-10 dataset.  All adversarial methods outperform ST in all attack settings since ST does not have a mechanism to handle adversarial samples. In general, increasing the noise percentage, we can find there is a decreasing trend in the performance among all methods. On the other hand, increasing the value of $\epsilon$ will also decrease all model performance. Even a small amount of label noise causes classifiers to have significant adversarial errors.

\begin{figure*}[t]
    \centering
    \begin{minipage}{0.49\linewidth}
		\centering
            \scriptsize{
            \scalebox{0.65}{
            \begin{tabular}{|c|cc|cc|cc|}
            \hline
            \multirow{2}{*}{\diagbox{Defense}{Data}} & \multicolumn{2}{c|}{MNIST}    & \multicolumn{2}{c|}{CIFAR-10}    & \multicolumn{2}{c|}{CIFAR-100}    \\ \cline{2-7} 
                              & \multicolumn{1}{c|}{$\epsilon\!\!=\!\!0.1$} & $\epsilon\!\!=\!\!0.2$ & \multicolumn{1}{c|}{$\epsilon\!\!=\!\!2/255$} & $\epsilon\!\!=\!\!8/255$ & \multicolumn{1}{c|}{$\epsilon\!\!=\!\!2/255$} & $\epsilon\!\!=\!\!8/255$ \\ \hline
                             ST & \multicolumn{1}{c|}{84.13} & 23.78 & \multicolumn{1}{c|}{23.54} & 10.41 & \multicolumn{1}{c|}{10.22} & 1.72 \\ 
                             AT & \multicolumn{1}{c|}{\1 96.73} & \1 93.23 & \multicolumn{1}{c|}{\1 44.56} & \1 23.06 & \multicolumn{1}{c|}{\3 17.11} & \3 9.89 \\ 
                             GAIRAT & \multicolumn{1}{c|}{\1 96.78} & \1 93.57 & \multicolumn{1}{c|}{\3 40.82 } & \3 20.57 & \multicolumn{1}{c|}{\3 16.54} & \3 9.97 \\ 
                             MAIL & \multicolumn{1}{c|}{\1 97.01} & \1 93.80 & \multicolumn{1}{c|}{\6 40.39} & \3 18.95 & \multicolumn{1}{c|}{\6 13.67} & \3 7.29 \\ 
                             WMMR & \multicolumn{1}{c|}{\1 96.88} & \1 93.60 & \multicolumn{1}{c|}{\3 43.52} & \3 20.68  & \multicolumn{1}{c|}{\3 15.69} & \3 8.73 \\ 
                             RAA & \multicolumn{1}{c|}{\1 96.98} & \1 93.62 & \multicolumn{1}{c|}{\1 45.07} & \1 22.29 & \multicolumn{1}{c|}{\3 17.70} & \3 9.78 \\ \hline
                             \rowcolor{lightgray}
                             \textbf{Ours} & \multicolumn{1}{c|}{\textbf{97.67}} & \textbf{95.05} & \multicolumn{1}{c|}{\textbf{45.59}} & \textbf{23.35} & \multicolumn{1}{c|}{\textbf{19.74}} & \textbf{12.22} \\ \hline
            \end{tabular}
            }
            }
            \vspace{-0.3cm}
            \captionof{table}{\small \em Testing accuracy (\%) on AutoAttack.} 
            \label{tab:aa_performance}

        \setlength\tabcolsep{5pt}
            \scriptsize{
            \scalebox{0.64}{
            \begin{tabular}{|c|c|cccc|cccc|}
\hline
\multirow{2}{*}{Noise} & \multirow{2}{*}{Defense} & \multicolumn{4}{c|}{MNIST ($\epsilon=0.1$)}                                                    & \multicolumn{4}{c|}{CIFAR-100 ($\epsilon=2/255$)}                                                    \\ \cline{3-10} 
                  &                   & \multicolumn{1}{c}{Na} & \multicolumn{1}{c}{FG} & \multicolumn{1}{c}{PGD} & CW & \multicolumn{1}{c}{Na} & \multicolumn{1}{c}{FG} & \multicolumn{1}{c}{PGD} & CW  \\ \hline \hline
\multirow{2}{*}{10\%} 
                  &         \makecell{AT w/o}           & \multicolumn{1}{c}{ 98.91} & \multicolumn{1}{c}{\1 98.08} & \multicolumn{1}{c}{\1 97.61} & \1 97.55 & \multicolumn{1}{c}{ 29.81} & \multicolumn{1}{c}{\3 21.09} & \multicolumn{1}{c}{\3 19.59} & \3 18.23 \\ \cline{2-10} 
                  &        \cellcolor{lightgray}\textbf{Ours}           & \multicolumn{1}{c}{\cellcolor{lightgray}\textbf{99.52}} & \multicolumn{1}{c}{\cellcolor{lightgray}\textbf{98.45}} & \multicolumn{1}{c}{\cellcolor{lightgray}\textbf{97.78}} & \cellcolor{lightgray}\textbf{97.79} & \multicolumn{1}{c}{\cellcolor{lightgray}\textbf{35.76}} & \multicolumn{1}{c}{\cellcolor{lightgray}\textbf{25.72}} & \multicolumn{1}{c}{\cellcolor{lightgray}\textbf{22.27}} & \cellcolor{lightgray}\textbf{21.28} \\ \hline
\multirow{2}{*}{30\%} 
                  &          \makecell{AT w/o}           & \multicolumn{1}{c}{ 97.82} & \multicolumn{1}{c}{\1 96.97} & \multicolumn{1}{c}{\1 96.47} & \1 96.35 & \multicolumn{1}{c}{ 24.18} & \multicolumn{1}{c}{\3 19.17} & \multicolumn{1}{c}{\3 17.31} & \3 16.71 \\ \cline{2-10} 
                 &       \cellcolor{lightgray}  \textbf{Ours}          & \multicolumn{1}{c}{\cellcolor{lightgray}\textbf{99.55}} & \multicolumn{1}{c}{\cellcolor{lightgray}\textbf{98.30}} & \multicolumn{1}{c}{\cellcolor{lightgray}\textbf{97.51}} & \cellcolor{lightgray}\textbf{97.53} & \multicolumn{1}{c}{\cellcolor{lightgray}\textbf{31.27}} & \multicolumn{1}{c}{\cellcolor{lightgray}\textbf{23.81}} & \multicolumn{1}{c}{\cellcolor{lightgray}\textbf{21.35}} & \cellcolor{lightgray}\textbf{19.59} \\ \hline
\end{tabular}
            }
            }
            \vspace{-0.3cm}
            \captionof{table}{\small \em Testing accuracy (\%) of AT w/o and \ORAT~on symmetric noisy data.
            } 
            \label{tab:partial_ATw/o}
        
	    \end{minipage}
    \hfill
    \begin{minipage}{0.5\linewidth}
    \centering
	    \setlength\tabcolsep{5pt}
            \scriptsize{
            \scalebox{0.57}{
            \begin{tabular}{|c|c|cccc|cccc|}
\hline
\multirow{2}{*}{Noise} & \multirow{2}{*}{Defense} & \multicolumn{4}{c|}{MNIST ($\epsilon=0.1$)}                                                    & \multicolumn{4}{c|}{CIFAR-100 ($\epsilon=2/255$)}                                                    \\ \cline{3-10} 
                  &                   & \multicolumn{1}{c}{Na} & \multicolumn{1}{c}{FG} & \multicolumn{1}{c}{PGD} & CW & \multicolumn{1}{c}{Na} & \multicolumn{1}{c}{FG} & \multicolumn{1}{c}{PGD} & CW \\ \hline\hline
\multirow{7}{*}{\rotatebox{90}{ \small \makecell{40\%  Symmetric Noise \ \ \ \ \  }} } &         ST          & \multicolumn{1}{c}{\makecell{98.38\\(0.18)}} & \multicolumn{1}{c}{\makecell{77.47\\(5.79)}} & \multicolumn{1}{c}{\makecell{59.31\\(11.13)}} & \makecell{51.47\\(12.79)} & \multicolumn{1}{c}{\makecell{21.53\\(0.71)}} & \multicolumn{1}{c}{\makecell{13.95\\(0.16)}} & \multicolumn{1}{c}{\makecell{12.12\\(0.28)}} & \makecell{10.12\\(0.28)} \\ \cline{2-10} 
                  &         AT          & \multicolumn{1}{c}{\makecell{96.60\\(0.87)}} & \multicolumn{1}{c}{\3\makecell{95.11\\(1.17)}} & \multicolumn{1}{c}{\3\makecell{94.36\\(1.31)}} & \3\makecell{94.16\\(1.37)} & \multicolumn{1}{c}{\makecell{21.82\\(0.30)}} & \multicolumn{1}{c}{\6\makecell{18.14\\(0.27)}} & \multicolumn{1}{c}{\3\makecell{17.02\\(0.29)}} & \3\makecell{15.62\\(0.32)} \\ \cline{2-10} 
                  &          GAIRAT	         & \multicolumn{1}{c}{\makecell{98.04\\(0.11)}} & \multicolumn{1}{c}{\1\makecell{96.60\\(0.08)}} & \multicolumn{1}{c}{\1\makecell{95.98\\(0.13)}} & \1\makecell{95.79\\(0.12)} & \multicolumn{1}{c}{\makecell{21.58\\(0.30)}} & \multicolumn{1}{c}{\6\makecell{17.94\\(0.36)}} & \multicolumn{1}{c}{\6\makecell{16.88\\(0.36)}} & \3\makecell{15.50\\(0.33)} \\ \cline{2-10} 
                  &         MAIL          & \multicolumn{1}{c}{\makecell{98.06\\(0.17)}} & \multicolumn{1}{c}{\1\makecell{96.88\\(0.16)}} & \multicolumn{1}{c}{\1\makecell{96.31\\(0.12)}} & \1\makecell{96.11\\(0.10)} & \multicolumn{1}{c}{\makecell{17.52\\(0.74)}} & \multicolumn{1}{c}{\6\makecell{14.31\\(0.43)}} & \multicolumn{1}{c}{\6\makecell{13.46\\(0.40)}} & \6\makecell{12.09\\(0.37)} \\ \cline{2-10} 
                  &         WMMR          & \multicolumn{1}{c}{\makecell{98.07\\(0.14)}} & \multicolumn{1}{c}{\1\makecell{96.69\\(0.10)}} & \multicolumn{1}{c}{\1\makecell{96.04\\(0.08)}} & \1\makecell{95.85\\(0.08)} & \multicolumn{1}{c}{\makecell{20.64\\(0.29)}} & \multicolumn{1}{c}{\6\makecell{17.14\\(0.33)}} & \multicolumn{1}{c}{\6\makecell{16.10\\(0.30)}} & \6\makecell{14.76\\(0.27)} \\ \cline{2-10} 
                  &         RAA          & \multicolumn{1}{c}{\makecell{98.58\\(0.11)}} & \multicolumn{1}{c}{\1\makecell{96.82\\(0.08)}} & \multicolumn{1}{c}{\1\makecell{95.99\\(0.10)}} & \1\makecell{95.91\\(0.09)} & \multicolumn{1}{c}{\makecell{22.21\\(0.32)}} & \multicolumn{1}{c}{\6\makecell{18.27\\(0.32)}} & \multicolumn{1}{c}{\3\makecell{17.22\\(0.30)}} & \3\makecell{15.85\\(0.33)} \\ \cline{2-10} 
                 &          \cellcolor{lightgray} \textbf{Ours}         & \multicolumn{1}{c}{\cellcolor{lightgray} \textbf{\makecell{99.34\\(0.03)}}} & \multicolumn{1}{c}{\cellcolor{lightgray} \textbf{\makecell{98.02\\(0.07)}}} & \multicolumn{1}{c}{\cellcolor{lightgray} \textbf{\makecell{97.19\\(0.08)}}} & \cellcolor{lightgray} \textbf{\makecell{97.18\\(0.07)}} & \multicolumn{1}{c}{\cellcolor{lightgray} \textbf{\makecell{30.12\\(0.47)	}}} & \multicolumn{1}{c}{\cellcolor{lightgray} \textbf{\makecell{23.62\\(0.42)}}} & \multicolumn{1}{c}{\cellcolor{lightgray} \textbf{\makecell{21.48\\(0.49)}}} & \cellcolor{lightgray} \textbf{\makecell{20.02\\(0.41)}} \\ \hline
\end{tabular}
            }
            }
        \vspace{-0.1cm}
            \captionof{table}{\small \em Mean and standard deviation (in parentheses) of testing accuracy (\%) across 10 random runs. 
            } 
            \label{tab:Stability}
        
	\end{minipage}

\end{figure*}

We also report testing accuracy for AA on all datasets with 20\% symmetric noise in Table \ref{tab:aa_performance}. We can find our method outperforms all compared methods even if using a strong attack strategy. {In addition, we compare AT w/o with our method and report performance in Table \ref{tab:partial_ATw/o}. Comparing Table \ref{tab:general_performance_1} and Table \ref{tab:partial_ATw/o} under the same setting, it is evident that simply removing outliers from training data cannot significantly improve AT performance, and our method still achieves the best performance.
The stability evaluation results of each method with 10 random runs are shown in Table \ref{tab:Stability}, where we use 40\% symmetric noisy data as an example. Comparing Table \ref{tab:general_performance_1} and Table \ref{tab:Stability} under the same setting, it is clear that the performance gap becomes larger when we report scores by using mean and standard deviation, and our method shows a stable and stronger ability in handling outliers and adversarial attacks.}

\begin{figure*}[t]
    \centering
    \begin{minipage}{0.5\linewidth}
		\centering
        \begin{subfigure}[t]{0.49\linewidth}
            \includegraphics[width=\linewidth]{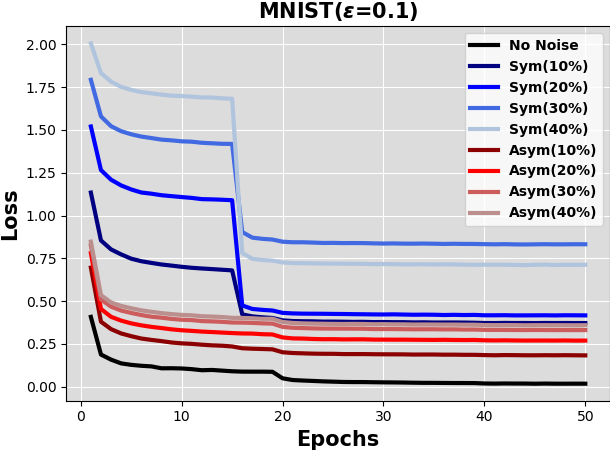}
        \end{subfigure}%
        \hfill%
        \begin{subfigure}[t]{0.49\linewidth}
            \includegraphics[width=\linewidth]{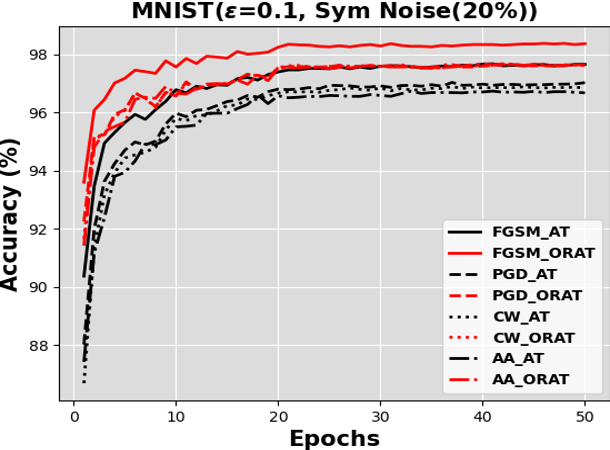}
        \end{subfigure}%
        
        \begin{subfigure}[t]{0.49\linewidth}
            \includegraphics[width=\linewidth]{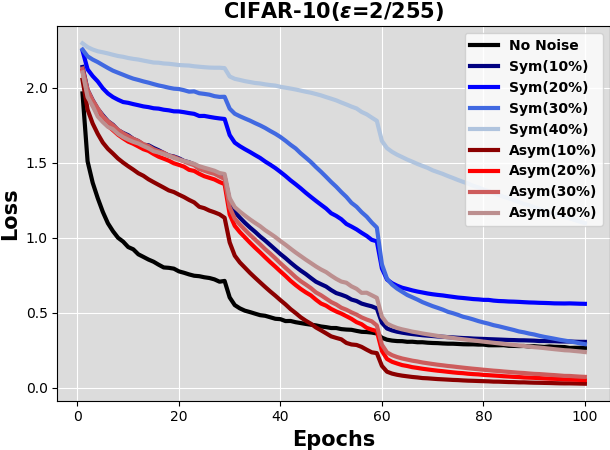}
        \end{subfigure}
        \hfill%
        \begin{subfigure}[t]{0.49\linewidth}
            \includegraphics[width=\linewidth]{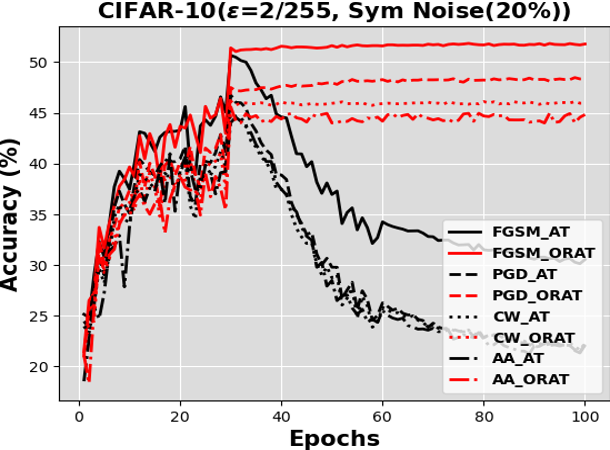}
        \end{subfigure}
        
        \begin{subfigure}[t]{0.49\linewidth}
            \includegraphics[width=\linewidth]{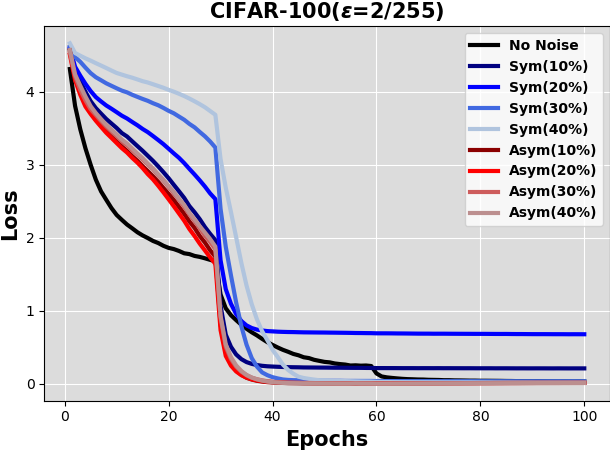}
        \end{subfigure}
        \hfill%
        \begin{subfigure}[t]{0.49\linewidth}
            \includegraphics[width=\linewidth]{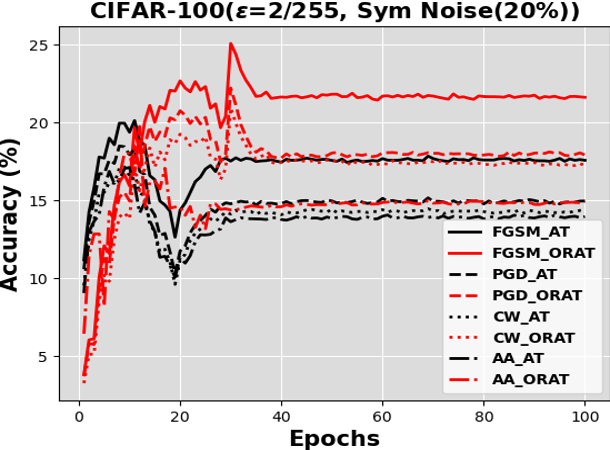}
        \end{subfigure}
        \vspace{-1.7em}
        \caption{\small \em  The tendency curves of training adversarial loss and test accuracy on three datasets. 
        }
        \vspace{-1.5em}
        \label{fig:loss_acc_1}
	    \end{minipage}
    \hfill
    \begin{minipage}{0.49\linewidth}
	    \centering
        \begin{subfigure}[t]{0.49\linewidth}
            \includegraphics[width=\linewidth]{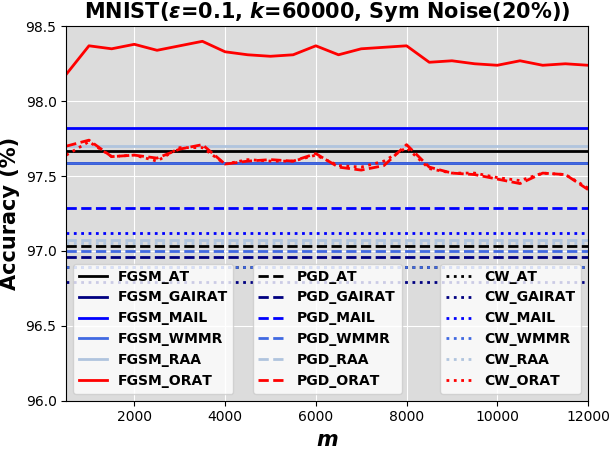}
        \end{subfigure}%
        \hfill%
        \begin{subfigure}[t]{0.49\linewidth}
            \includegraphics[width=\linewidth]{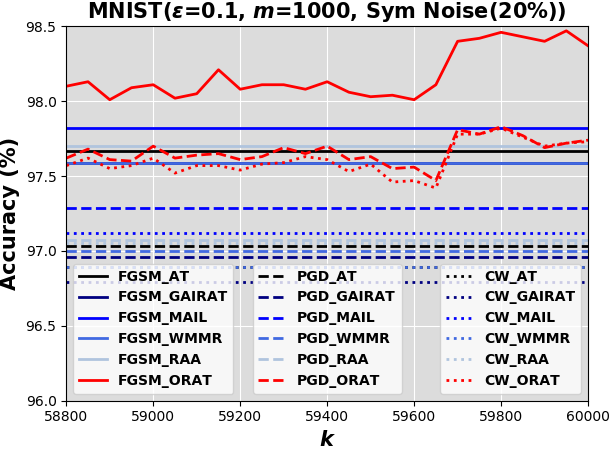}
        \end{subfigure}%
        
        \begin{subfigure}[t]{0.49\linewidth}
            \includegraphics[width=\linewidth]{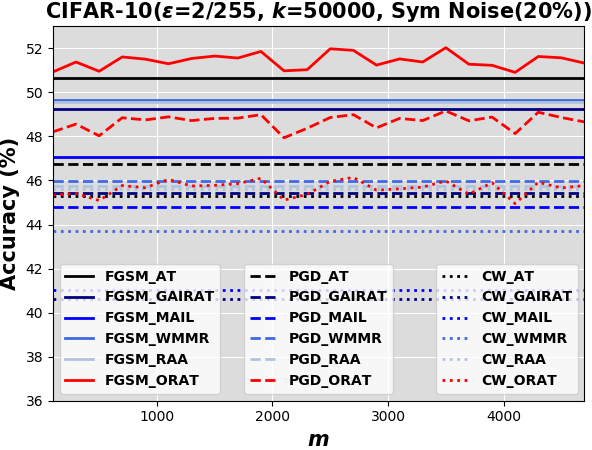}
        \end{subfigure}%
        \hfill%
        \begin{subfigure}[t]{0.49\linewidth}
            \includegraphics[width=\linewidth]{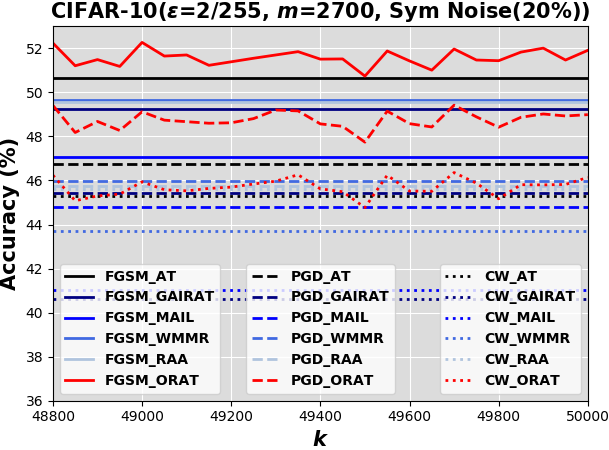}
        \end{subfigure}%
        
        \begin{subfigure}[t]{0.49\linewidth}
            \includegraphics[width=\linewidth]{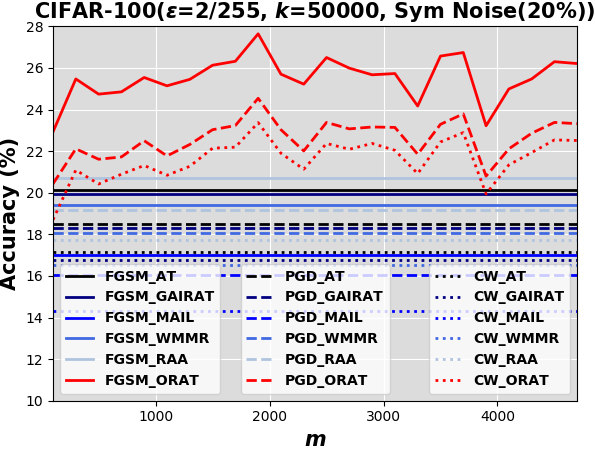}
        \end{subfigure}%
        \hfill%
        \begin{subfigure}[t]{0.49\linewidth}
            \includegraphics[width=\linewidth]{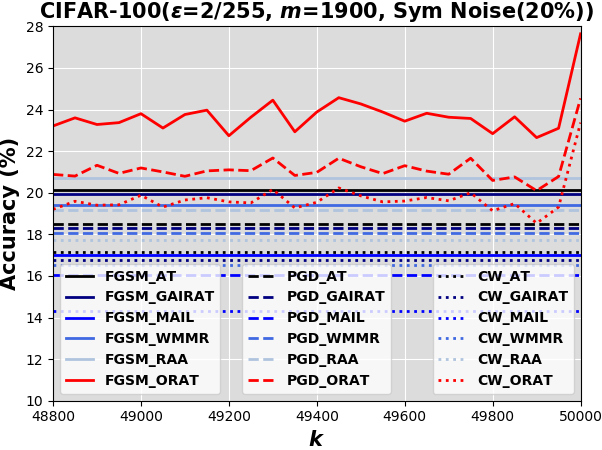}
        \end{subfigure}%
        \vspace{-0.6em}
        \caption{\small \em  Effect of $k$ and $m$ on the test accuracy of \ORAT~on three datasets.}
        \vspace{-1.5em}
        \label{fig:k_m}
	\end{minipage}

	
\end{figure*}




\textbf{Convergence and Robustness Tendency}. We show the tendency curves of the training loss on all datasets with different noise when using Algorithm \ref{alg:atrr} in the left panel of Figure \ref{fig:loss_acc_1}. The sharp drops in the curves correspond to decreases in training learning rate. We can observe a steady decrease in the training loss on all datasets with the increase of training epochs when training against adversarial samples, which supports that Algorithm \ref{alg:atrr} can effectively optimize and solve Eq. (\ref{eq:theorem1}) even if it is a non-smooth loss. 
We also compare the robust accuracy of AT and \ORAT~by using four attack strategies on all datasets with 20\% symmetric noise in Figure \ref{fig:loss_acc_1} right panel. It is clear that our method outperforms the original AT method in all settings. In general, these plots illustrate that we can consistently reduce the value of the objective function of \ORAT, thus producing an increasingly robust classifier.

\textbf{Effect of $k$ and $m$}.
We study how the choices of $k$ and $m$ affect the performance of \ORAT~with two types of experiments on all datasets under 20\% symmetric noise setting, together with those from other defense methods. In the first set of experiments, we fix $k$ to the total number of training samples and run the algorithm with different values of $m$. The results are plotted in Figure \ref{fig:k_m} (left panel). We can see that there is a clear range of $m$ with better performance than all compared methods. In the second set of experiments, we fix $m$ with the best performance from the first set of experiments and run \ORAT~with different values of $k$. The results are shown in Figure \ref{fig:k_m} (right panel). Note that there is also a range of $k$ with better performance, in particular, the optimal value of $k$ is less than the number of total training samples. Similar trends are observed on all datasets.  

\vspace{-1em}
\section{Conclusion}
In this work, we introduce the outlier robust adversarial training  (\ORAT), which considers both outliers in training data and adversarial attacks in the model training. We provide an optimizing algorithm and analyze the theoretical aspects of \ORAT. Empirical results showed the effectiveness and robustness of \ORAT~on three benchmark datasets. In the future, we will study the 
optimization error or 
convergence rate of our proposed learning algorithm. We will also evaluate our method on large datasets and large deep neural networks. {Although achieving fairness \citep{ju2023improving, hu2022distributionally} is not a goal of this work, we have found that our method can benefit the minority subgroup of data (see Figure \ref{fig:interpretation} right panel). Studying fairness with \ORAT~is an interesting direction in the future.} 
Furthermore, we plan to design an efficient method to automatically determine the hyperparameters $k$ and $m$ during the training.


\bibliography{Camera-ready}

\newpage
\onecolumn
\appendix
\numberwithin{equation}{section}
\numberwithin{theorem}{section}
\numberwithin{figure}{section}
\numberwithin{table}{section}
\renewcommand{\thesection}{{\Alph{section}}}
\renewcommand{\thesubsection}{\Alph{section}.\arabic{subsection}}
\renewcommand{\thesubsubsection}{\Roman{section}.\arabic{subsection}.\arabic{subsubsection}}

\def\p{\mathbf{p}}
\def\v{\mathbf{v}}
\def\u{\mathbf{u}}

\begin{center}
\textbf{\Large Appendix for ``Outlier Robust Adversarial Training"}
\end{center}

\section{Related Works}\label{sec:related_works}

\textbf{Traditional Robust Learning}. Training accurate machine learning models in the presence of noisy data is of great practical importance \citep{sukhbaatar2015training}. However, a degradation in the  performance of classification models is inevitable when there are outliers in the training data. 
To combat outliers, the traditional robust learning methods are designed from four directions. 1) The \textit{label correction methods} \citep{wang2018iterative} improve the quality of the raw labels by correcting wrong labels into correct ones. 
However, it requires an extra clean dataset or potentially expensive detection process to estimate the outliers. 
2) The \textit{loss correction methods} \citep{han2020training} improve the robustness by modifying the loss function based on an estimated noise transition matrix that defines the probability of mislabeling one class with another.  However, these methods are sensitive to the noise transition matrix, which is also hard to be estimated. 
3) The \textit{refined training strategies} such as Co-teaching \citep{yu2019does}, MentorNet \citep{jiang2018mentornet} are robust to outliers. These studies all rely  on an auxiliary network for sample weighting or learning supervision, which is hard to adapt and tune. 
4) Some simpler and arguably generic \textit{robust loss functions} are also designed for robust learning. For example, 
a recent work \cite{hu2020learning} proposed AoRR loss, which 
can mitigate the influence of the outliers if their proportion in training data is known. 
Furthermore, Some smoothing methods are proposed in \cite{chaudhari2019entropy} and have been proven to be effective in solving the problems of label and data noise. However, none of these methods are related to adversarial robust learning.

\textbf{Adversarial Robust Learning}. The omnipotent DNN models are surprisingly vulnerable to adversarial examples \citep{goodfellow2014explaining},
which can easily mislead a DNN model to make erroneous
predictions. To mitigate
this issue, the adversarial training (AT) \citep{madry2018towards} is first proposed as one of the most effective robust learning methodologies against adversarial attacks. 
To improve adversarial robustness, instance-reweighted AT methods are studied by considering the unequal importance of the adversarial data in several recent works. 
Intuitively, the samples assigned a low weight to correspond to samples on which the classifier is already
sufficiently robust. Specifically, the reweight mechanism in WMMR \citep{zeng2021adversarial} and MAIL \citep{liu2021probabilistic} is based on the multi-class probabilistic
margin of the model outputs \citep{zhang2019defending}. The reweighting method in work GAIRAT \citep{zhang2021geometry} identifies non-robust (easily be-attacked) data by estimating how many steps the PGD method needs to attack natural data successfully. The most recent work BiLAW \citep{robustnesslearning} uses a validation set to learn weights based on bi-level optimization and meta-learning. 
The most significant assumption in these works is that the natural dataset is clean. However, the performance of the model based on these methods will be degraded if the training dataset contains outliers.
In \cite{sanyal2021benign}, the authors identified label noise as one of the causes of adversarial vulnerability. 
However, no defense methods are proposed to solve this problem.
The work \cite{zhu2021understanding} empirically studies the efficacy of AT for mitigating the effect of label noise in training data. 
However, their proposed annotator algorithm is based on the label correction strategy, which inevitably introduces more extra noisy labels due to the bottleneck of the classifier. In \cite{dong2020adversarial}, the authors proposed an adversarial distributional training. They focus on the distribution shift of adversarial samples but they do not consider the outliers problem. Several works \citep{augustin2020adversarial, bitterwolf2020certifiably} connect adversarial robustness to out-of-distribution (OOD) problems. However, they are in different settings from ours because the notion of outliers is different from OOD points. Dong et al. \citep{dongexploring} also discuss the effect of the label noise. However, they focus on the memorization effect in AT. We focus on outlier problems in AT. Huang et al. \citep{huang2020self} created a self-adaptive method for robust learning with noisy labels or adversarial examples, but did not consider both present simultaneously. This is also mentioned in \cite{zhu2021understanding}.


\section{Explicit Forms of (sub)gradients}\label{appendix:explicit_forms}
From Eq.(\ref{eq:theorem1}), we have $\hat{\mathcal{L}}(f_{\theta},\lambda, \hat{\lambda}):= \frac{k-m}{n}\lambda+\frac{n-m}{n}\hat{\lambda}-[\hat{\lambda}-[\ell(f_{\theta}(\tilde{\x}_i), y_i)-\lambda]_+]_+$.
Denote $\mathbb{I}_{[a]}$ as an indicator function with $\mathbb{I}_{[a]}=1$ if $a$ is true and 0 otherwise. Then we can get
\begin{equation*}
    \begin{aligned}
    \partial_\theta \hat{\mathcal{L}}(f_{\theta^{(t)}},\lambda^{(t)}, \hat{\lambda}^{(t)})= \partial \ell(f_{\theta^{(t)}}(\tilde{\x}_i),y_i)\cdot \mathbb{I}_{[\hat{\lambda}^{(t)}>[\ell(f_{\theta^{(t)}}(\tilde{\x}_i),y_i)-\lambda^{(t)}]_+]} \cdot \mathbb{I}_{[\ell(f_{\theta^{(t)}}(\tilde{\x}_i),y_i)>\lambda^{(t)}]},
    \end{aligned}
\end{equation*}
\begin{equation*}
    \begin{aligned}
    \partial_\lambda \hat{\mathcal{L}}(f_{\theta^{(t)}},\lambda^{(t)}, \hat{\lambda}^{(t)})= \frac{k-m}{n} - \mathbb{I}_{[\hat{\lambda}^{(t)}>[\ell(f_{\theta^{(t)}}(\tilde{\x}_i),y_i)-\lambda^{(t)}]_+]}\cdot \mathbb{I}_{[\ell(f_{\theta^{(t)}}(\tilde{\x}_i),y_i)>\lambda^{(t)}]},
    \end{aligned}
\end{equation*}
\begin{equation*}
    \begin{aligned}
    \partial_{\hat{\lambda}} \hat{\mathcal{L}}(f_{\theta^{(t)}},\lambda^{(t)}, \hat{\lambda}^{(t)})= \frac{n-m}{n}-\mathbb{I}_{[\hat{\lambda}^{(t)}>[\ell(f_{\theta^{(t)}}(\tilde{\x}_i),y_i)-\lambda^{(t)}]_+]}.
    \end{aligned}
\end{equation*}

\section{Proofs}\label{sec:proofs}
\subsection{Proof of Theorem \ref{theorem1}}\label{sec:proof_theorem1}
Denote $[a]_+ = \max\{0,a\}$ as the hinge function. First, we introduce two Lemmas as follows,
\begin{lemma}\label{lemma1} \citep{hu2020learning}
For a set of real numbers $S=\{s_1,\cdots, s_n\}$, $s_i\in \mathbb{R}$, and $s_{[i]}$ represents the $i$-th largest value after sorting the elements in $S$, we have 
\begin{equation*}
\begin{aligned}
    \sum_{i=1}^k s_{[i]}=\min_{\lambda\in\mathbb{R}}\Big\{k\lambda+\sum_{i=1}^n [s_i-\lambda]_+\Big\}.
\end{aligned}
\end{equation*}
Furthermore, { $s_{[k]}\in\arg\min_{\lambda\in\mathbb{R}}\{k\lambda+\sum_{i=1}^n [s_i-\lambda]_+\}$}.
\end{lemma}
\begin{proof}
We know $\sum_{i=1}^k s_{[i]}$ is the solution of 
\begin{equation*}
    \max_\p \p^\top S, ~\text{s.t.}~ \p^\top \mathbf{1}=k, \mathbf{0}\leq\p\leq \mathbf{1}.
\end{equation*}
We apply Lagrangian to this equation and get
\begin{equation*}
    L = -\p^\top S-\v^\top\p+\u^\top(\p-1)+\lambda(\p^\top\mathbf{1}-k)
\end{equation*}
where $\u\geq\mathbf{0}$, $\v\geq\mathbf{0}$ and $\lambda\in\mathbb{R}$ are Lagrangian multipliers. Taking its derivative w.r.t. $\p$ and set it to 0, we have $\v=\u-S+\lambda\mathbf{1}$. Substituting it back into the Lagrangian, we get 
\begin{equation*}
    \min_{\u,\lambda} \u^\top \mathbf{1}+k\lambda, ~\text{s.t.}~ \u\geq\mathbf{0}, \u+\lambda\mathbf{1}-S\geq 0.
\end{equation*}
This means 
\begin{equation}
    \sum_{i=1}^k s_{[i]}=\min_{\lambda}\Big\{k\lambda+\sum_{i=1}^n[s_i-\lambda]_+\Big\}.
\label{eq:topk_convex}
\end{equation}
Furthermore, we can see that $\lambda=s_{[k]}$ is always one optimal solution for Eq.(\ref{eq:topk_convex}). So
\begin{equation*}
    s_{[k]}\in\arg\min_{\lambda}\Big\{k\lambda+\sum_{i=1}^n[s_i-\lambda]_+\Big\}.
\end{equation*}
\end{proof}

\begin{lemma}\label{lemma2}
For a set of real numbers $S=\{s_1,\cdots, s_n\}$, $s_i\in \mathbb{R}$, we have 
\begin{equation*}
\begin{aligned}
    \sum_{i=m+1}^ns_{[i]} = \max_{\lambda\in\mathbb{R}}\Big\{(n-m)\lambda-\sum_{i=1}^{n}[\lambda-s_i]_+\Big\}.
\end{aligned}
\end{equation*}
Furthermore, { $s_{[m]}\in\arg\max_{\lambda\in\mathbb{R}}\{(n-m)\lambda-\sum_{i=1}^{n}[\lambda-s_i]_+\}$}.
\end{lemma}
\begin{proof}
\begin{equation*}
    \begin{aligned}
    \sum_{i=m+1}^ns_{[i]}&=\sum_{i=1}^{n}s_{i}-\sum_{i=1}^{m}s_{[i]}\\
    &=\sum_{i=1}^{n}s_{i} - \min_{\lambda}\bigg\{m\lambda+\sum_{i=1}^n[s_i-\lambda]_+\bigg\}\\
    &=-\min_{\lambda}\bigg\{-\sum_{i=1}^{n}(s_{i}-\lambda) -(n-m)\lambda +\sum_{i=1}^n[s_i-\lambda]_+\bigg\}\\
    &=-\min_{\lambda}\bigg\{-(n-m)\lambda+\sum_{i=1}^{n}[\lambda-s_i]_+\bigg\}\\
    &=\max_{\lambda}\bigg\{(n-m)\lambda-\sum_{i=1}^{n}[\lambda-s_i]_+\bigg\}
    \end{aligned}.
\end{equation*}
The second equation holds because of Lemma \ref{lemma1}. The fourth equation holds because the fact of $[a]_+-a=[-a]_+$. Furthermore, we can see that $\lambda=s_{[m]}$ is always one optimal solution. So
\begin{equation*}
    s_{[m]}\in\arg\max_{\lambda\in\mathbb{R}}\Big\{(n-m)\lambda-\sum_{i=1}^{n}[\lambda-s_i]_+\Big\}.
\end{equation*}
\end{proof}

\begin{theorem} (Theorem \ref{theorem1} restated)
Suppose $\lambda\in\mathbb{R}$, $\hat{\lambda}\in\mathbb{R}$, then Eq.(\ref{eq:atrr}) is equivalent to
\begin{equation}
\begin{aligned}
    \min_{\theta,\lambda} \max_{\hat{\lambda}} \ \ &\frac{1}{k-m}\sum_{i=1}^n\Big[ \frac{k-m}{n}\lambda+\frac{n-m}{n}\hat{\lambda}   -[\hat{\lambda}-[\ell(f_{\theta}(\tilde{\x}_i), y_i)-\lambda]_+]_+\Big]\\
    \mbox{s.t.} \ \ & \tilde{\x}_{i} = \arg\max_{\tilde{\x}\in \mathcal{B}_\epsilon(\x_{i})} \ell(f_{\theta}(\tilde{\x}), y_i)
\end{aligned}
\end{equation}
Furthermore, $\hat{\lambda}>\lambda$, when the optimal solution is achieved.
\end{theorem}
\begin{proof}
To extract the sum of ($m,k$)-ranked range individual losses, we can first select a subset, which contains the bottom $n-m$ losses from the ranked list of $L\left(\{(\x_j,y_j)\}_{j=1}^n\right)$. Then we select top-($k-m$) individual losses from this subset as the finalized ($m,k$)-ranked range. Therefore, 
We sum the bottom $n-m$ individual losses as follows,
\begin{equation*}
    \begin{aligned}
    \sum_{i=m+1}^n \ell(f_{\theta}(\tilde{\x}_{[i]}),y_{[i]}) = \min_{q} \sum_{i=1}^n q_i \ell(f_{\theta}(\tilde{\x}_{[i]}),y_{[i]}) \ \ \ \text{s.t.} \ q_i\in\{0,1\}, \ ||q||_0 =n-m,
    \end{aligned}
\end{equation*}
where $q=\{q_1,\cdots,q_n\}\in \{0,1\}^n$, and $q_i$ is an indicator. When $q_i=0$, it indicates that the $i$-th individual loss is not included in the objective function. Otherwise, the objective function should include this individual loss. Next, we sum the top-($k-m$) individual losses from the bottom $n-m$ individual losses as follows,
\begin{equation}
    \begin{aligned}
    &\min_{q}\sum_{i=1}^{k-m} (q \ell(f_\theta(\tilde{\x}),y))_{[i]} \ \ \ \text{s.t.} \ q_i\in\{0,1\}, \ ||q||_0 =n-m \\
    =&\min_{\lambda,q} (k-m)\lambda+\sum_{i=1}^n [q_i\ell(f_\theta(\tilde{\x}_{i}),y_{i})-\lambda]_+\ \ \ \text{s.t.} \ q_i\in\{0,1\}, \ ||q||_0 =n-m\\
    =&\min_{\lambda,q}(k-m)\lambda+\sum_{i=1}^n q_i[\ell(f_\theta(\tilde{\x}_{i}),y_{i})-\lambda]_+\ \ \ \text{s.t.} \ q_i\in [0,1], \ ||q||_0 =n-m\\
    =&\min_{\lambda} (k-m)\lambda + \sum_{i=m+1}^{n}[[\ell(f_\theta(\tilde{\x}),y)-\lambda]_+]_{[i]}\\
    =& \min_{\lambda} (k-m)\lambda + \max_{\hat{\lambda}}\Big\{(n-m)\hat{\lambda}-\sum_{i=1}^n [\hat{\lambda}-[\ell(f_\theta(\tilde{\x}_{i}),y_{i})-\lambda]_+]_+\Big\},
    \end{aligned}
\label{eq:aorr_newform}
\end{equation}
where $q\ell(f_\theta(\tilde{\x}),y)=\{q_1\ell(f_\theta(\tilde{\x}_1),y_1),\cdots,q_n\ell(f_\theta(\tilde{\x}_n),y_n)\}$. The first equation holds because of Lemma \ref{lemma1}. Since $q_i\ell(f_\theta(\tilde{\x}_i),y_i)\geq 0$,  we know the optimal $\lambda^*\geq 0$ from Lemma \ref{lemma1}. If $q_i=0$, $[q_i\ell(f_\theta(\tilde{\x}_i),y_i)-\lambda^*]_+=0=q_i[\ell(f_\theta(\tilde{\x}_i),y_i)-\lambda^*]_+$. If $q_i=1$, $[q_i\ell(f_\theta(\tilde{\x}_i),y_i)-\lambda^*]_+=[\ell(f_\theta(\tilde{\x}_i),y_i)-\lambda^*]_+=q_i[\ell(f_\theta(\tilde{\x}_i),y_i)-\lambda^*]_+$. Thus the second equation holds. It should be mentioned that the discrete indicator $q_i$ can be replaced by a continue one, which means $q_i\in [0,1]$. The third equation holds because we take the optimal $q^*$ into the objective function and remove the constraints. The fourth equation can be obtained by applying Lemma \ref{lemma2}.

Therefore, 
\begin{equation}
    \begin{aligned}
    &\min_\theta\frac{1}{k-m}\sum_{i=m+1}^k \ell(f_\theta(\tilde{x}_{[i]}),y_{[i]})\\
    &=\min_\theta\frac{1}{k-m}\bigg\{\min_{\lambda} (k-m)\lambda + \max_{\hat{\lambda}}\Big\{(n-m)\hat{\lambda}-\sum_{i=1}^n [\hat{\lambda}-[\ell(f_\theta(\tilde{\x}_{i}),y_{i})-\lambda]_+]_+\Big\}\bigg\}\\
    &=\min_{\theta,\lambda} \max_{\hat{\lambda}} \frac{1}{k-m}\sum_{i=1}^n\Big[ \frac{k-m}{n}\lambda+\frac{n-m}{n}\hat{\lambda}-[\hat{\lambda}-[\ell(f_{\theta}(\tilde{\x}_i), y_i)-\lambda]_+]_+\Big].
    \end{aligned}
\end{equation}
Furthermore, according to Lemma \ref{lemma1} and \ref{lemma2}, we know the optimal $\lambda^*$ and $\hat{\lambda}^*$ can be obtained at the top-$k$ and top-$m$ values of loss $\ell$, respectively. Since $m < k$, we have $\lambda^*<\hat{\lambda}^*$. Therefore, $\hat{\lambda}>\lambda$, when the optimal solution is achieved.
\end{proof}

\subsection{Proof of Theorem \ref{theorem2}}\label{proof_consistency}
To prove Theorem \ref{theorem2}, we first introduce the calibration function as follows,
\begin{definition}
(Calibration function). \citep{awasthi2021calibration} Given a hypothesis set $\mathcal{H}$, we define the calibration function $\delta_{\max}$ for a pair of losses ($\ell_1,\ell_2$) as follows: for all $\x\in \mathcal{X}$, $\eta\in[0,1]$ and $\tau>0$,
\begin{equation}
    \begin{aligned}
    \delta_{\max} (\tau,\x,\eta)={\inf}_{f\in\mathcal{H}}\{\mathcal{C}_{\ell_1}(f,\x,\eta)-\mathcal{C}_{\ell_1,\mathcal{H}}^*(\x,\eta)|\mathcal{C}_{\ell_2}(f,\x,\eta)-\mathcal{C}_{\ell_2,\mathcal{H}}^*(\x,\eta)\geq \tau\}.
    \end{aligned}
\end{equation}
\end{definition}
The calibration function gives the maximal $\delta$ satisfying the calibration condition (Definition \ref{eq:calibration}). The following proposition is an important result from \cite{steinwart2007compare}.
\begin{proposition}
\citep{steinwart2007compare}. Given a hypothesis set $\mathcal{H}$, loss $\ell_1$ is $\mathcal{H}$-calibrated with respect to $\ell_2$ if and only if its calibration function $\delta_{\max}$ satisfies $\delta_{\max}(\tau,\x,\eta)>0$ for all $\x\in\mathcal{X}$, $\eta\in[0,1]$, and $\tau>0$.
\end{proposition}
Next, we define the adversarial loss of $f\in\mathcal{H}$ at ($\x,y$) as 
\begin{equation}
    \begin{aligned}
    \tilde{\ell}_s(f,\x,y)=\sup_{\tilde{\x}\in \mathcal{B}_\epsilon(\x)} \ell_s(yf(\tilde{\x})).
    \end{aligned}
\label{eq:adv_loss}
\end{equation}
The above naturally motivates supremum-based surrogate losses that are commonly used to optimize the adversarial 0/1 loss \citep{goodfellow2014explaining,madry2018towards,zhang2019theoretically}. When $\ell_s$ is non-increasing, the following equality holds \citep{yin2019rademacher}:
\begin{equation}
    \begin{aligned}
    \sup_{\tilde{\x}\in \mathcal{B}_\epsilon(\x)} \ell_s(yf(\tilde{\x})) = \ell_s \Big(\inf_{\tilde{\x}\in \mathcal{B}_\epsilon(\x)} yf(\tilde{\x})\Big).
    \end{aligned}
\label{eq:adv_loss_relax}
\end{equation}
Therefore, the adversarial 0/1 loss $\widetilde{\ell}_0$ has the equivalent form
\begin{equation}
    \begin{aligned}
    \widetilde{\ell}_0(f,\x,y):=\sup_{\tilde{\x}\in \mathcal{B}_\epsilon(\x)}\mathds{1}_{yf(\tilde{\x})\leq 0}=\mathds{1}_{\underset{\tilde{\x}\in \mathcal{B}_\epsilon(\x)}{\inf}yf(\tilde{\x})\leq 0 }.
    \end{aligned}
\end{equation}
In this paper, we aim to characterize surrogate losses $\ell_1$ satisfying $\mathcal{H}$-calibration (Definition \ref{eq:calibration}) with $\ell_2 = \widetilde{\ell}_0$ and for the hypothesis sets $\mathcal{H}$ which are regular for adversarial calibration.

For convenience, let 
$\underline{M}(f,\x,\epsilon):=\inf_{\tilde{\x}\in \mathcal{B}_\epsilon(\x)}f(\tilde{\x})$ and $\overline{M}(f,\x,\epsilon):=-\inf_{\tilde{\x}\in \mathcal{B}_\epsilon(\x)}-f(\tilde{\x})=\sup_{\tilde{\x}\in \mathcal{B}_\epsilon(\x)}f(\tilde{\x})$. Then we provide three useful Lemmas as follows,

\begin{lemma}\label{lemma:calibration_ref}
(\cite{awasthi2021calibration}, Lemma 28).
Let $\mathcal{H}$ be a symmetric hypothesis set, $\ell$ be a surrogate loss function, and $\mathcal{X}_2=$\{$\x\in \mathcal{X}$: there exists $f'\in \mathcal{H}$ such that $\underline{M}(f',\x,\epsilon)>0$ \}. If $\mathcal{X}_2=\emptyset$, any loss $\ell$ is $\mathcal{H}$-calibrated with respect to $\widetilde{\ell}_0$. If $\mathcal{X}_2\neq\emptyset$, then $\ell$ is $\mathcal{H}$-calibrated with respect to $\widetilde{\ell}_0$ if and only if for any $\x\in\mathcal{X}_2$,
\begin{equation}
    \begin{aligned}
    \inf_{f\in \mathcal{H}:\underline{M}(f,\x,\epsilon)\leq 0 \leq \overline{M}(f,\x,\epsilon)}\mathcal{C}_{\ell}(f,\x,\frac{1}{2})&>\inf_{f\in \mathcal{H}}\mathcal{C}_{\ell} (f,\x,\frac{1}{2}), and \\
    \inf_{f\in \mathcal{H}:\underline{M}(f,\x,\epsilon)\leq 0}\mathcal{C}_{\ell}(f,\x,\eta)&>\inf_{f\in \mathcal{H}}\mathcal{C}_{\ell} (f,\x,\eta), \ \ \forall \eta\in (\frac{1}{2},1], and\\
    \inf_{f\in \mathcal{H}:0 \leq \overline{M}(f,\x,\epsilon)}\mathcal{C}_{\ell}(f,\x,\eta)&>\inf_{f\in \mathcal{H}}\mathcal{C}_{\ell} (f,\x,\eta), \ \ \forall \eta\in [0,\frac{1}{2}).
    \end{aligned}
\end{equation}
\end{lemma}

\begin{lemma}\label{lemma:consistency}
(\cite{awasthi2021calibration}, Theorem 23 and Theorem 24). Let $\mathcal{H}$ be a symmetric hypothesis set consisting of the family of all measurable functions $\mathcal{H}_{all}$, $\phi$ be a non-increasing margin-based loss, and $\tilde{\phi}(f,\x,y)=\sup_{\tilde{\x}\in \mathcal{B}_\epsilon(\x)} \phi(yf(\tilde{\x}))$. If $\tilde{\phi}$ is $\mathcal{H}$-calibrated with respect to $\widetilde{\ell}_0$, then $\tilde{\phi}$ is $\mathcal{H}$-consistent with respect to $\widetilde{\ell}_0$ for all distributions $\mathcal{D}$ over $\mathcal{X}\times \mathcal{Y}$ that satisfy: $\mathcal{R}_{\widetilde{\ell}_0, \mathcal{H}}^* =0$ and there exists $f^*\in\mathcal{H}$ such that $\mathcal{R}_{\phi}(f^*)=\mathcal{R}_{\phi,\mathcal{H}_{all}}^*<+\infty$.
\end{lemma}

The proofs of the above two Lemmas can be found in  \cite{awasthi2021calibration}.

\begin{lemma}\label{lemma:lambda_opt}
Let $\mathcal{H}$ be a symmetric hypothesis set and $f\in \mathcal{H}$. Suppose $0\leq \lambda^*<\hat{\lambda}^*$, $\nu>\min\{\hat{\lambda}^*, \mathcal{R}^*_{\ell,\mathcal{H}}\}$, $\ell(yf(\x))\geq 0 \ \forall \x$,  and $\lambda^*$ is bounded, then $\lambda^*< \ell(0)$. 
\end{lemma}
\begin{proof}
Based on the definition of $(f_0^*, \lambda^*, \hat{\lambda}^*)=\arg \inf_{f,\lambda} \sup_{\hat{\lambda}}\Big\{\mathbb{E}\Big[\hat{\lambda}-[\hat{\lambda}-[\ell(Yf(X))-\lambda]_+]_+ \Big] +\nu \lambda-\mu \hat{\lambda}\Big\}$.
We choose $f=0$, $\lambda = \ell(0)$ 
and $\hat{\lambda}=\hat{\lambda}^*$ there holds
\begin{equation*}
\begin{aligned}
    \nu \lambda^*-\mu \hat{\lambda}^* &\leq \mathbb{E}\Big[\hat{\lambda}^*-[\hat{\lambda}^*-[\ell(Yf_0^*(X))-\lambda^*]_+]_+ \Big] +\nu \lambda^*-\mu \hat{\lambda}^*\\
    &\leq \mathbb{E}\Big[\hat{\lambda}^*-[\hat{\lambda}^*-[\ell(0)-\ell(0)]_+]_+ \Big] +\nu\ell(0) -\mu \hat{\lambda}^*\\
    &=\nu\ell(0) -\mu \hat{\lambda}^*
\end{aligned}
\end{equation*}
Thus $\nu \lambda^*\leq \nu\ell(0)$ which shows that $\lambda^*\leq \ell(0)$.
Let $\beta = \ell(0)-\lambda$ which implies 
\begin{equation*}
\begin{aligned}
    (f_0^*, \beta^*, \hat{\lambda^*})=\arg \inf_{f,\lambda} \sup_{\hat{\lambda}}\Bigg\{\mathbb{E}\Big[\hat{\lambda}-[\hat{\lambda}-[\ell(Yf(X))+\beta-\ell(0)]_+]_+ \Big] -\nu \beta-\mu \hat{\lambda}\Bigg\}.
\end{aligned}
\end{equation*}
Let ($f_0^*, \beta^*, \hat{\lambda^*}$) be the minimizer. we have, for any $f$ and choosing $\beta=\ell(0)$, that 
\begin{equation*}
\begin{aligned}
    -\nu \beta^*-\mu \hat{\lambda}^*&\leq \mathbb{E}\Big[\hat{\lambda}^*-[\hat{\lambda}^*-[\ell(Yf(X))+\beta^*-\ell(0)]_+]_+ \Big] -\nu \beta^*-\mu \hat{\lambda}^*\\
    &\leq \mathbb{E}\Big[\hat{\lambda}^*-[\hat{\lambda}^*-[\ell(Yf(X))+\ell(0)-\ell(0)]_+]_+ \Big] -\nu\ell(0) -\mu \hat{\lambda}^*.
\end{aligned}
\end{equation*}
Therefore, we have 
\begin{equation*}
\begin{aligned}
    -\nu \beta^*\leq \mathbb{E}\Big[\hat{\lambda}^*-[\hat{\lambda}^*-[\ell(Yf(X))]_+]_+ \Big] -\nu\ell(0).
\end{aligned}
\end{equation*}
Since $f$ is arbitrary, $\beta^*\geq \frac{\nu-\mathbb{E}\Big[\hat{\lambda}^*-[\hat{\lambda}^*-[\ell(Yf^*(X))]_+]_+ \Big]}{\nu}$. 
Since $\ell(Yf^*(X))\geq 0$, we have
\begin{equation*}
\begin{aligned}
    0\leq\mathbb{E}\Big[\hat{\lambda}^*-[\hat{\lambda}^*-[\ell(Yf^*(X))]_+]_+ \Big] = \mathbb{E}\Big[\hat{\lambda}^*-[\hat{\lambda}^*-\ell(Yf^*(X))]_+ \Big]\leq \min\Big\{\hat{\lambda}^*, \inf_f \mathbb{E}[\ell(yf(\x))]\Big\}.
\end{aligned}
\end{equation*}
By using the assumption $\nu>\min\{\hat{\lambda}^*, \mathcal{R}^*_{\ell,\mathcal{H}}\}$=$\min\Big\{\hat{\lambda}^*, \inf_f \mathbb{E}[\ell(yf(\x))]\Big\}$, we get $\beta^*\geq \frac{\nu-\mathbb{E}\Big[\hat{\lambda}^*-[\hat{\lambda}^*-[\ell(Yf^*(X))]_+]_+ \Big]}{\nu}>0$. Consequently, the above arguments show that $0\leq \lambda^*=\ell(0)-\beta^*<\ell(0)$ if $\nu>\min\Big\{\hat{\lambda}^*, \inf_f \mathbb{E}[\ell(yf(\x))]\Big\}$.
\end{proof}

\begin{theorem} (Theorem \ref{theorem2} restated)
Let $\mathcal{H}$ be a symmetric hypothesis set consisting of the family of all measurable functions $\mathcal{H}_{all}$, suppose  $\nu>\min\{\hat{\lambda}^*, \mathcal{R}^*_{\ell,\mathcal{H}}\}$, $0\leq\lambda^*<\hat{\lambda}^*$, $\lambda^*$ and $\hat{\lambda}^*$ are bounded, and $\ell$ is a non-negative, continuous, and non-increasing margin-based loss.

(i) Then $\tilde{\phi}_{\ORAT}$ is $\mathcal{H}$-calibrated with respect to $\widetilde{\ell}_0$. 

(ii) Furthermore, $\tilde{\phi}_{\ORAT}$ is $\mathcal{H}$-consistent with respect to $\widetilde{\ell}_0$ for all distributions $\mathcal{D}$ over $\mathcal{X}\times \mathcal{Y}$ that satisfy: $\mathcal{R}_{\widetilde{\ell}_0, \mathcal{H}}^* =0$ and there exists $f^*\in\mathcal{H}$ such that $\mathcal{R}_{\phi_{\ORAT}}(f^*)=\mathcal{R}_{\phi_{\ORAT},\mathcal{H}_{all}}^*<+\infty$.
\end{theorem}
\begin{proof}
Below we will prove the theorem using Lemma \ref{lemma:calibration_ref} which is from \cite{awasthi2021calibration}. Recall that, from the definition of $\phi_{\ORAT}(t)$ in Eq.(\ref{eq:two_phi}), $\ell(t)$ is a continuous and non-increasing function, and $\lambda^*$ and $\hat{\lambda}^*$ are bounded, we can conclude $\phi_{\ORAT}(t)$ is bounded, continuous, non-increasing. 

By Lemma \ref{lemma:calibration_ref}, if $\mathcal{X}_2=\emptyset$, $\tilde{\phi}_{\ORAT}$ is $\mathcal{H}$-calibrated with respect to $\widetilde{\ell}_0$ . Consequently, it suffices to  consider the case where $\mathcal{X}_2\neq\emptyset$. In this case, in order to show   $\tilde{\phi}_{\ORAT}$ is $\mathcal{H}$-calibrated with respect to $\widetilde{\ell}_0$,  from Lemma \ref{lemma:calibration_ref} we only need to show, $\forall \x\in \mathcal{X}_2$, that 
\begin{equation*}
    \begin{aligned}
    \inf_{f\in \mathcal{H}:\underline{M}(f,\x,\epsilon)\leq 0 \leq \overline{M}(f,\x,\epsilon)}\mathcal{C}_{\tilde{\phi}_{\ORAT}}(f,\x,\frac{1}{2})&>\inf_{f\in \mathcal{H}}\mathcal{C}_{\tilde{\phi}_{\ORAT}} (f,\x,\frac{1}{2}), and \\
    \inf_{f\in \mathcal{H}:\underline{M}(f,\x,\epsilon)\leq 0}\mathcal{C}_{\tilde{\phi}_{\ORAT}}(f,\x,\eta)&>\inf_{f\in \mathcal{H}}\mathcal{C}_{\tilde{\phi}_{\ORAT}} (f,\x,\eta), \ \ \forall \eta\in (\frac{1}{2},1], and\\
    \inf_{f\in \mathcal{H}:0 \leq \overline{M}(f,\x,\epsilon)}\mathcal{C}_{\tilde{\phi}_{\ORAT}}(f,\x,\eta)&>\inf_{f\in \mathcal{H}}\mathcal{C}_{\tilde{\phi}_{\ORAT}} (f,\x,\eta), \ \ \forall \eta\in [0,\frac{1}{2}).
    \end{aligned}
\end{equation*}

To this end, recall that, by the definition of inner $\ell_s$-risk, the inner $\tilde{\phi}_{\ORAT}$-risk is given by 
\begin{equation*}
    \begin{aligned}
    \mathcal{C}_{\tilde{\phi}_{\ORAT}}(f,x,\eta)&=\eta \tilde{\phi}_{\ORAT}(f,\x,+1)+(1-\eta)\tilde{\phi}_{\ORAT}(f,\x,-1)\\
    &= \eta\phi_{\ORAT}\Big(\inf_{\tilde{\x}\in \mathcal{B}_\epsilon(\x)} f(\tilde{\x})\Big) + (1-\eta)\phi_{\ORAT}\Big(\inf_{\tilde{\x}\in \mathcal{B}_\epsilon(\x)} -f(\tilde{\x})\Big)\\
    &=\eta\phi_{\ORAT}\Big(\underline{M}(f,\x,\epsilon)\Big)+(1-\eta)\phi_{\ORAT}\Big(-\overline{M}(f,\x,\epsilon)\Big).
    \end{aligned}
\end{equation*}
For any $\x\in \mathcal{X}_2$, let $M_{\x}=\sup_{f\in \mathcal{H}}\underline{M}(f,\x,\epsilon)>0$. Since $\mathcal{H}$ is symmetric consisting of all measurable functions, we have $-M_\x=\inf_{f\in \mathcal{H}}\overline{M}(f,\x,\epsilon)<0$. Since $\phi_{\ORAT}(\cdot)$ is continuous, for any $\x\in \mathcal{X}_2$ and $\tau>0$, there exists $f_\x^{\tau}\in \mathcal{H}$ such that $\phi_{\ORAT}(\underline{M}(f_\x^{\tau},\x,\epsilon))<\phi_{\ORAT}(M_{\x})+\tau$, $\phi_{\ORAT}(-\overline{M}(f_\x^{\tau},\x,\epsilon))<\phi_{\ORAT}(0)+\tau$, $\overline{M}(f_\x^{\tau},\x,\epsilon)\geq \underline{M}(f_\x^{\tau},\x,\epsilon)>0$,  $\underline{M}(-f_\x^{\tau},\x,\epsilon)\leq \overline{M}(-f_\x^{\tau},\x,\epsilon)=-\underline{M}(f_\x^{\tau},\x,\epsilon)<0$, and $\phi_{\ORAT}(\underline{M}(-f_\x^{\tau},\x,\epsilon))<\phi_{\ORAT}(0)+\tau$. Next we analyze three cases:

1. When $\eta=\frac{1}{2}$, since $\phi_{\ORAT}$ is non-increasing,
\begin{equation*}
    \begin{aligned}
    &\inf_{f\in \mathcal{H}:\underline{M}(f,\x,\epsilon)\leq 0 \leq \overline{M}(f,\x,\epsilon)}\mathcal{C}_{\tilde{\phi}_{\ORAT}}(f,\x,\frac{1}{2})\\
    &=\inf_{f\in \mathcal{H}:\underline{M}(f,\x,\epsilon)\leq 0 \leq \overline{M}(f,\x,\epsilon)}\frac{1}{2}\phi_{\ORAT}\Big(\underline{M}(f,\x,\epsilon)\Big)+\frac{1}{2}\phi_{\ORAT}\Big(-\overline{M}(f,\x,\epsilon)\Big)\\
    &\geq \frac{1}{2}\phi_{\ORAT}(0)+\frac{1}{2}\phi_{\ORAT}(0)\\
    &=\phi_{\ORAT}(0)\\
    &=\hat{\lambda}^*-[\hat{\lambda}^*-[\ell(0)-\lambda^*]_+]_+.
    \end{aligned} 
\end{equation*}
For any $\x\in \mathcal{X}_2$, there exists $f'\in \mathcal{H}$ such that $\underline{M}(f',\x,\epsilon)>0$ and $-\overline{M}(f',\x,\epsilon)\leq -\underline{M}(f',\x,\epsilon)<0$, we obtain
\begin{equation*}
    \begin{aligned}
    \mathcal{C}_{\tilde{\phi}_{\ORAT}}(f',\x,\frac{1}{2})=\frac{1}{2}\phi_{\ORAT}\Big(\underline{M}(f',\x,\epsilon)\Big)+\frac{1}{2}\phi_{\ORAT}\Big(-\overline{M}(f',\x,\epsilon)\Big)
    \end{aligned} 
\end{equation*}
According to Lemma \ref{lemma:lambda_opt}, we have $0\leq\lambda^*<\hat{\lambda}^*$ and $\lambda^*< \ell(0)$.  Therefore, we also analyze two cases: 
\begin{list}{(\alph{mylistcntr})}{\usecounter{mylistcntr}\itemsep 0mm}
    \item If $0<\lambda^*+\hat{\lambda}^*\leq \ell(0)$, then we have 
    \begin{equation*}
    \begin{aligned}
    \inf_{f\in \mathcal{H}:\underline{M}(f,\x,\epsilon)\leq 0 \leq \overline{M}(f,\x,\epsilon)}\mathcal{C}_{\tilde{\phi}_{\ORAT}}(f,\x,\frac{1}{2})\geq \hat{\lambda}^*-[\hat{\lambda}^*-[\ell(0)-\lambda^*]_+]_+=\hat{\lambda^*}.
    \end{aligned} 
    \end{equation*}
    On the other hand, since $\phi_{\ORAT}$ is continuous, there exists $f'\in \mathcal{H}$ and $t=\underline{M}(f',\x,\epsilon)$, then $0\leq \phi_{\ORAT}\Big(\underline{M}(f',\x,\epsilon)\Big)<\hat{\lambda}^*$. Thus, 
    \begin{equation*}
    \begin{aligned}
    \mathcal{C}_{\tilde{\phi}_{\ORAT}}(f',\x,\frac{1}{2})&=\frac{1}{2}\phi_{\ORAT}\Big(\underline{M}(f',\x,\epsilon)\Big)+\frac{1}{2}\phi_{\ORAT}\Big(-\overline{M}(f',\x,\epsilon)\Big)\\
    &\leq \frac{1}{2}\phi_{\ORAT}\Big(\underline{M}(f',\x,\epsilon)\Big)+\frac{1}{2}\hat{\lambda}^*\\
    &<\frac{1}{2}\hat{\lambda}^*+\frac{1}{2}\hat{\lambda}^*=\hat{\lambda}^*.
    \end{aligned} 
    \end{equation*}
    Therefore, for any $\x\in \mathcal{X}_2$,
    \begin{equation}
    \begin{aligned}
    \inf_{f\in \mathcal{H}}\mathcal{C}_{\tilde{\phi}_{\ORAT}}(f,\x,\frac{1}{2})\leq\mathcal{C}_{\tilde{\phi}_{\ORAT}}(f',\x,\frac{1}{2})<\hat{\lambda}^*\leq \inf_{f\in \mathcal{H}:\underline{M}(f,\x,\epsilon)\leq 0 \leq \overline{M}(f,\x,\epsilon)}\mathcal{C}_{\tilde{\phi}_{\ORAT}}(f,\x,\frac{1}{2}).
    \end{aligned}
    \label{eq:calibration_1_a}
    \end{equation}
    
    \item If $\lambda^*+\hat{\lambda}^*> \ell(0)$,
    \begin{equation*}
    \begin{aligned}
    \inf_{f\in \mathcal{H}:\underline{M}(f,\x,\epsilon)\leq 0 \leq \overline{M}(f,\x,\epsilon)}\mathcal{C}_{\tilde{\phi}_{\ORAT}}(f,\x,\frac{1}{2})\geq \hat{\lambda}^*-[\hat{\lambda}^*-[\ell(0)-\lambda^*]_+]_+=\ell(0)-\lambda^*.
    \end{aligned} 
    \end{equation*}
  On the other hand, recall both $\phi_{\ORAT}(\cdot)$ and $\ell(\cdot)$ are continuous and non-increasing and $\ell(0) > \lambda^*$ from Lemma 5. Therefore, we can find  $f'\in \mathcal{H}$ such that  $\ell(0)>\ell(\underline{M}(f',\x,\epsilon))>\lambda^*$, $\lambda^*+\hat{\lambda}^*>\ell(-\overline{M}(f',\x,\epsilon))>\ell(0)>\lambda^*$, and $\ell(\underline{M}(f',\x,\epsilon))+\ell(-\overline{M}(f',\x,\epsilon))<2\ell(0)$. 
  Consequently, there holds  
    \begin{equation*}
    \begin{aligned}
    &\mathcal{C}_{\tilde{\phi}_{\ORAT}}(f',\x,\frac{1}{2})\\
    &=\frac{1}{2}\phi_{\ORAT}\Big(\underline{M}(f',\x,\epsilon)\Big)+\frac{1}{2}\phi_{\ORAT}\Big(-\overline{M}(f',\x,\epsilon)\Big)\\
    &=\frac{1}{2}\Big[\hat{\lambda}^*-[\hat{\lambda}^*-[\ell(\underline{M}(f',\x,\epsilon))-\lambda^*]_+]_+\Big]+\frac{1}{2}\Big[\hat{\lambda}^*-[\hat{\lambda}^*-[\ell(-\overline{M}(f',\x,\epsilon))-\lambda^*]_+]_+\Big]\\
    &=\frac{1}{2}[\ell(\underline{M}(f',\x,\epsilon))-\lambda^*]+\frac{1}{2}[\ell(-\overline{M}(f',\x,\epsilon))-\lambda^*]\\
    &=\frac{1}{2}[\ell(\underline{M}(f',\x,\epsilon))+\ell(-\overline{M}(f',\x,\epsilon))]-\lambda^*\\
    &<\frac{1}{2}\times 2\ell(0)-\lambda^*=\ell(0)-\lambda^*.
    \end{aligned} 
    \end{equation*}
    Therefore, for any $\x\in \mathcal{X}_2$,
    \begin{equation}
    \begin{aligned}
    \inf_{f\in \mathcal{H}}\mathcal{C}_{\tilde{\phi}_{\ORAT}}(f,\x,\frac{1}{2})\leq\mathcal{C}_{\tilde{\phi}_{\ORAT}}(f',\x,\frac{1}{2})<\ell(0)-\lambda^*\leq \inf_{f\in \mathcal{H}:\underline{M}(f,\x,\epsilon)\leq 0 \leq \overline{M}(f,\x,\epsilon)}\mathcal{C}_{\tilde{\phi}_{\ORAT}}(f,\x,\frac{1}{2}).
    \end{aligned} 
    \label{eq:calibration_1_b}
    \end{equation}
    

\end{list}

2. When $\eta\in (\frac{1}{2},1]$, since $\phi_{\ORAT}$ is non-increasing, for any $\x\in\mathcal{X}_2$,
\begin{equation*}
    \begin{aligned}
    \inf_{f\in \mathcal{H}: \underline{M}(f,\x,\epsilon)\leq 0 }\mathcal{C}_{\tilde{\phi}_{\ORAT}}(f,\x,\eta) &=\inf_{f\in \mathcal{H}: \underline{M}(f,\x,\epsilon)\leq 0 } \eta\phi_{\ORAT}\Big(\underline{M}(f,\x,\epsilon)\Big)+(1-\eta)\phi_{\ORAT}\Big(-\overline{M}(f,\x,\epsilon)\Big)\\
    &\geq \eta \phi_{\ORAT}(0)+(1-\eta) \phi_{\ORAT}(M_\x).
    \end{aligned} 
\end{equation*}
On the other hand, for any $\x\in \mathcal{X}_2$ and $\tau>0$,
\begin{equation*}
    \begin{aligned}
    \mathcal{C}_{\tilde{\phi}_{\ORAT}}(f_\x^\tau,\x,\eta)&=\eta\phi_{\ORAT}\Big(\underline{M}(f_\x^\tau,\x,\epsilon)\Big)+(1-\eta)\phi_{\ORAT}\Big(-\overline{M}(f_\x^\tau,\x,\epsilon)\Big)\\
    &<\eta[\phi_{\ORAT}(M_\x)+\tau]+(1-\eta)[\phi_{\ORAT}(0)+\tau]\\
    &=\eta\phi_{\ORAT}(M_\x)+(1-\eta)\phi_{\ORAT}(0)+\tau.
    \end{aligned} 
\end{equation*}
    
Since $\eta>\frac{1}{2}$ and $M_\x>0$, we have 
\begin{equation*}
    \begin{aligned}
    &\inf_{f\in \mathcal{H}: \underline{M}(f,\x,\epsilon)\leq 0 }\mathcal{C}_{\tilde{\phi}_{\ORAT}}(f,\x,\eta)- \mathcal{C}_{\tilde{\phi}_{\ORAT}}(f_\x^\tau,\x,\eta)\\
    &>[\eta \phi_{\ORAT}(0)+(1-\eta) \phi_{\ORAT}(M_\x)]-[\eta\phi_{\ORAT}(M_\x)+(1-\eta)\phi_{\ORAT}(0)+\tau]\\
    &=[2\eta-1][\phi_{\ORAT}(0)-\phi_{\ORAT}(M_\x)]-\tau\\
    &>0,
    \end{aligned} 
\end{equation*}
where we take $0<\tau<[2\eta-1][\phi_{\ORAT}(0)-\phi_{\ORAT}(M_\x)]$. Therefore, for any $\eta\in(\frac{1}{2},1]$ and $\x\in \mathcal{X}_2$, there exists $0<\tau<[2\eta-1][\phi_{\ORAT}(0)-\phi_{\ORAT}(M_\x)]$ such that 
\begin{equation}
    \begin{aligned}
    \inf_{f\in\mathcal{H}}\mathcal{C}_{\tilde{\phi}_{\ORAT}}(f,\x,\eta)\leq \mathcal{C}_{\tilde{\phi}_{\ORAT}}(f_\x^\tau,\x,\eta)<\inf_{f\in \mathcal{H}: \underline{M}(f,\x,\epsilon)\leq 0 }\mathcal{C}_{\tilde{\phi}_{\ORAT}}(f,\x,\eta).
    \end{aligned} 
\label{eq:calibration_2}
\end{equation}

3. When $\eta\in [0,\frac{1}{2})$, since $\phi_{\ORAT}$ is non-increasing, for any $\x\in\mathcal{X}_2$,
\begin{equation*}
    \begin{aligned}
    \inf_{f\in \mathcal{H}: \overline{M}(f,\x,\epsilon)\geq 0 }\mathcal{C}_{\tilde{\phi}_{\ORAT}}(f,\x,\eta)&=\inf_{f\in \mathcal{H}: \overline{M}(f,\x,\epsilon)\geq 0 } \eta\phi_{\ORAT}\Big(\underline{M}(f,\x,\epsilon)\Big)+(1-\eta)\phi_{\ORAT}\Big(-\overline{M}(f,\x,\epsilon)\Big)\\
    &\geq (1-\eta)\phi_{\ORAT}(0)+\inf_{f\in \mathcal{H}: \overline{M}(f,\x,\epsilon)\geq 0 } \eta\phi_{\ORAT}\Big(\underline{M}(f,\x,\epsilon)\Big)\\
    &\geq (1-\eta)\phi_{\ORAT}(0)+\eta\phi_{\ORAT}\Big(M_\x\Big).
    \end{aligned} 
\end{equation*}
On the other hand, for any $\x\in \mathcal{X}_2$ and $\tau>0$,
\begin{equation*}
    \begin{aligned}
    \mathcal{C}_{\tilde{\phi}_{\ORAT}}(-f_\x^\tau,\x,\eta)&=\eta\phi_{\ORAT}\Big(\underline{M}(-f_\x^\tau,\x,\epsilon)\Big)+(1-\eta)\phi_{\ORAT}\Big(-\overline{M}(-f_\x^\tau,\x,\epsilon)\Big)\\
    &=\eta\phi_{\ORAT}\Big(\underline{M}(-f_\x^\tau,\x,\epsilon)\Big)+(1-
    \eta)\phi_{\ORAT}\Big(\underline{M}(f_\x^\tau,\x,\epsilon)\Big)\\
    &<\eta [\phi_{\ORAT}(0)+\tau]+(1-
    \eta)\phi_{\ORAT}\Big(\underline{M}(f_\x^\tau,\x,\epsilon)\Big)\\
    &<\eta[\phi_{\ORAT}(0)+\tau]+(1-
    \eta)[\phi_{\ORAT}(M_\x)+\tau]\\
    &=\eta\phi_{\ORAT}(0)+(1-\eta)\phi_{\ORAT}(M_\x)+\tau.
    \end{aligned} 
\end{equation*}
Since $\eta<\frac{1}{2}$ and $M_\x>0$, we have 
\begin{equation*}
    \begin{aligned}
    &\inf_{f\in \mathcal{H}: \overline{M}(f,\x,\epsilon)\geq 0 }\mathcal{C}_{\tilde{\phi}_{\ORAT}}(f,\x,\eta)-\mathcal{C}_{\tilde{\phi}_{\ORAT}}(-f_\x^\tau,\x,\eta)\\
    &>[(1-\eta)\phi_{\ORAT}(0)+\eta\phi_{\ORAT}(M_\x)]-[\eta\phi_{\ORAT}(0)+(1-\eta)\phi_{\ORAT}(M_\x)+\tau]\\
    &=(1-2\eta)[\phi_{\ORAT}(0)-\phi_{\ORAT}(M_\x)]-\tau,
    \end{aligned} 
\end{equation*}
where we take $0<\tau<(1-2\eta)[\phi_{\ORAT}(0)-\phi_{\ORAT}(M_\x)]$.
Therefore for any $\eta\in [0,\frac{1}{2})$ and $\x\in \mathcal{X}_2$, there exists $0<\tau<(1-2\eta)[\phi_{\ORAT}(0)-\phi_{\ORAT}(M_\x)]$ such that 
\begin{equation}
    \begin{aligned}
    \inf_{f\in\mathcal{H}}\mathcal{C}_{\tilde{\phi}_{\ORAT}}(f,\x,\eta)\leq\mathcal{C}_{\tilde{\phi}_{\ORAT}}(-f_\x^\tau,\x,\eta)<\inf_{f\in \mathcal{H}: \overline{M}(f,\x,\epsilon)\geq 0 }\mathcal{C}_{\tilde{\phi}_{\ORAT}}(f,\x,\eta)
    \end{aligned} 
\label{eq:calibration_3}
\end{equation}

From (\ref{eq:calibration_1_a}), (\ref{eq:calibration_1_b}), (\ref{eq:calibration_2}), (\ref{eq:calibration_3}), we conclude that $\tilde{\phi}_{\ORAT}$ is $\mathcal{H}$-calibrated with respect to $\widetilde{\ell}_0$. Thus, (i) holds.

According to Lemma \ref{lemma:consistency}, we can conclude that the $\tilde{\phi}_{\ORAT}$ is $\mathcal{H}$-consistent with respect to $\widetilde{\ell}_0$ for all distributions $\mathcal{D}$ over $\mathcal{X}\times \mathcal{Y}$ that satisfy: $\mathcal{R}_{\widetilde{\ell}_0, \mathcal{H}}^* =0$ and there exists $f^*\in\mathcal{H}$ such that $\mathcal{R}_{\phi_{\ORAT}}(f^*)=\mathcal{R}_{\phi_{\ORAT},\mathcal{H}_{all}}^*<+\infty$. Therefore, (ii) holds.
\end{proof}

\subsection{Cross-entropy as A Margin-based Loss}\label{sec:margin_ce}
The cross-entropy loss can be rewritten as a margin-based loss. For example, in binary classification, the conventional binary cross-entropy (bce) loss is given by $bce = -(ylog(\sigma(f(\x))) + (1-y)log(1-\sigma(f(\x))))$ when $y=\{0,1\}$. Here $\sigma$ is the sigmoid function. It is clear that this conventional bce loss is not a margin-based loss. However, we can transfer the negative label 0 to -1.  In this case, by the property of the sigmoid function $1-\sigma(\x)=\sigma(-\x)$, the original bce loss can be rewritten as $bce = -log(\sigma(yf(\x)))$ when $y=\{-1,1\}$. This is in fact a non-negative, continuous, and non-increasing margin-based loss.

\subsection{Proof of Theorem \ref{thm:generalization} \label{sec:proof-gen}}

To get the generalization error bound, we need an equivalent formulation of \eqref{eq:theorem1} which is stated in the following lemma.
\begin{lemma}
Suppose $\lambda \in \mathbb{R}, \hat{\lambda} \in \mathbb{R}$, then the empirical risk $\mathcal{R}_{\widetilde{\ell}}(f;\mathcal{S})$ defined by Eq. \eqref{eq:theorem1} is equivalent to 
\begin{equation}\label{eq:difference}
\mathcal{R}_{\widetilde{\ell}}(f;\mathcal{S}) = \frac{1}{k-m}\Bigg(\min_{\lambda \in \mathbb{R}}\Big\{k\lambda + \sum_{i=1}^n [\ell(f(\tilde{\x}_i), y_i) - \lambda]_+\Big\} - \min_{\hat{\lambda} \in \mathbb{R}}\Big\{m\hat{\lambda} + \sum_{i=1}^n [\ell(f(\tilde{\x}_i), y_i) - \hat{\lambda}]_+\Big\} \Bigg)  
\end{equation}
\end{lemma}
\begin{proof}
According to Eq.(\ref{eq:aorr_newform}), we have 
\begin{equation*}
    \begin{aligned}
    & \min_{\lambda} (k-m)\lambda + \max_{\hat{\lambda}}\Big\{(n-m)\hat{\lambda}-\sum_{i=1}^n [\hat{\lambda}-[\ell(f(\tilde{\x}_{i}),y_{i})-\lambda]_+]_+\Big\}\\
    =&\min_{\lambda} (k-m)\lambda + \sum_{i=m+1}^{n}[[\ell(f(\tilde{\x}),y)-\lambda]_+]_{[i]}\\
    =&\min_{\lambda,q}(k-m)\lambda+\sum_{i=1}^n q_i[\ell(f(\tilde{\x}_{i}),y_{i})-\lambda]_+\ \ \ \text{s.t.} \ q_i\in [0,1], \ ||q||_0 =n-m.
    \end{aligned}
\end{equation*}
Under the constraints, we can rewrite the last formula as 
\begin{equation*}
    \begin{aligned}
    &(k-m)\lambda + \sum_{i=1}^n q_i[\ell(f(\tilde{\x}_{i}),y_{i})-\lambda]_+ \\
    = &(k-m)\lambda + \sum_{i=1}^n [\ell(f(\tilde{\x}_{i}),y_{i})-\lambda]_+ - \sum_{i=1}^n (1-q_i)[\ell(f(\tilde{\x}_{i}),y_{i})-\lambda]_+\\
    =&k\lambda + \sum_{i=1}^n [\ell(f(\tilde{\x}_{i}),y_{i})-\lambda]_+ - \sum_{i=1}^n (1-q_i)\{[\ell(f(\tilde{\x}_{i}),y_{i})-\lambda]_++\lambda\}.
    \end{aligned}
\end{equation*}
The last equality holds because $\sum_{i=1}^n(1-q_i)=n-(n-m)=m$.

For the term $\sum_{i=1}^n (1-q_i)\{[\ell(f(\tilde{\x}_{i}),y_{i})-\lambda]_++\lambda\}$, we assume $\ell(f^*(\tilde{\x}_{i}),y_{i})$, $\forall i$, are sorted in descending order when getting the optimal model $f^*$. For example, $\ell(f^*(\tilde{\x}_{1}),y_{1})\geq \ell(f^*(\tilde{\x}_{2}),y_{2})\geq \cdots \geq \ell(f^*(\tilde{\x}_{n}),y_{n})$. Since $\lambda^*\geq 0$, the optimal $q^*$ should be $q_1^*=\cdots=q_m^*=0$, $q_{m+1}^*=\cdots=q_n^*=1$. Note that $\lambda^*$ must be an optimal solution of the problem
\begin{equation*}
    \begin{aligned}
    \min_{\lambda} (k-m)\lambda + \sum_{i=m+1}^n q_i^*[\ell(f^*(\tilde{\x}_{i}),y_{i})-\lambda]_+. 
    \end{aligned}
\end{equation*}
From Lemma \ref{lemma1}, we know $\ell(f^*(\tilde{\x}_{m+1}),y_{m+1})\geq \lambda^*$, which implies that $\ell(f^*(\tilde{\x}_{i}),y_{i})-\lambda^*\geq 0$ holds for $q_i<1$. Therefore, $\sum_{i=1}^n (1-q_i)\{[\ell(f(\tilde{\x}_{i}),y_{i})-\lambda]_++\lambda\} = \sum_{i=1}^n (1-q_i)\ell(f(\tilde{\x}_{i}),y_{i})$. Furthermore, we know 
\begin{equation*}
    \begin{aligned}
    \min_{\hat{\lambda}}\Big\{m\hat{\lambda}+\sum_{i=1}^n[\ell(f(\tilde{\x}_{i}),y_{i})-\hat{\lambda}]_+\Big\} = \max_{q}\Big\{\sum_{i=1}^n (1-q_i)\ell(f(\tilde{\x}_{i}),y_{i})\Big|\ q_i\in [0,1], \ ||q||_0 =n-m\Big\}.
    \end{aligned}
\end{equation*}
Then we get 
\begin{equation*}
    \begin{aligned}
    &\frac{1}{k-m}\Bigg(\min_{\lambda} (k-m)\lambda + \max_{\hat{\lambda}}\Big\{(n-m)\hat{\lambda}-\sum_{i=1}^n [\hat{\lambda}-[\ell(f(\tilde{\x}_{i}),y_{i})-\lambda]_+]_+\Big\}\Bigg) \\
    =&\frac{1}{k-m}\Bigg(\min_{\lambda}\Big\{k\lambda + \sum_{i=1}^n [\ell(f(\tilde{\x}_{[i]}), y_{[i]}) - \lambda]_+\Big\} - \min_{\hat{\lambda}}\Big\{m\hat{\lambda} + \sum_{i=1}^n [\ell(f(\tilde{\x}_{[i]}), y_{[i]}) - \hat{\lambda}]_+\Big\} \Bigg).
    \end{aligned}
\end{equation*}

The proof is complete.
\end{proof}

Considering the limit case of \eqref{eq:difference}, the population risk $\mathcal{R}_{\widetilde{\ell}}(f)$ can be written as 
\begin{multline*}
\frac{1}{k-m}\Bigg(\min_{\lambda \in \mathbb{R}}\Big\{k\lambda + \sum_{i=1}^n [\ell(f(\tilde{\x}_{[i]}), y_{[i]}) - \lambda]_+\Big\} - \min_{\hat{\lambda} \in \mathbb{R}}\Big\{m\hat{\lambda} + \sum_{i=1}^n [\ell(f(\tilde{\x}_{[i]}), y_{[i]}) - \hat{\lambda}]_+\Big\} \Bigg) \\
\xrightarrow[n\rightarrow \infty]{\frac{k-m}{n}\rightarrow \nu, \frac{m}{n} \rightarrow \mu}\frac{1}{\nu}\Bigg(\min_{\lambda \in \mathbb{R}} \Big\{(\nu + \mu) \lambda + \mathbb{E}[\widetilde{\ell}(f(\x), y) - \lambda]_+ \Big\} - \min_{\hat{\lambda} \in \mathbb{R}} \Big\{\mu\hat{\lambda} +\mathbb{E}[\widetilde{\ell}(f(\x), y) - \hat{\lambda}]_+ \Big\}\Bigg)=\mathcal{R}_{\widetilde{\ell}}(f).
\end{multline*}
The next Lemma tells us if the loss function is bounded, we can constrain the problem of $\mathcal{R}_{\widetilde{\ell}}(f)$ and $\mathcal{R}_{\widetilde{\ell}}(f;\mathcal{S})$ in the bounded range as well.

\begin{lemma}
Suppose that the range of $\ell$ is $[0,M]$. Then we have
\begin{equation}\label{eq:population-risk-bounded}
\mathcal{R}_{\widetilde{\ell}}(f) = \frac{1}{\nu}\Bigg(\min_{\lambda \in [0,M]} \Big\{(\nu + \mu)\lambda +\mathbb{E}[\widetilde{\ell}(f(\x), y) - \lambda]_+ \Big\} - \min_{\hat{\lambda} \in [0, M]} \Big\{\mu\hat{\lambda} +\mathbb{E}[\widetilde{\ell}(f(\x), y) - \hat{\lambda}]_+ \Big\}\Bigg),
\end{equation}
and so does the empirical risk
\begin{equation}\label{eq:empirical-risk-bounded}
\mathcal{R}_{\widetilde{\ell}}(f;\mathcal{S}) = \frac{1}{k-m}\Bigg(\min_{\lambda \in [0, M]} \Big\{k\lambda + \sum_{i=1}^n[\widetilde{\ell}(f(\x_i), y) - \lambda]_+ \Big\} - \min_{\hat{\lambda} \in [0,M]} \Big\{m\hat{\lambda} + \sum_{i=1}^n[\widetilde{\ell}(f(\x_i), y) - \hat{\lambda}]_+ \Big\}\Bigg).
\end{equation}
\end{lemma}

\begin{proof}
The proof of \eqref{eq:empirical-risk-bounded} is straight forward. By Lemma \ref{lemma1} and Lemma \ref{lemma2}, we know $\lambda_\mathcal{S}^* = \widetilde{\ell}(f(\x_{[k]}), y_{[k]})$ and $\hat{\lambda}_\mathcal{S}^* = \widetilde{\ell}(f(\x_{[m]}), y_{[m]})$ are a pair of solution of \eqref{eq:empirical-risk-bounded}. Since $\widetilde{\ell}(f(\x), y) = \max_{\tilde{\x} \in \mathcal{B}_\epsilon (\x)} \ell(f(\tilde{\x}), y) \in [0, M]$ for any $\x, y$, we have $\lambda_\mathcal{S}^*, \hat{\lambda}_\mathcal{S}^* \in [0,M]$.

Next we move on to \eqref{eq:population-risk-bounded}. Let $\lambda^*$ and $\hat{\lambda}^*$ be a pair of solution of \eqref{eq:population-risk-bounded}. Let $\lambda = M$, then we have
\begin{align*}
(\nu+\mu)\lambda^* &\leq (\nu+\mu)\lambda^* +\mathbb{E}[\widetilde{\ell}(f(\x), y) - \lambda^*]_+\\
&\leq (\nu+\mu) M +\mathbb{E}[\widetilde{\ell}(f(\x), y) - M]_+ \\
&\leq (\nu+\mu) M + \mathbb{E}[M - M]_+ = (\nu+\mu) M,
\end{align*}
which implies $\lambda^* \leq M$. On the other hand, assume $\lambda^* = -\varepsilon$ for some $\varepsilon > 0$. Let $\lambda = 0$, then we have
\begin{align*}
-(\nu+\mu)\varepsilon +\mathbb{E}[\widetilde{\ell}(f(\x), y) + \varepsilon]_+ = (1 - (\nu+\mu))\varepsilon +\mathbb{E}[\widetilde{\ell}(f(\x), y)]_+ \leq  \mathbb{E}[\widetilde{\ell}(f(\x), y)]_+,
\end{align*}
which contradicts with $(1-(\nu+\mu))\varepsilon > 0$. Therefore we have $\lambda^* \geq 0$. Similarly, we can show that $\hat{\lambda}^* \in [0, M]$. The proof is complete.
\end{proof}

The next lemma shows the uniform convergence of learning with \ORAT~without using perturbation.
\begin{lemma}\label{lem:aorr-generalization}
Suppose that the range of $\ell(f(\x), y)$ is $[0, M]$. Then, for any $\delta \in (0, 1)$, with probability at least $1 -\delta$ over the draw of an i.i.d. training dataset of size $n$, the following holds for all $\ell_f \in \ell_\mathcal{H}$, 
\[
\mathcal{R}_\ell(f) - \mathcal{R}_\ell(f;\mathcal{S}) \leq \frac{2}{\nu }\Big( 2\mathfrak{R}_n(\ell_\mathcal{H}) + \frac{M(2\sqrt{2} + 3\sqrt{\log(2/\delta)})}{\sqrt{2n}}\Big).
\]
\end{lemma}

\begin{proof}
By the subadditivity of $\max$ operator, for any $\ell_f \in \ell_\mathcal{H}$, we have
\begin{align*}
&\mathcal{R}_\ell(f) - \mathcal{R}_\ell(f;\mathcal{S}) \\
= & \frac{1}{\nu} \min_{\lambda \in [0, M]} \Big\{(\nu+\mu)\lambda + \mathbb{E}[\ell(f(\x), y) - \lambda]_+ \Big\} - \frac{1}{\nu} \min_{\lambda \in [0, M]} \Big\{(\nu+\mu)\lambda + \frac{1}{n}\sum_{i=1}^{n}(\ell(f(\x), y) - \lambda)_+ \Big\} \\ 
& + \frac{1}{\nu} \min_{\hat{\lambda} \in [0, M]} \Big\{\mu\hat{\lambda} + \frac{1}{n}\sum_{i=1}^{n}(\ell(f(\x), y) - \hat{\lambda})_+ \Big\} - \frac{1}{\nu} \min_{\hat{\lambda} \in [0, M]} \Big\{\mu\hat{\lambda} + \mathbb{E}[\ell(f(\x), y) - \hat{\lambda}]_+ \Big\} \\ 
\leq & \max_{\lambda \in [0, M]} \Big\{\frac{1}{\nu}\mathbb{E}[\ell(f(\x), y) - \lambda]_+ -  \frac{1}{n\nu }\sum_{i=1}^{n}(\ell(f(\x), y) - \lambda)_+ \Big\} := L_1(f, \ell) \numberthis \label{eq: max-generalization-1}\\
& + \max_{\hat{\lambda} \in [0, M]}\Big\{ \frac{1}{n\nu }\sum_{i=1}^{n}(\ell(f(\x), y) - \hat{\lambda})_+ - \frac{1}{\nu}\mathbb{E}[\ell(f(\x), y) - \hat{\lambda}]_+ \Big\}: = L_2(f, \ell). \numberthis \label{eq: max-generalization-2}
\end{align*}
Without loss of generality, we consider \eqref{eq: max-generalization-1}, the bound for \eqref{eq: max-generalization-2} can be derived in a similar manner. Taking supremum on both sides, we have
\[
\sup_{\ell_f \in \ell_\mathcal{H}} L_1(f, \ell) \leq \sup_{\ell_f \in \ell_\mathcal{H}, \lambda \in [0,M]}\Big\{\frac{1}{\nu}\mathbb{E}[\ell(f(\x), y) - \lambda]_+ -  \frac{1}{n\nu}\sum_{i=1}^{n}(\ell(f(\x), y) - \lambda)_+ \Big\} := \Phi(\mathcal{S}).
\]
It is standard to verify that $\Phi(\mathcal{S})$ satisfies the bounded differences condition with parameter $\frac{M}{\nu}$ and one can apply McDiarmid's inequality \citep{mcdiarmid1989method} so that with probability at least $1 - \delta/4$, there holds
\[
\Phi(\mathcal{S}) \leq \mathbb{E}[\Phi(\mathcal{S})] + \frac{M}{\nu}\sqrt{\frac{\log(4/\delta)}{2n}}.
\]
By further standard reduction from the expectation to Rademacher complexity (Theorem 3.3 \cite{mohri2018foundations}), with probability at least $1 - \delta/2$, there holds 
\begin{equation}\label{eq:rademacher-standard}
\Phi(\mathcal{S}) \leq 2\mathfrak{R}_n\Big(\frac{1}{\nu}(\mathcal{G})_+\Big) + \frac{3M}{\nu}\sqrt{\frac{\log(4/\delta)}{2n}},    
\end{equation}
where $\mathcal{G} = \{\ell_f - \lambda|\ell_f \in \ell_\mathcal{H}, \lambda \in [0,M]\}$ and $(\cdot)_+=\max(\cdot,0)$. Since the ramp function $(\cdot)_+$ is $1$-Lipschitz and $(0)_+ = 0$, by Ledoux-Talagrand contraction inequality \citep{ledoux1991probability} we have 
\begin{align*}\label{eq:rademacher-lambda}
\mathfrak{R}_n\Big(\frac{1}{\nu}(\mathcal{G})_+\Big) \leq &  \frac{1}{\nu}\mathfrak{R}_n(\mathcal{G}) = \frac{1}{\nu}\mathbb{E}_{\sigma}\Big[ \sup_{\ell_f \in \ell_\mathcal{H}, \lambda \in [0, M]} \Big(\frac{1}{n} \sum_{i=1}^n \sigma_i \ell(f(\x_i), y_i) - \frac{1}{n}\sum_{i=1}^n \sigma_i \lambda\Big)\Big] \\
\leq & \frac{1}{\nu}\Big(\mathbb{E}_{\sigma}\Big[ \sup_{\ell_f \in \ell_\mathcal{H}} \frac{1}{n} \sum_{i=1}^n \sigma_i \ell(f(\x_i), y_i)\Big] + \mathbb{E}_{\sigma}\Big[\sup_{\lambda \in [0, M]}\frac{1}{n}\sum_{i=1}^n \sigma_i \lambda\Big]\Big) \\
\leq & \frac{1}{\nu}\Big(\mathfrak{R}_n(\ell_\mathcal{H}) + \frac{M}{n} \mathbb{E}_{\sigma}\Big|\sum_{i=1}^n \sigma_i\Big|\Big) \\
\leq & \frac{1}{\nu}\Big(\mathfrak{R}_n(\ell_\mathcal{H}) + \frac{M}{\sqrt{n}}\Big), \numberthis
\end{align*}
where the last inequality follows by $\Big(\mathbb{E}_{\sigma}\Big[\sum_{i=1}^n \sigma_i\Big]\Big)^2 \leq \mathbb{E}_{\sigma}\Big(\sum_{i=1}^n \sigma_i\Big)^2 = n$. By putting \eqref{eq:rademacher-lambda} into \eqref{eq:rademacher-standard}, we have
\[
\sup_{\ell_f \in \ell_\mathcal{H}} L_1(f, \ell) \leq \frac{1}{\nu}\Big( 2\mathfrak{R}_n(\ell_\mathcal{H}) + \frac{M(2\sqrt{2} + 3\sqrt{\log(4/\delta)})}{\sqrt{2n}}\Big)
\]
with probability at least $1 - \delta/2$. The lemma holds by noting $\sup_{\ell_f \in \ell_\mathcal{H}} \{R(f,\ell) - R_n(f,\ell)\} \leq \sup_{\ell_f \in \ell_\mathcal{H}} L_1(f, \ell) + \sup_{\ell_f \in \ell_\mathcal{H}} L_2(f, \ell)$.
\end{proof}

The next corollary is straight-forward from Lemma \ref{lem:aorr-generalization} by replacing $\ell$ with $\widetilde{\ell}$. 

\begin{corollary}[Theorem \ref{thm:generalization} restated]
Suppose that the range of $\ell(f(\x), y)$ is $[0, M]$. Then, for any $\delta \in (0, 1)$, with probability at least $1 -\delta$ over the draw of an i.i.d. training dataset of size $n$, the following holds for all $\ell_f \in \ell_\mathcal{H}$, \[R_{\widetilde{\ell}}(f) - R_{\widetilde{\ell}}(f;\mathcal{S}) \leq \frac{2}{\nu}\Big( 2\mathfrak{R}_n(\widetilde{\ell}_\mathcal{H}) + \frac{M(2\sqrt{2} + 3\sqrt{\log(2/\delta)})}{\sqrt{2n}}\Big).
\]
\end{corollary}

\subsection{Examples of Hypothesis Sets}\label{sec:rademacher_example}
We give two examples of hypothesis sets: linear classifiers and nonlinear neural networks, that satisfy the condition in Theorem \ref{theorem2} and \ref{thm:generalization}. 
Suppose $\ell: \mathbb{R} \rightarrow [0,M]$ is monotonically non-increasing and $L$-Lipschitz continuous. 
In this case, the adversarial loss can be written \citep{yin2019rademacher} as $\widetilde{\ell}(f_\theta(\x), y):=\max_{\tilde{\x} \in \mathcal{B}_\epsilon(\x)} \ell(f_\theta(\tilde{\x}), y) = \ell\big(\min_{\tilde{\x} \in \mathcal{B}_\epsilon(\x)} y f_\theta(\x)\big)$.
Therefore, given any function class $\mathcal{H}$, we can define the function class $\widetilde{\mathcal{H}}\subseteq \mathbb{R}^{\mathcal{X} \times \{\pm 1\}}$, such that $\widetilde{\mathcal{H}} = \{\min_{\tilde{\x} \in \mathcal{B}_\epsilon(\x)} yf_\theta(\x): f_\theta \in \mathcal{H}\}$. By the Ledoux-Talagrand contraction inequality \citep{ledoux1991probability}, we have $\mathfrak{R}_n(\widetilde{\ell}_{\mathcal{H}}) \leq L \mathfrak{R}_n(\widetilde{\mathcal{H}})$. Hence we only need to characterize the Rademacher complexity of $\widetilde{\mathcal{H}}$ given $\mathcal{H}$.


\textbf{Linear Classifiers.} Let the hypothesis set $\mathcal{H}_{lin} \subseteq \mathbb{R}^{\mathcal{X}}$ be a set of linear functions of $\x \in \mathcal{X}$. More specifically, we consider prediction vector $\theta$ with $l_p$ ($p \geq 1$) norm constraint, i.e. $\mathcal{H}_{lin} = \{f_\theta(\x) = \theta^\top \x : \|\theta\|_p \leq r\}$. Then it is straight-forward to check $\mathcal{H}_{lin}$ is a symmetric hypothesis set. Furthermore, let $\tilde{d} =d^{1 - 1/p}$. \cite{awasthi2020adversarial} showed that $\max\Big\{\mathfrak{R}_n(\mathcal{H}_{lin}), \epsilon r \frac{\max\{\tilde{d},1\}}{2\sqrt{2n}}\Big\}\leq\mathfrak{R}_n(\widetilde{\mathcal{H}}_{lin})
\leq \mathfrak{R}_n(\mathcal{H}_{lin}) + \epsilon r \frac{\max\{\tilde{d},1\}}{2\sqrt{n}}$. 
Combined with Theorem \ref{thm:generalization} with probability at least $1-\delta$, we have $\mathcal{R}_{\widetilde{\ell}}(f) - \mathcal{R}_{\widetilde{\ell}}(f; \mathcal{S}) 
\leq \frac{2}{\nu}\Big( 2L\mathfrak{R}_n(\mathcal{H}_{lin}) + L\epsilon r \frac{\max\{\tilde{d},1\}}{2\sqrt{n}}
+ \frac{M(2\sqrt{2} + 3\sqrt{\log(2/\delta)})}{\sqrt{2n}}\Big)$, where $\mathfrak{R}_n(\mathcal{H}_{lin})$ is given by classical result in \cite{kakade2008complexity}.

\textbf{Neural Networks.} We consider feedforward neural networks with ReLU activation function $\rho$, i.e. $\rho(t) = \max\{0,t\}$. 
In particular, if the hypothesis set is one-layer neural networks defined as $\mathcal{H}_{one} = \{f_\theta(\x) = \theta_0^\top\rho(\Theta \x): \|\theta_0\|_1 \leq r_0, \|\Theta_i\|_p \leq r\}$ where $\Theta_i \in \mathbb{R}^d$ is the $i$-th row of $\Theta \in \mathbb{R}^{d' \times d}$. This is a symmetric hypothesis set. Furthermore, the Rademacher complexity can be upper bounded \citep{awasthi2020adversarial} as $\mathfrak{R}_n(\widetilde{\mathcal{H}}_{one}) \leq \frac{rr_0\max\{1, \tilde{d}(\|\mathbf{X}\|_\infty + \epsilon)\}}{\sqrt{n}}(1+\sqrt{d(d'+1)\log(36)})$, where $\mathbf{X} = (\x_1, \cdots, \x_n)^\top$. Combined with Theorem \ref{thm:generalization} with probability at least $1-\delta$, we have $\mathcal{R}_{\widetilde{\ell}}(f) - \mathcal{R}_{\widetilde{\ell}}(f; \mathcal{S}) 
\leq \frac{2}{\nu}\Big( \frac{2Lrr_0\max\{1, \tilde{d}(\|\mathbf{X}\|_\infty + \epsilon)\}}{\sqrt{n}}
\times (1+\sqrt{d(d'+1)\log(36)}) + \frac{M(2\sqrt{2} + 3\sqrt{\log(2/\delta)})}{\sqrt{2n}}\Big).$ 
Such bound implies the generalization error depends on the perturbation size $\epsilon$,
which demonstrates the intrinsic complexity of adversarial training.

\section{Additional Experimental Details}\label{additional_exp_details}

\subsection{Source Code}
For the purpose of review, the source code is accessible in the supplementary file.


\subsection{Settings of Networks and Computing Infrastructure Description}\label{sec:computing_description}

For all networks, we training them by using  (mini-batch) stochastic gradient descent with momentum 0.9, weight decay 2e-4, batch size 128, epochs 50 (for LeNet) / 100  (for Small-CNN) / 100 (for ResNet-18), and initial learning rate 0.03 (for LeNet) / 0.1 (for Small-CNN) / 0.1 (for ResNet-18) which is divided by the 10 at 20-th and 40-th epoch for LeNet /  30-th and 60-th epoch for Small-CNN and ResNet-18. 

All algorithms are implemented in Python 3.6 and trained and tested on an Intel(R) Xeon(R) CPU W5590 @3.33GHz with 48GB of RAM and an NVIDIA Quadro RTX 6000 GPU with 24GB memory.

\subsection{Training Settings on Toy Examples}
In this section, we provide more details about how to generate synthetic datasets in Figure \ref{fig:interpretation}.

We generate two sets of 2D synthetic data (Figure \ref{fig:interpretation}). Each dataset contains 200 samples from Gaussian distributions with different means and variances. We consider both the case of the balanced (Figure \ref{fig:interpretation} left) and the imbalanced (Figure \ref{fig:interpretation} right) data distributions, in the former, the training data for the two classes are approximately equal while in the latter one class has a dominating number of samples in comparison to the other. In the balanced dataset (Figure \ref{fig:interpretation} left), we create two outliers. One is in the blue class (shown as red $\times$), the other is in the red class (shown as blue $\circ$). In the imbalanced dataset, we create one outlier in the blue class (shown as red $\times$). For both datasets, the yellow squares around data samples represent the samples are perturbed within a $\ell_\infty$ ball.

For the balanced dataset (Figure \ref{fig:interpretation} left), we build a simple network contains two linear layers and one ReLU layer \citep{nair2010rectified}. The number of hidden units is three.   For the imbalanced dataset (Figure \ref{fig:interpretation} right), the network contains four linear layers and three ReLU layers. The number of hidden units is 20.   We train these networks using SGD with 0.9 momentum for 3000 (balanced dataset) / 100,000 (imbalanced dataset) iterations with the learning rate of 0.02. We set $k=20$ and $m=2$ for balanced dataset, and $k=20$ and $m=1$ for imbalanced dataset when run our \ORAT~algorithm. In AT and \ORAT~, the training attack is PGD$^{10}$ and we set the perturbation bound $\epsilon=0.01$ and the PGD step size $\epsilon/4$.

\begin{table*}[t]
\centering
\setlength\tabcolsep{3pt}
\begin{tabular}{|cc|cccc|cccc|cccc|}
\hline
\multicolumn{2}{|c|}{\multirow{3}{*}{Noise}}    & \multicolumn{4}{c|}{MNIST}                                                    & \multicolumn{4}{c|}{CIFAR-10}                                                    & \multicolumn{4}{c|}{CIFAR-100}                                                    \\ \cline{3-14} 
\multicolumn{2}{|c|}{}                     & \multicolumn{2}{c|}{$\epsilon$=0.1}                         & \multicolumn{2}{c|}{$\epsilon$=0.2}    & \multicolumn{2}{c|}{$\epsilon$=2/255}                         & \multicolumn{2}{c|}{$\epsilon$=8/255}    & \multicolumn{2}{c|}{$\epsilon$=2/255}                         & \multicolumn{2}{c|}{$\epsilon$=8/255}    \\ \cline{3-14} 
\multicolumn{2}{|c|}{}                     & \multicolumn{1}{c|}{$k$} & \multicolumn{1}{c|}{$m$} & \multicolumn{1}{c|}{$k$} & $m$ & \multicolumn{1}{c|}{$k$} & \multicolumn{1}{c|}{$m$} & \multicolumn{1}{c|}{$k$} & $m$ & \multicolumn{1}{c|}{$k$} & \multicolumn{1}{c|}{$m$} & \multicolumn{1}{c|}{$k$} & $m$ \\ \hline
\multicolumn{2}{|c|}{0}                     & \multicolumn{1}{c|}{60000} & \multicolumn{1}{c|}{1} & \multicolumn{1}{c|}{60000} & 1 & \multicolumn{1}{c|}{45000} & \multicolumn{1}{c|}{1} & \multicolumn{1}{c|}{45000} & 1 & \multicolumn{1}{c|}{49950} & \multicolumn{1}{c|}{1} & \multicolumn{1}{c|}{49950} & 1 \\ \hline
\multicolumn{1}{|c|}{\multirow{4}{*}{\rotatebox{90}{\small \makecell{Symmetric \\ Noise}}}} & 10\% & \multicolumn{1}{c|}{59950} & \multicolumn{1}{c|}{2000} & \multicolumn{1}{c|}{60000} & 2000 & \multicolumn{1}{c|}{50000} & \multicolumn{1}{c|}{300} & \multicolumn{1}{c|}{50000} & 500 & \multicolumn{1}{c|}{50000} & \multicolumn{1}{c|}{300} & \multicolumn{1}{c|}{50000} & 500 \\ \cline{2-14} 
\multicolumn{1}{|l|}{}                  & 20\% & \multicolumn{1}{c|}{59950} & \multicolumn{1}{c|}{6000} & \multicolumn{1}{c|}{60000} & 3000 & \multicolumn{1}{c|}{50000} & \multicolumn{1}{c|}{300} & \multicolumn{1}{c|}{50000} & 300 & \multicolumn{1}{c|}{50000} & \multicolumn{1}{c|}{500} & \multicolumn{1}{c|}{49800} & 500 \\ \cline{2-14} 
\multicolumn{1}{|l|}{}                  & 30\% & \multicolumn{1}{c|}{59950} & \multicolumn{1}{c|}{5000} & \multicolumn{1}{c|}{60000} & 5000 & \multicolumn{1}{c|}{50000} & \multicolumn{1}{c|}{10} & \multicolumn{1}{c|}{50000} & 200 & \multicolumn{1}{c|}{49800} & \multicolumn{1}{c|}{100} & \multicolumn{1}{c|}{50000} &  500\\ \cline{2-14} 
\multicolumn{1}{|l|}{}                  & 40\% & \multicolumn{1}{c|}{59950} & \multicolumn{1}{c|}{11000} & \multicolumn{1}{c|}{60000} & 11000 & \multicolumn{1}{c|}{50000} & \multicolumn{1}{c|}{100} & \multicolumn{1}{c|}{50000} & 50 & \multicolumn{1}{c|}{49950} & \multicolumn{1}{c|}{100} & \multicolumn{1}{c|}{49950} & 500 \\ \hline
\multicolumn{1}{|l|}{\multirow{4}{*}{\rotatebox{90}{\small \makecell{Asymmetric \\ Noise}}}} & 10\% & \multicolumn{1}{c|}{60000} & \multicolumn{1}{c|}{100} & \multicolumn{1}{c|}{60000} & 10 & \multicolumn{1}{c|}{49950} & \multicolumn{1}{c|}{300} & \multicolumn{1}{c|}{50000} & 500 & \multicolumn{1}{c|}{49900} & \multicolumn{1}{c|}{100} & \multicolumn{1}{c|}{49950} & 500 \\ \cline{2-14} 
\multicolumn{1}{|l|}{}                  & 20\% & \multicolumn{1}{c|}{59950} & \multicolumn{1}{c|}{100} & \multicolumn{1}{c|}{59950} & 100 & \multicolumn{1}{c|}{49950} & \multicolumn{1}{c|}{500} & \multicolumn{1}{c|}{50000} & 300 & \multicolumn{1}{c|}{50000} & \multicolumn{1}{c|}{100} & \multicolumn{1}{c|}{50000} & 500 \\ \cline{2-14} 
\multicolumn{1}{|l|}{}                  & 30\% & \multicolumn{1}{c|}{59950} & \multicolumn{1}{c|}{10} & \multicolumn{1}{c|}{60000} & 100 & \multicolumn{1}{c|}{49950} & \multicolumn{1}{c|}{200} & \multicolumn{1}{c|}{50000} & 300 & \multicolumn{1}{c|}{50000} & \multicolumn{1}{c|}{100} & \multicolumn{1}{c|}{50000} & 500 \\ \cline{2-14} 
\multicolumn{1}{|l|}{}                  & 40\% & \multicolumn{1}{c|}{59950} & \multicolumn{1}{c|}{10} & \multicolumn{1}{c|}{60000} & 10 & \multicolumn{1}{c|}{50000} & \multicolumn{1}{c|}{450} & \multicolumn{1}{c|}{50000} & 500 & \multicolumn{1}{c|}{50000} & \multicolumn{1}{c|}{100} & \multicolumn{1}{c|}{49800} & 500 \\ \hline
\end{tabular}
\vspace{-0.4em}
\caption{\it The $k$ and $m$ settings of \ORAT~on real datasets in different noise.}
\label{tab:general_performance_settings}
\vspace{-0.4em}
\end{table*}

\subsection{Details of Outliers Generation by Using Asymmetric Noise} \label{sec:outlier_generation}
In asymmetric noise generation procedure, for MNIST,  flipping
2$\rightarrow$7, 3$\rightarrow$8, 5$\leftrightarrow$6 and 7$\rightarrow$1; for CIFAR-10, flipping TRUCK$\rightarrow$AUTOMOBILE, BIRD$\rightarrow$AIRPLANE,
DEER$\rightarrow$HORSE, CAT$\leftrightarrow$DOG; for CIFAR-100, the 100 classes are grouped into 20 super-classes with each having 5 sub-classes, then flipping between two randomly selected sub-classes within each super-class. 

\subsection{$k$ and $m$ Settings on Real Datasets}
We provide a reference for setting $k$ and $m$ to reproduce our \ORAT~experimental results (Table \ref{tab:general_performance_1}) on real datasets in Table \ref{tab:general_performance_settings}.

\section{Additional Experimental Results}\label{additional_exp_results}

\subsection{More Experiments on Toy Example}
\begin{figure}[t]
\begin{subfigure}[t]{0.33\linewidth}
    \includegraphics[width=\linewidth]{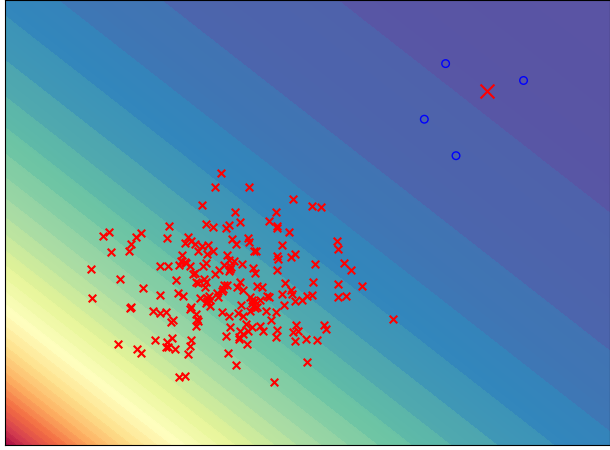}
    \caption{Standard Training (ST)}
\end{subfigure}%
\hfill%
\begin{subfigure}[t]{0.33\linewidth}
    \includegraphics[width=\linewidth]{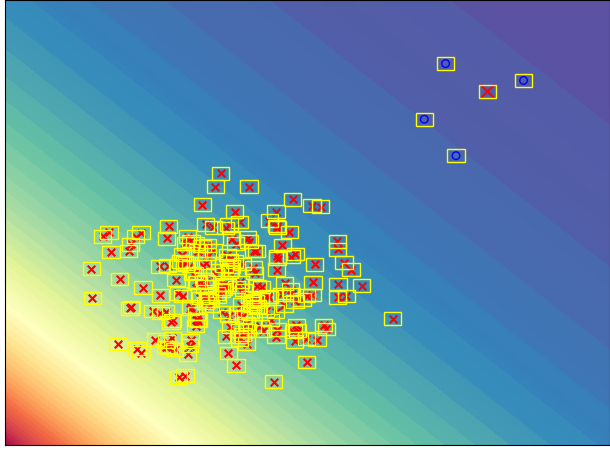}
    \caption{Adversarial Training (AT)}
\end{subfigure}%
\hfill%
\begin{subfigure}[t]{0.33\linewidth}
    \includegraphics[width=\linewidth]{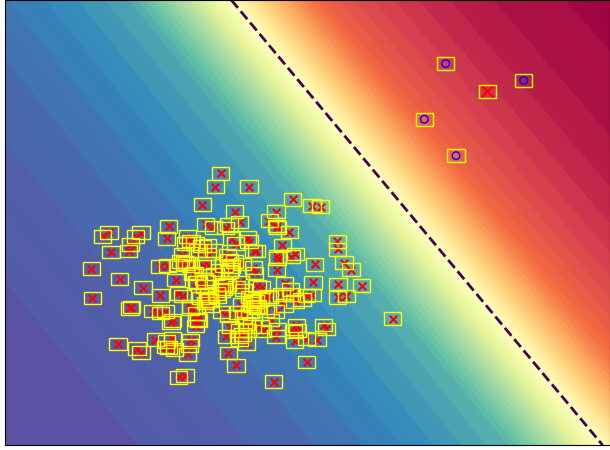}
    \caption{\ORAT}
\end{subfigure}%
\vspace{-0.4em}
\caption{\small \it An additional illustrative example of standard training (ST), adversarial training (AT), and ATRR for binary classification on an imbalanced synthetic dataset with one outlier (shown as red $\times$) in the blue class. The yellow squares around data samples represent the samples are perturbed within a $\ell_\infty$ ball. The dashed line is the decision boundary. The figure is better viewed in color.}
\label{fig:interpretation_2}
\end{figure}
We generate additional 2D synthetic data as shown in Figure \ref{fig:interpretation_2} to demonstrate the performance of our \ORAT~method. This imbalanced dataset contains 200 samples from Gaussian distribution with different means and variances. For this dataset, we create one outlier in the blue class (shown as red $\times$). In order to train this dataset, we build a network, which contains two linear layers and one ReLU layer. The number of hidden units is 64. We train this network using SGD with 0.9 momentum for 100,000 iterations with a learning rate of 0.1. We set $k=5$ and $m=1$ for \ORAT~. Similarly, in AT and \ORAT~, the training attack is PGD$^{10}$, the perturbation bound $\epsilon=0.01$, and the PGD step size is $\epsilon/4$.

From Figure \ref{fig:interpretation_2}, we can find the classifiers are trained from ST (a) and AT (b) cannot separate two classes in the training data.  However, we find the classifier is training by using our proposed \ORAT~can separate these two classes. The results demonstrate that our \ORAT~can eliminate the influence of the outliers when doing the adversarial training. 

\begin{figure}[t]
\begin{subfigure}[t]{0.33\linewidth}
    \includegraphics[width=\linewidth]{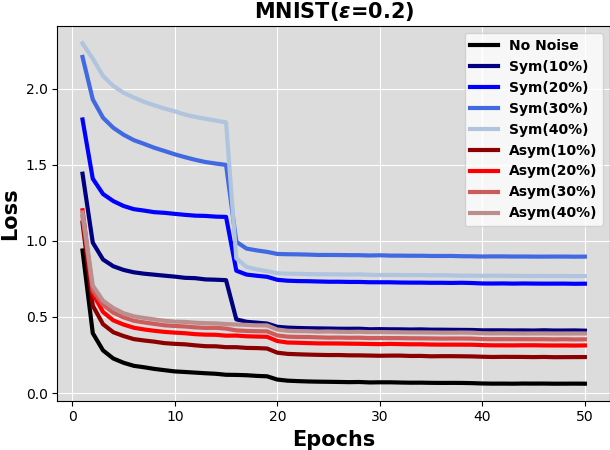}
\end{subfigure}%
\hfill%
\begin{subfigure}[t]{0.33\linewidth}
    \includegraphics[width=\linewidth]{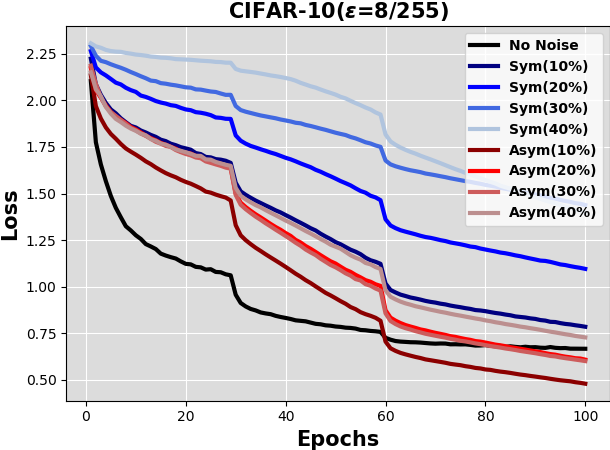}
\end{subfigure}
\hfill%
\begin{subfigure}[t]{0.33\linewidth}
    \includegraphics[width=\linewidth]{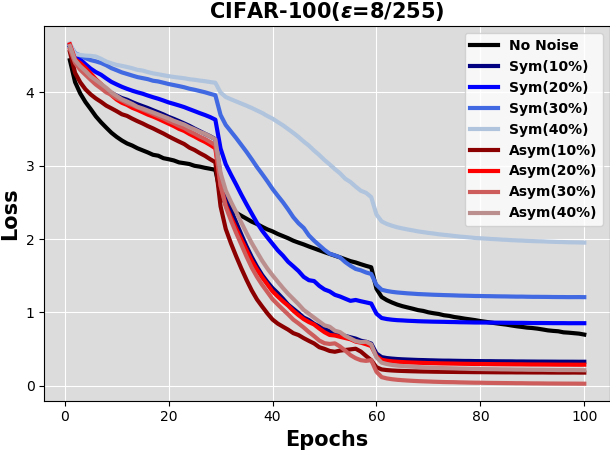}
\end{subfigure}

\begin{subfigure}[t]{0.33\linewidth}
    \includegraphics[width=\linewidth]{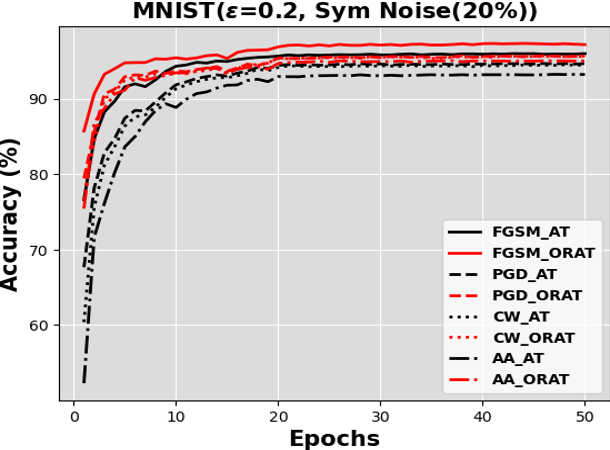}
\end{subfigure}%
\hfill%
\begin{subfigure}[t]{0.33\linewidth}
    \includegraphics[width=\linewidth]{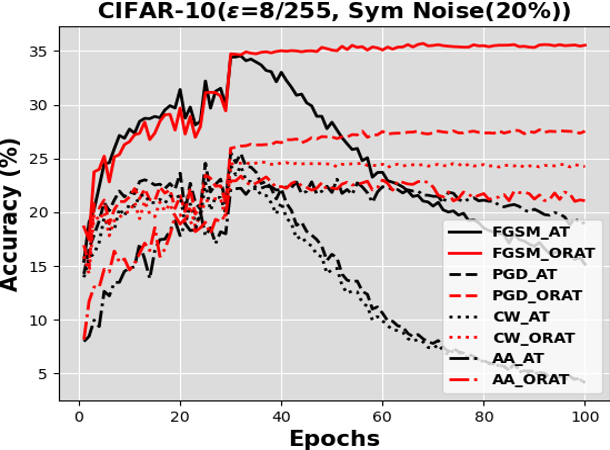}
\end{subfigure}
\hfill%
\begin{subfigure}[t]{0.33\linewidth}
    \includegraphics[width=\linewidth]{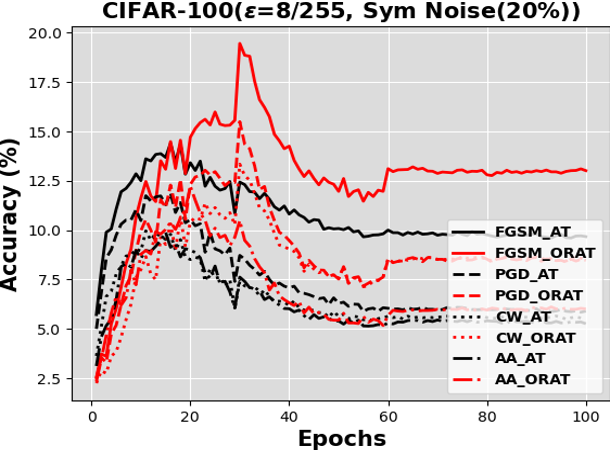}
\end{subfigure}

\vspace{-0.4em}
\caption{\em  The tendency curves of  training adversarial loss and test accuracy on three datasets. The sharp drops
in the curves correspond to decreases in training learning
rate.}
\vspace{-1em}
\label{fig:loss_acc_2}
\end{figure}

\subsection{More Experiments on Real Datasets}
In the main paper, we only show the tendency curves for MNIST when $\epsilon$=0.1 and CIFAR-10 and CIFAR-100 when $\epsilon$=2/255. In this section, we show more results on three datasets with 20\% symmetric noise by setting a big value of $\epsilon$ in Figure \ref{fig:loss_acc_2}. Similar to the observations in Figure \ref{fig:loss_acc_1}, we can find the losses are dramatically decreased in the first row of Figure \ref{fig:loss_acc_2}, which means Algorithm \ref{alg:atrr} can be successfully applied to solve \ORAT~optimization problem. From the second row of Figure \ref{fig:loss_acc_2}, it is obvious that the performance of our method is higher than the original AT approach on all attacks. 

\subsection{More Experiments on the Effect of $k$ and $m$}
\begin{figure}[t]
\begin{subfigure}[t]{0.33\linewidth}
    \includegraphics[width=\linewidth]{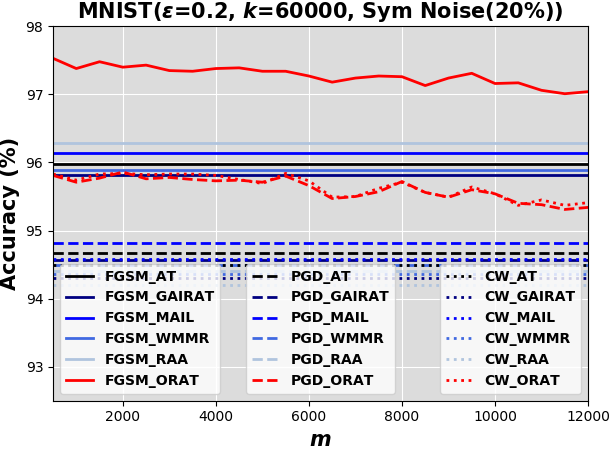}
\end{subfigure}%
\hfill%
\begin{subfigure}[t]{0.33\linewidth}
    \includegraphics[width=\linewidth]{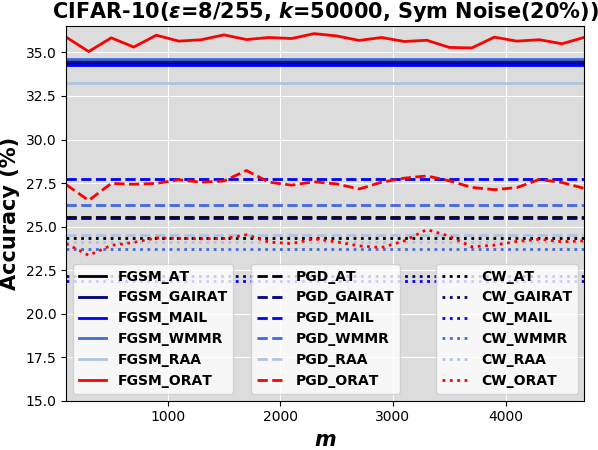}
\end{subfigure}%
\hfill%
\begin{subfigure}[t]{0.33\linewidth}
    \includegraphics[width=\linewidth]{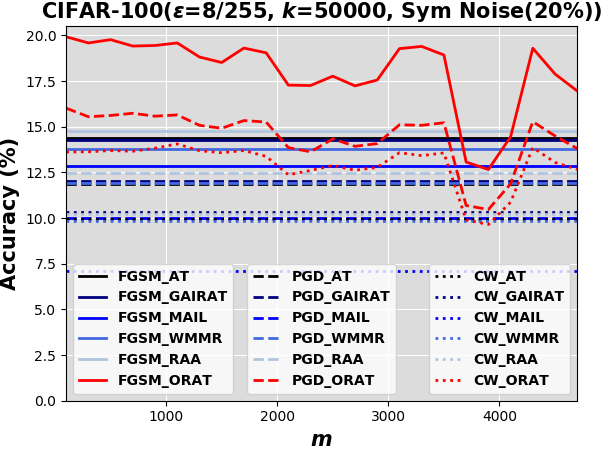}
\end{subfigure}%

\begin{subfigure}[t]{0.33\linewidth}
    \includegraphics[width=\linewidth]{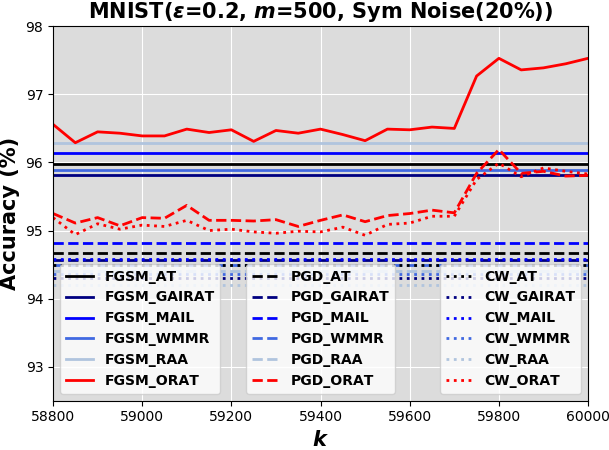}
\end{subfigure}%
\hfill%
\begin{subfigure}[t]{0.33\linewidth}
    \includegraphics[width=\linewidth]{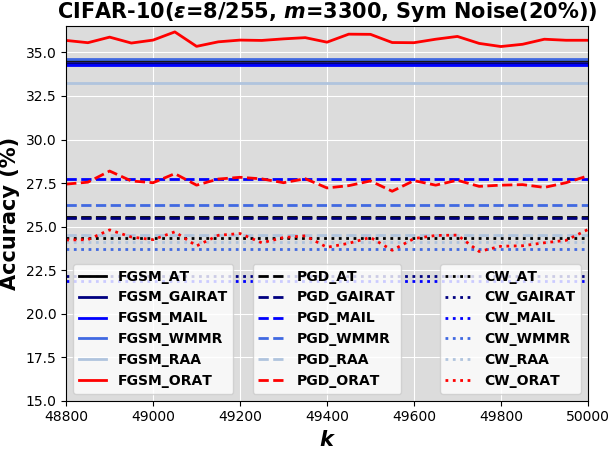}
\end{subfigure}%
\hfill%
\begin{subfigure}[t]{0.33\linewidth}
    \includegraphics[width=\linewidth]{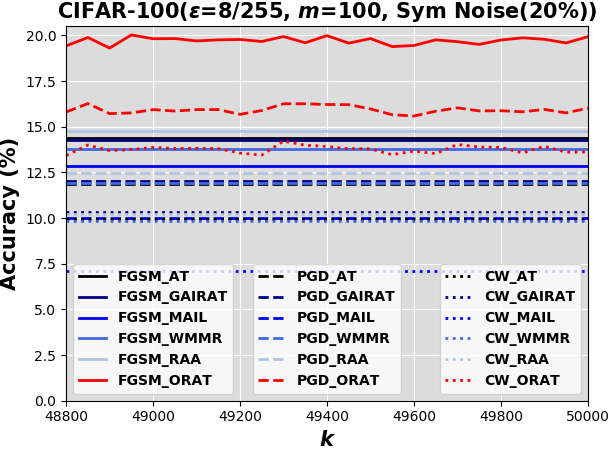}
\end{subfigure}%
\vspace{-0.4em}
\caption{\em  Effect of $k$ and $m$ on the test accuracy of \ORAT~on three datasets.}
\vspace{-2em}
\label{fig:k_m_more}
\end{figure}
We conduct more experiments to study the effect of hyperparameters $k$ and $m$ with using 20\% symmetric noise on all datasets by setting a big value of $\epsilon$. The results are shown in Figure \ref{fig:k_m_more}. Similar to the results that we get in Figure \ref{fig:k_m}, we can see that there is a clear range of $m$ with better  performance than all compared methods. Fix $m$ and test various $k$, we can find the performance can be improved by using some specific $k$ values. 

\subsection{Connection with Adversarial Training on Out-of-Distribution Problems}

Out-of-Distribution (OOD) problem exists due to the training and test data distributions mismatching \citep{hendrycks2021many}. Although the OOD problem setting is different from our outliers problem setting, some similarities exist between OOD data and outliers. For example, both of them are not from the data generating distribution. Therefore, whether the OOD methods can directly apply to solving our outlier problem in adversarial training is a question. Some works such as \cite{zeng2021adversariala,  varshney2022investigating, yi2021improved} connect adversarial robustness to out-of-distribution (OOD) problems.

Specifically, \cite{zeng2021adversariala} focuses on OOD detection. The problem in \cite{zeng2021adversariala} is that not enough labeled OOD samples can be used for training the OOD detection model. To improve the diversity of the unlabeled data augmentation, they apply an adversarial attack technique on unlabeled data to generate pseudo-positive samples. Then use these pseudo-positive samples with labeled data to improve the performance of the OOD detection model. However, their approach cannot directly apply to our setting since we only focus on supervised learning. All training data points are labeled in our setting, and the adversarial training works on labeled data.  
The authors in \cite{varshney2022investigating}  test different selective prediction approaches for Natural Language Processing systems in in-domain, OOD, and adversarial settings. They regard several existing datasets as adversarial datasets for testing. However, no adversarial training approach is proposed and involved in \cite{varshney2022investigating}. For \cite{yi2021improved}, the authors theoretically and experimentally show that a model (original AT \citep{madry2018towards} or pre-trained AT \citep{salman2020adversarially}) robust to input perturbation generalizes well on OOD data.

Therefore, we test whether the pre-trained AT method \citep{salman2020adversarially} can solve outlier problems in adversarial training. Following the experimental setting from \cite{yi2021improved}, we download the ImageNet-based adversarially pre-trained robust ResNet-18 model in the setting of $L_\infty$ and $\epsilon=2/255$ from the public repository \footnote{\url{https://github.com/microsoft/robust-models-transfer}}.  
Then fine-tune it on our noisy training datasets. We report pre-trained AT testing accuracy (\%) on CIFAR-100 in Table \ref{tab:OOD_experiments}. To make the comparison explicit, we also attach our method performance. From Table \ref{tab:OOD_experiments}, we can find our method outperforms pre-trained AT under all settings. Most of the performance gaps between pre-trained AT and our method in Table \ref{tab:OOD_experiments} are more than 5\%. One reason is that the pre-trained AT is not designed to handle outliers. According to these results, it is clear that pre-trained AT cannot directly apply to solving our problem even if it has a good performance on OOD data.

\begin{table*}[t]
\centering
\setlength\tabcolsep{3pt}
\begin{tabular}{|cc|c|cccc|}
\hline
\multicolumn{2}{|c|}{\multirow{2}{*}{Noise}}                     & \multirow{2}{*}{Defense} & \multicolumn{4}{c|}{CIFAR-100($\epsilon=2/255$)}                                                    \\ \cline{4-7} 
\multicolumn{2}{|l|}{}                                      &                   & \multicolumn{1}{l|}{Na} &  \multicolumn{1}{l|}{FG} & \multicolumn{1}{l|}{PGD} & CW  \\ \hline
\multicolumn{1}{|l|}{\multirow{8}{*}{\rotatebox{90}{\small \makecell{Symmetric Noise}}}} & \multirow{2}{*}{10\%} &        pre-trained AT           &  \multicolumn{1}{l|}{24.20} & \multicolumn{1}{l|}{17.09} & \multicolumn{1}{l|}{14.83} & 14.17 \\ \cline{3-7} 
\multicolumn{1}{|l|}{}                  &                   &     Ours              & \multicolumn{1}{l|}{\textbf{35.76}} & \multicolumn{1}{l|}{\textbf{25.72}} & \multicolumn{1}{l|}{\textbf{22.27}} & \textbf{21.28} \\ \cline{2-7} 
\multicolumn{1}{|l|}{}                  & \multirow{2}{*}{20\%} &   pre-trained AT                & \multicolumn{1}{l|}{19.53} & \multicolumn{1}{l|}{16.64} & \multicolumn{1}{l|}{14.01} & 13.26 \\ \cline{3-7} 
\multicolumn{1}{|l|}{}                  &                   &         Ours          & \multicolumn{1}{l|}{\textbf{34.45}} & \multicolumn{1}{l|}{\textbf{25.07}} & \multicolumn{1}{l|}{\textbf{22.21}} & \textbf{20.92} \\ \cline{2-7} 
\multicolumn{1}{|l|}{}                  & \multirow{2}{*}{30\%} &    pre-trained AT               & \multicolumn{1}{l|}{19.41} & \multicolumn{1}{l|}{16.22} & \multicolumn{1}{l|}{13.06} & 13.04 \\ \cline{3-7} 
\multicolumn{1}{|l|}{}                  &                   &          Ours         & \multicolumn{1}{l|}{\textbf{31.27}} & \multicolumn{1}{l|}{\textbf{23.81}} & \multicolumn{1}{l|}{\textbf{21.35}} & \textbf{19.59} \\ \cline{2-7} 
\multicolumn{1}{|l|}{}                  & \multirow{2}{*}{40\%} &      pre-trained AT             & \multicolumn{1}{l|}{18.86} & \multicolumn{1}{l|}{15.63} & \multicolumn{1}{l|}{12.78} & 11.91 \\ \cline{3-7} 
\multicolumn{1}{|l|}{}                  &                   &         Ours          & \multicolumn{1}{l|}{\textbf{29.38}} & \multicolumn{1}{l|}{\textbf{22.99}} & \multicolumn{1}{l|}{\textbf{20.85}} & \textbf{19.20} \\ \hline
\multicolumn{1}{|l|}{\multirow{8}{*}{\rotatebox{90}{\small \makecell{Asymmetric Noise}}}} & \multirow{2}{*}{10\%} &          pre-trained AT         & \multicolumn{1}{l|}{20.15} & \multicolumn{1}{l|}{19.66} & \multicolumn{1}{l|}{17.78} & 16.9 \\ \cline{3-7} 
\multicolumn{1}{|l|}{}                  &                   &        Ours           & \multicolumn{1}{l|}{\textbf{37.09}} & \multicolumn{1}{l|}{\textbf{27.07}} & \multicolumn{1}{l|}{\textbf{23.65}} & \textbf{22.59} \\ \cline{2-7} 
\multicolumn{1}{|l|}{}                  & \multirow{2}{*}{20\%} &   pre-trained AT                & \multicolumn{1}{l|}{24.60} & \multicolumn{1}{l|}{17.43} & \multicolumn{1}{l|}{16.67} & 15.15 \\ \cline{3-7} 
\multicolumn{1}{|l|}{}                  &                   &       Ours            & \multicolumn{1}{l|}{\textbf{36.05}} & \multicolumn{1}{l|}{\textbf{25.76}} & \multicolumn{1}{l|}{\textbf{22.83}} & \textbf{21.47} \\ \cline{2-7} 
\multicolumn{1}{|l|}{}                  & \multirow{2}{*}{30\%} &      pre-trained AT             & \multicolumn{1}{l|}{22.86} & \multicolumn{1}{l|}{16.92} & \multicolumn{1}{l|}{15.94} & 14.89 \\ \cline{3-7} 
\multicolumn{1}{|l|}{}                  &                   &          Ours         & \multicolumn{1}{l|}{\textbf{34.58}} & \multicolumn{1}{l|}{\textbf{24.18}} & \multicolumn{1}{l|}{\textbf{21.05}} & \textbf{20.11} \\ \cline{2-7} 
\multicolumn{1}{|l|}{}                  & \multirow{2}{*}{40\%} &     pre-trained AT              & \multicolumn{1}{l|}{21.58} & \multicolumn{1}{l|}{16.18} & \multicolumn{1}{l|}{15.32} & 14.48 \\ \cline{3-7} 
\multicolumn{1}{|l|}{}                  &                   &        Ours           & \multicolumn{1}{l|}{\textbf{33.65}} & \multicolumn{1}{l|}{\textbf{23.35}} & \multicolumn{1}{l|}{\textbf{20.76}} & \textbf{19.46 }\\ \hline
\end{tabular}
\vspace{-0.4em}
\caption{\it Testing accuracy (\%) of pre-trained AT and our method (ORAT) on CIFAR-100 ($\epsilon=2/255$) with different levels of symmetric and asymmetric  noise. The best results are shown in bold.}
\label{tab:OOD_experiments}
\vspace{-2em}
\end{table*}

\subsection{Extension of Table \ref{tab:partial_ATw/o}}\label{sec:AT_w/o}
Self-learning \citep{han2019deep} is a useful strategy for learning model on noise data. For example, we can use AoRR to filter examples with larger loss (potential outliers), then conducting adversarial training on the cleaner set. We call this method AT w/o. However, it is not an end-to-end training approach. In contrast, our method is an end-to-end method, which means it is very easy to be conducted. To compare the effectiveness of \ORAT~and this AT w/o approach, we conduct experiments on MNIST with symmetric noise and CIFAR-100 with symmetric noise as follows.  

In the first stage, for each dataset, we apply a grid search to select the values of $k$ and $m$ for training the model using the AoRR approach that can return a good testing accuracy. Then we use the trained model to test the loss for each sample from the training set. Therefore, we can obtain a training sample loss list. Next, we delete data points for the $m$ largest losses in the training set to construct a clean set. This is because the AoRR uses $m$ to determine how many examples (potential outliers) with the largest losses are ignored during each training epoch.

In the second stage, after we get a clean set, we use the conventional AT approach to train the model on the clean set and test the trained model on the testing set.

We report the testing accuracy (\%) of the AT w/o approach on MNIST (symmetric noise, $\epsilon=0.1$) and CIFAR-100 (symmetric noise, $\epsilon=2/255$) in Table \ref{tab:self-learning}. To make the comparison explicit, we also attach our method performance. From Table \ref{tab:self-learning}, we can find our method outperforms the AT w/o approach under all settings. For example, the performance gap between the AT w/o approach and our method (\ORAT) on MNIST can achieve more than 2\% under the 40\% symmetric noise setting. Most of the performance gaps on CIFAR-100 can achieve more than 4\%.

\begin{table*}[t]
\centering
\setlength\tabcolsep{3pt}
\begin{tabular}{|cc|c|cccc|cccc|}
\hline
\multicolumn{2}{|c|}{\multirow{2}{*}{Noise}}                     & \multirow{2}{*}{Defense} & \multicolumn{4}{c|}{MNIST ($\epsilon=0.1$)}                                                    & \multicolumn{4}{c|}{CIFAR-100 ($\epsilon=2/255$)}                                                    \\ \cline{4-11} 
\multicolumn{2}{|c|}{}                                      &                   & \multicolumn{1}{c|}{Na} & \multicolumn{1}{c|}{FG} & \multicolumn{1}{c|}{PGD} & CW & \multicolumn{1}{c|}{Na} & \multicolumn{1}{c|}{FG} & \multicolumn{1}{c|}{PGD} & CW  \\ \hline
\multicolumn{1}{|c|}{\multirow{8}{*}{\rotatebox{90}{\small \makecell{Symmetric Noise}}}} & \multirow{2}{*}{10\%} &  AT w/o                 & \multicolumn{1}{c|}{98.91} & \multicolumn{1}{c|}{98.08} & \multicolumn{1}{c|}{97.61} & 97.55 & \multicolumn{1}{c|}{29.81} & \multicolumn{1}{c|}{21.09} & \multicolumn{1}{c|}{19.59} & 18.23 \\ \cline{3-11} 
\multicolumn{1}{|c|}{}                  &                   &            Ours       & \multicolumn{1}{c|}{\textbf{99.52}} & \multicolumn{1}{c|}{\textbf{98.45}} & \multicolumn{1}{c|}{\textbf{97.78}} & \textbf{97.79} & \multicolumn{1}{c|}{\textbf{35.76}} & \multicolumn{1}{c|}{\textbf{25.72}} & \multicolumn{1}{c|}{\textbf{22.27}} & \textbf{21.28} \\ \cline{2-11} 
\multicolumn{1}{|c|}{}                  & \multirow{2}{*}{20\%} &         AT w/o          & \multicolumn{1}{c|}{98.77} & \multicolumn{1}{c|}{97.76} & \multicolumn{1}{c|}{97.20} & 97.12 & \multicolumn{1}{c|}{27.81} & \multicolumn{1}{c|}{20.92} & \multicolumn{1}{c|}{19.15} & 17.95 \\ \cline{3-11} 
\multicolumn{1}{|c|}{}                  &                   &            Ours       & \multicolumn{1}{c|}{\textbf{99.56}} & \multicolumn{1}{c|}{\textbf{98.37}} & \multicolumn{1}{c|}{\textbf{97.65}} & \textbf{97.64} & \multicolumn{1}{c|}{\textbf{34.45}} & \multicolumn{1}{c|}{\textbf{25.07}} & \multicolumn{1}{c|}{\textbf{22.21}} & \textbf{20.92} \\ \cline{2-11} 
\multicolumn{1}{|c|}{}                  & \multirow{2}{*}{30\%} &         AT w/o         & \multicolumn{1}{c|}{97.82} & \multicolumn{1}{c|}{96.97} & \multicolumn{1}{c|}{96.47} & 96.35 & \multicolumn{1}{c|}{24.18} & \multicolumn{1}{c|}{19.17} & \multicolumn{1}{c|}{17.31} & 16.71 \\ \cline{3-11} 
\multicolumn{1}{|c|}{}                  &                   &            Ours       & \multicolumn{1}{c|}{\textbf{99.55}} & \multicolumn{1}{c|}{\textbf{98.30}} & \multicolumn{1}{c|}{\textbf{97.51}} & \textbf{97.53} & \multicolumn{1}{c|}{\textbf{31.27}} & \multicolumn{1}{c|}{\textbf{23.81}} & \multicolumn{1}{c|}{\textbf{21.35}} & \textbf{19.59} \\ \cline{2-11} 
\multicolumn{1}{|c|}{}                  & \multirow{2}{*}{40\%} &          AT w/o         & \multicolumn{1}{c|}{97.03} & \multicolumn{1}{c|}{95.85} & \multicolumn{1}{c|}{95.22} & 95.05 & \multicolumn{1}{c|}{21.17} & \multicolumn{1}{c|}{17.81} & \multicolumn{1}{c|}{16.85} & 15.48 \\ \cline{3-11} 
\multicolumn{1}{|c|}{}                  &                   &            Ours       & \multicolumn{1}{c|}{\textbf{99.36}} & \multicolumn{1}{c|}{\textbf{98.00}} & \multicolumn{1}{c|}{\textbf{97.22}} & \textbf{97.20} & \multicolumn{1}{c|}{\textbf{29.38}} & \multicolumn{1}{c|}{\textbf{22.99}} & \multicolumn{1}{c|}{\textbf{20.85}} & \textbf{19.20} \\ \hline
\end{tabular}
\vspace{-0.6em}
\caption{\it Testing accuracy (\%) of self-learning based method (Self-learning) and our method (ORAT) on MNIST ($\epsilon=0.1$) and CIFAR-100 ($\epsilon=2/255$) with different levels of symmetric  noise. The best results are shown in bold.}
\label{tab:self-learning}
\vspace{-1em}
\end{table*}

One reason for low performance from the self-learning approach is that the training data points ignored by AoRR may contain clean data points. In this case, the constructed clean set is smaller than the original dataset. This may hurt the final model performance. Moreover, removing the examples with the largest losses before the adversarial training may lose the important feature information from the original training dataset. In other words, this compromises the richness and representational power of the data. In contrast, our ORAT method considers all examples during adversarial training. According to these results, it is clear that our approach (\ORAT) gives a better solution than the self-learning approach for solving outlier problems in adversarial training either in the algorithm efficiency or effectiveness.

\subsection{More Analysis on Stability of \ORAT}\label{sec: Stability_Evaluation}
\vspace{-1mm}
To evaluate the stability of each method, we report the the mean and standard deviation of testing accuracy (\%) of all methods on MNIST (40\% symmetric noise, $\epsilon=0.1$) and CIFAR-100 (40\% symmetric noise, $\epsilon=2/255$) in Table \ref{tab:Stability}. For each method, the reported performance is obtained by averaging the testing accuracy according to 10 random seeds. From Table \ref{tab:Stability}, we can find our method still outperforms the compared methods in both datasets. For MNIST, our method can even outperform AT by more than 2\%. Most importantly, we can find that the standard deviation in our method is less than or equal to that of other compared methods. For CIFAR-100, we can find the mean value of our method ORAT even higher than the reported performance in our submission. The standard deviation of the performance of our method differs from the comparison methods by at most 0.26\% (compared to ST on FGSM attack). Comparing Table \ref{tab:Stability} and Table \ref{tab:general_performance_1}, it is clear that the performance gap becomes larger when we report scores by using mean and standard deviation, and our method shows a stable and stronger ability in handling outliers and adversarial attacks.

\subsection{Evaluation on Wide ResNet}
\begin{table}[t]
\centering
\setlength\tabcolsep{3pt}
{\scalebox{0.9}{

            \begin{tabular}{|c|c|cccc|cccc|}
\hline
\multirow{2}{*}{Noise} & \multirow{2}{*}{Defense} & \multicolumn{4}{c|}{CIFAR-10 ($\epsilon=2/255$)}                                                    & \multicolumn{4}{c|}{CIFAR-10 ($\epsilon=8/255$)}                                                    \\ \cline{3-10} 
                  &                   & \multicolumn{1}{c}{Na} & \multicolumn{1}{c}{FG} & \multicolumn{1}{c}{PGD} & CW & \multicolumn{1}{c}{Na} & \multicolumn{1}{c}{FG} & \multicolumn{1}{c}{PGD} & CW \\ \hline\hline
\multirow{7}{*}{\rotatebox{90}{ \footnotesize {20\%  Sym Noise }} } &         ST          & \multicolumn{1}{c}{91.31} & \multicolumn{1}{c}{53.25} & \multicolumn{1}{c}{27.66} & 26.47 & \multicolumn{1}{c}{91.09} & \multicolumn{1}{c}{35.36} & \multicolumn{1}{c}{7.11} & 7.34 \\ \cline{2-10} 
                  &         AT          & \multicolumn{1}{c}{90.74} & \multicolumn{1}{c}{\3 82.86} & \multicolumn{1}{c}{\3 78.97} & \1 78.96 & \multicolumn{1}{c}{81.80} & \multicolumn{1}{c}{\g 62.43} & \multicolumn{1}{c}{\ggg 50.84} & \3 50.92 \\ \cline{2-10} 
                  &          GAIRAT	         & \multicolumn{1}{c}{88.17} & \multicolumn{1}{c}{\1 84.00} & \multicolumn{1}{c}{\1 81.72} & \3 76.05 & \multicolumn{1}{c}{78.55} & \multicolumn{1}{c}{\1 67.95} & \multicolumn{1}{c}{\1 61.94} & \g 47.81 \\ \cline{2-10} 
                  &         MAIL          & \multicolumn{1}{c}{73.31} & \multicolumn{1}{c}{\ggg 65.94} & \multicolumn{1}{c}{\ggg 62.85} & \ggg 58.05 & \multicolumn{1}{c}{69.01} & \multicolumn{1}{c}{\ggg 52.09} & \multicolumn{1}{c}{\ggg 44.33} & \ggg 38.85 \\ \cline{2-10} 
                  &         WMMR          & \multicolumn{1}{c}{87.75} & \multicolumn{1}{c}{\g 79.83} & \multicolumn{1}{c}{\g 76.21} & \3 75.48 & \multicolumn{1}{c}{80.97} & \multicolumn{1}{c}{\g 61.70} & \multicolumn{1}{c}{\ggg 51.17} & \3 49.82 \\ \cline{2-10} 
                  &         RAA          & \multicolumn{1}{c}{90.55} & \multicolumn{1}{c}{\3 82.44} & \multicolumn{1}{c}{\3 78.38} & \1 78.56 & \multicolumn{1}{c}{77.62} & \multicolumn{1}{c}{\g 60.05} & \multicolumn{1}{c}{\ggg 48.87} & \3 48.99 \\ \cline{2-10} 
                 &          \cellcolor{lightgray} \textbf{Ours}         & \multicolumn{1}{c}{\cellcolor{lightgray} {90.72}} & \multicolumn{1}{c}{\cellcolor{lightgray} \textbf{85.47}} & \multicolumn{1}{c}{\cellcolor{lightgray} \textbf{82.30}} & \cellcolor{lightgray} \textbf{80.34} & \multicolumn{1}{c}{\cellcolor{lightgray} {81.99}} & \multicolumn{1}{c}{\cellcolor{lightgray} \textbf{69.08}} & \multicolumn{1}{c}{\cellcolor{lightgray} \textbf{63.87}} & \cellcolor{lightgray} \textbf{53.53} \\ \hline
\end{tabular}
            }
            
}
\vspace{-0.6em}
\caption{\small \it Testing accuracy (\%) using Wide ResNet.}
\label{tab:wide-resnet}
\vspace{-1em}
\end{table}
We evaluate all methods using Wide ResNet 
on CIFAR-10 dataset with 20\% symmetric noise. The Wide ResNet framework is WRN-32-10, which is the same as \cite{madry2018towards}. 
Results in Table \ref{tab:wide-resnet} show our approach outperforms others when using a large model.

\subsection{Experiments on Clothing1M}\label{sec: Clothing1M}

\begin{wrapfigure}{r}{.5\textwidth}
\includegraphics[width=.5\textwidth]{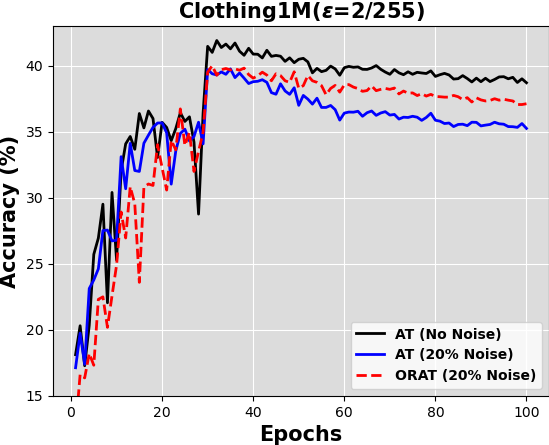}
\caption{\small \em The tendency curves of testing accuracy on the Clothing1M dataset.}
\label{fig:Clothing1M}
\end{wrapfigure}
To demonstrate the effectiveness of our method \ORAT~on a more real scenario, we conduct experiments on the Clothing1M dataset \cite{xiao2015learning}. This dataset contains roughly one million clothing images crawled from the Internet. Most of them have noisy labels extracted from their surrounding texts. A few of them have clean labels, which are manually annotated by Xiao et al. \cite{xiao2015learning}. Specifically, we extract 30000 clean labeled images as the clean training set and 10000 clean labeled images as the test set. To create a noise training set, we select 80\% images from the clean training set and extract 30000$\times $20\%=6000 images from the original noise labeled images. Therefore, we can obtain a noise training set with the same sample size as the clean training set. Then we use AT to train the Small-CNN model on the clean training set (named AT (No noise)) and noise training set (named AT (20\% noise)), respectively. For our method \ORAT, we use it to train the same model on the noise training set, named \ORAT~(20\% noise). We show the tendency curves of the test accuracy in Figure \ref{fig:Clothing1M}. From the table, comparing AT (No noise) and AT (20\% noise), we can see that AT performance is decreased if the data contains noise, which means outliers affect the performance of AT. In addition, our \ORAT~method outperforms AT on the noise data, which means our method can reduce the influence of outliers on adversarial training.

\end{document}